\documentclass{article}
\usepackage[preprint]{cs-conference}

\usepackage{amssymb}
\usepackage{amsmath}
\usepackage{amsfonts}
\usepackage{amsthm}
\usepackage{bm}
\usepackage{booktabs}
\usepackage{float}
\usepackage{placeins}
\usepackage{graphicx}
\usepackage{microtype}
\usepackage{titletoc}
\usepackage{xcolor}
\usepackage{url}
\definecolor{linkcolor}{rgb}{0.10,0.25,0.55}
\definecolor{citecolor}{rgb}{0.10,0.45,0.25}
\definecolor{learnedparamred}{RGB}{145,35,35}
\newcommand{\learnedparam}[1]{\textcolor{learnedparamred}{#1}}
\usepackage[
  pagebackref=false,
  breaklinks=true,
  colorlinks=true,
  bookmarks=false,
  citecolor=citecolor,
  linkcolor=linkcolor,
  urlcolor=linkcolor
]{hyperref}
\hypersetup{
  pdftitle={A Flow Matching Framework for Neural Representational Dissimilarity},
  pdfauthor={Zeyuan Ye and Xue-Xin Wei}
}

\definecolor{revisiongreen}{RGB}{45,125,80}

\newtheorem{proposition}{Proposition}[section]
\newtheorem{lemma}[proposition]{Lemma}

\title{A Flow Matching Framework\\for Neural Representational Dissimilarity\thanks{Code: \url{https://github.com/AgeYY/FlowRDM}.}}

\author{Zeyuan Ye \qquad Xue-Xin Wei\\
Department of Neuroscience, The University of Texas at Austin\\
\texttt{\{y.zeyuan, weixx\}@utexas.edu}}

\begin{document}

\raggedbottom

\maketitle

\begin{abstract}
Neural representational dissimilarity quantifies differences between neural response distributions, and is essential for comparing neural codes across stimuli, brain areas, tasks, and models. Commonly used distance metrics involve different assumptions and are estimated with separate methods. Here, we show that a variety of distance metrics can be unified under a flow matching framework developed in deep generative models. That is, these distances arise as Jeffreys divergences under different velocity constraints. We find that flow matching has advantages for estimating distances involving complicated distributions and continuous variables. Furthermore, this framework enables the design of new distance metrics in a principled way. Together, flow matching provides a unified approach for understanding, estimating, and designing neural representational dissimilarity metrics.

\end{abstract}

\begin{figure}[H]
\centering
\includegraphics[width=0.7\linewidth]{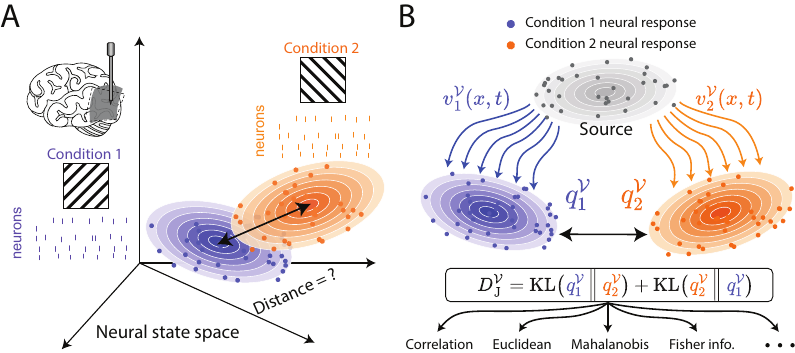}
\caption{\textbf{Overview of the proposed flow matching framework for neural representational dissimilarity.} (\textbf{A}) Different stimuli evoke condition-specific neural responses, producing different population-response distributions whose dissimilarity is to be quantified. (\textbf{B}) Flow matching learns velocity fields $v_1^{\mathcal V}(x,t)$ and $v_2^{\mathcal V}(x,t)$ that transport a shared source distribution to the fitted distributions $q_1^{\mathcal V}$ and $q_2^{\mathcal V}$. Here, we define representational dissimilarity as the Jeffreys divergence between $q_1^{\mathcal V}$ and $q_2^{\mathcal V}$. We prove that different choices of $\mathcal V$ correspond to different distance metrics.}
\label{fig:overview}
\vspace{-0.25em}
\end{figure}

\section{Introduction}

A central question in neuroscience is how the brain represents information through stochastic neural population responses. Representing information means not only ``storing'' it but also ``structuring'' it in an appropriate format to support downstream computation given the biological constraints \citep{KriegeskorteKievit2013Geometry,KriegeskorteWei2021TuningGeometry}. How can we study representational structures? Representational similarity analysis (RSA) takes a geometric perspective, probing representational structure by measuring pairwise distances between neural population response distributions across conditions, such as stimuli or task states \citep{Kriegeskorte2008RSA,Nili2014RSAToolbox}. RSA has been widely used to study how neural representations support perception and task performance \citep{Kriegeskorte2008RSA,Cichy2014Resolving,Diedrichsen2017RepresentationalModels} and to compare representational systems across species and computational models, including deep neural network models~\citep{Kriegeskorte2008Categorical,Kriegeskorte2009PopulationCodes,Kornblith2019Similarity}.

A key step in RSA is to define and estimate representational distances. Because neural responses are stochastic, this often becomes a mathematical question of how to compare distributions. No single distance metric is best for every purpose as different metrics reveal different properties of the distributions \citep{Walther2016Reliability,Diedrichsen2017RepresentationalModels,KriegeskorteWei2021TuningGeometry}. For example, correlation distance compares the relative pattern of mean activity across neurons. Euclidean distance measures separation between mean response vectors. Mahalanobis distance additionally accounts for response covariance. Kullback-Leibler (KL) divergence compares full response distributions. Lastly, Fisher information measures how sensitive the response distribution changes with a small change of the stimulus variable \citep{Kullback1951Information,Amari2016InformationGeometry}. 

Given the diversity of these useful distance metrics, a natural question arises: 
can they be unified under a single mathematical framework? Such a framework, if exists, could offer several benefits. First, it would make the statistical assumptions underlying different metrics more explicit, clarifying their interpretation. Second, it would provide a principled way to design new distance metrics. Finally, it could provide a unified estimator for these metrics, so that improvements to the estimator could benefit multiple metrics rather than requiring separate efforts for each one.

Here, we provide such a framework using flow matching. Flow matching is a modern deep generative modeling method that is widely used in image, video, and language generation \citep{Lipman2023FlowMatching,Esser2024RectifiedFlow,Davtyan2023VideoFlow,Hu2024TextFlow}. It starts by generating a data sample from a simple source distribution (e.g., a Gaussian distribution) and then evolves the sample toward a target distribution through a velocity field. In our flow-matching framework for neural dissimilarity, we define the representational distance as the Jeffreys divergence between two target neural-response distributions under a given velocity constraint. We mathematically prove that this Jeffreys divergence reduces to existing distance metrics under different velocity-field constraints (Fig.~\ref{fig:overview}). We show that flow matching offers advantages when estimating distances involving complex distributions and continuous variables, and allows the principled design of new distance metrics.

\section{Background and Related Work}
\label{sec:background}

\paragraph{Representational dissimilarity analysis.}
RSA summarizes neural, behavioral, or model representations by a representational dissimilarity matrix (RDM), whose entries measure the pairwise separation between conditions \citep{Kriegeskorte2008RSA,Kriegeskorte2009PopulationCodes,KriegeskorteKievit2013Geometry,Nili2014RSAToolbox}. Common RDM entries include correlation, Euclidean, Mahalanobis, and cross-validated Mahalanobis distances; these choices differ in how they treat response scale, noise covariance, and finite-sample bias \citep{Walther2016Reliability,Diedrichsen2017RepresentationalModels}.

\paragraph{Fisher information estimation.}
Fisher information measures how rapidly $p_\theta(x)$ changes with a continuous stimulus or behavioral variable, $\theta$ \citep{Fisher1922Foundations,Amari2016InformationGeometry,SeungSompolinsky1993PopulationCodes,Pouget2000PopulationCodes}. Its inverse lower-bounds the variance of unbiased estimators, making Fisher information a local measure of discriminability \citep{Series2009Homunculus, Rao1945Information,Brunel1998PopulationCoding,Abbott1999CorrelatedVariability,Averbeck2006NeuralCorrelations}. With repeated trials at nearby stimulus values, Fisher information can be estimated from tuning-curve derivatives and response covariance or from decoder-based estimator \citep{Kanitscheider2015MeasuringFisher,Kohn2016CorrelationsInformation}. Continuous and naturalistic experiments, however, have no exact repeats. Estimation procedures need to exploit information shared across nearby values of the condition $\theta$, for example by using Gaussian processes \citep{RasmussenWilliams2006GaussianProcesses,YeWessel2025GridInformation,Nejatbakhsh2023Wishart}. However, Gaussian processes scale poorly \citep{Bruinsma2020ScalableGP}. Our flow-matching framework unifies Fisher information estimation with the other distance metrics that have traditionally been treated separately. %

\paragraph{Continuous normalizing flows and flow matching.}
Normalizing flows model densities through invertible transformations \citep{Rezende2015Variational,Papamakarios2021NormalizingFlows}, while continuous normalizing flows transport samples through an ODE \citep{Chen2018NeuralODE}. Roughly speaking, a sample begins at \(x_0\sim \rho_0\), follows the velocity field \(v\) over time, and arrives at a distribution intended to approximate the observed data. In the original continuous-flow formulation, the velocity parameters are trained by maximum likelihood, which requires solving the ODE and estimating the divergence of the velocity field during optimization, which can be slow \citep{Chen2018NeuralODE,Grathwohl2019FFJORD}. Flow matching reformulates this training problem. Instead of maximizing likelihoods, one specifies a probability path between source samples and data samples, computes the target velocity along that path, and trains \(v\) by regressing onto those velocities \citep{Lipman2023FlowMatching,Albergo2023StochasticInterpolants}. Flow matching thus provides a practical way to learn a transport map from a source distribution to an empirical response distribution.

\section{A flow matching framework for computing representational distance}
\label{sec:framework}
Flow matching learns a time-dependent velocity field that transports samples from a simple \textit{source} distribution to a \textit{target} distribution \citep{Lipman2023FlowMatching}. Consequently, different constraints on the velocity field yield different approximations to the target distribution. Let $x\in\mathbb R^d$ denote a neural population response vector, where $d$ is the number of neurons, and let $p_c(x)$ denote the unknown target distribution under condition $c\in\mathcal C$. Starting from a source sample $X_0\sim\rho_0$, sampling follows the ODE
\begin{equation}
    \frac{d X_t}{dt}=v^{\mathcal V}_c(X_t,t), \qquad t\in[0,1],
    \label{eq:flow-ode}
\end{equation}
where the velocity field $v_c^{\mathcal V}(x,t)$ specifies the direction and speed of the sample at each flow time and $\mathcal V$ denotes the chosen velocity class. Integrating Eq.~\eqref{eq:flow-ode} from $t=0$ to $t=1$ defines a transformation $T_{c,1}^{\mathcal V}$ and the fitted distribution
\begin{equation}
    q_c^{\mathcal V}=(T_{c,1}^{\mathcal V})_{\#}\rho_0,
    \label{eq:fitted-distribution}
\end{equation}
so that $T_{c,1}^{\mathcal V}(X_0)\sim q_c^{\mathcal V}$ when $X_0\sim\rho_0$. The fitted distribution $q_c^{\mathcal V}$ approximates $p_c$. Here, $t$ is an artificial flow-matching coordinate, not physical time.

In general, our framework has three steps: (1) choose a source distribution and a velocity class $\mathcal V$; (2) learn the velocity field by flow matching to approximate the target distributions; and (3) use the learned model to compute the Jeffreys divergence between the approximated distributions as the representational distance.

\paragraph{Notation.}
We use $a,b\in\mathcal C$ for a pair of categorical conditions being compared and $\mu_c=\mathbb E_{X\sim p_c}[X]$ for the mean response under condition $c$. We use $\theta\in\Theta$ for a continuously varying condition, such as stimulus orientation, physical time, or behavior. For simplicity, we describe the framework below using the categorical condition $c$; the same framework applies to the continuous condition $\theta$.

\paragraph{1. Choosing a source distribution and velocity class.}
We first choose a source distribution $\rho_0$ and a velocity class $\mathcal V$. The source specifies the initial distribution, which is typically shared across conditions. Meanwhile, the velocity class specifies the transformations allowed from the source to each target. For example, an unconstrained velocity field, parameterized by a neural network, permits flexible nonlinear transformations, whereas a constant, state-independent velocity restricts the transformation to a translation. These choices determine which properties of the target distributions can be represented (see Fig.~\ref{fig:velocity-class-distance-correspondence}) and, consequently, which representational distance is induced. Section~\ref{sec:metric-correspondences} presents the specific choices used to recover different distance metrics.

\begin{figure}[!htbp]
  \centering
  \includegraphics[width=0.90\linewidth]{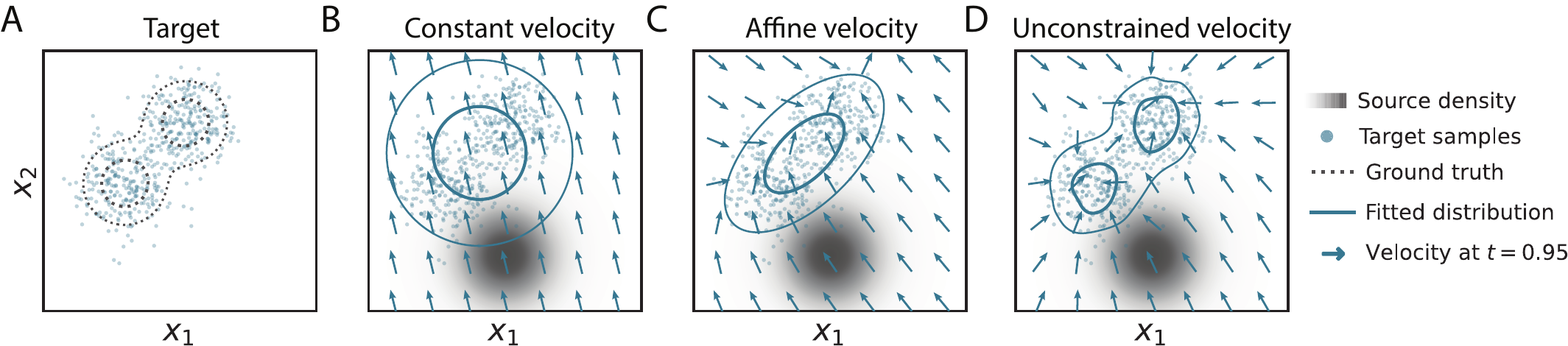}
  \caption{\textbf{Velocity classes $\mathcal V$ used in flow matching determine the fitted distributions $q_c^{\mathcal V}$} (also see Table~\ref{tab:constraints}). (A) A two-component Gaussian-mixture target, shown as samples and dotted density contours. (B) A constant velocity translates the Gaussian source (black region) without changing its shape. (C) An affine velocity allows an affine transformation of the source distribution, yielding a Gaussian approximation to the target. (D) An unconstrained velocity captures the two modes of the target. Arrows indicate velocity direction, not magnitude.}
  \label{fig:velocity-class-distance-correspondence}
\end{figure}

\paragraph{2. Flow-matching training.}
Given the source distribution and velocity class, flow matching learns the velocity field with a simple regression objective \citep{Lipman2023FlowMatching,Albergo2023StochasticInterpolants}. During training, we draw source samples $X_0\sim \rho_0$ and observed responses $X_1^{(c)}\sim p_c$ and connect them along a predefined path
\begin{equation}
    X_t^{(c)} = \alpha_t X_0 + \beta_t X_1^{(c)},
    \qquad
    (\alpha_0,\beta_0)=(1,0),\quad
    (\alpha_1,\beta_1)=(0,1),
    \label{eq:path}
\end{equation}
where $\alpha_t$ and $\beta_t$ are prescribed interpolation schedules controlling the contributions of the source and target samples at flow time $t$. The boundary conditions ensure that the path begins at the source sample and ends at the recorded response. Its target instantaneous velocity is
\begin{equation}
    U_t^{(c)}=\dot\alpha_t X_0+\dot\beta_t X_1^{(c)},
\end{equation}
where the dots denote derivatives with respect to $t$. A neural network receives $(X_t^{(c)},t,c)$ and is trained to predict this target velocity. The flow-matching objective is
\begin{equation}
    \mathcal L_{\mathrm{FM}}(v)
    =
    \mathbb E_{\substack{
        C\sim\pi_{\mathcal C},\;t\sim\operatorname{Unif}[0,1]\\
        X_0\sim \rho_0,\;X_1^{(C)}\sim p_C
    }}
    \left[
        \left\|v_C(X_t^{(C)},t)-U_t^{(C)}\right\|^2
    \right],
    \label{eq:fm-objective}
\end{equation}
where $C\sim\pi_{\mathcal C}$ denotes a randomly sampled categorical condition. For a continuous condition, we instead sample $\Theta\sim\pi_\Theta$ and replace $C$, $p_C$, $X_1^{(C)}$, and $v_C$ by $\Theta$, $p_\Theta$, $X_1^{(\Theta)}$, and $v_\Theta$, respectively. The population minimizer $v^{\mathcal V}=\arg\min_{v\in\mathcal V}\mathcal L_{\mathrm{FM}}(v)$ is learned by backpropagation through the velocity network.

Optionally, the learned velocity field can be fine-tuned by maximizing the log likelihood (Section~\ref{sec:nll-fine-tuning}). We use this fine-tuning only for the new distance metric introduced in Supplementary Section~\ref{sec:geometry}, where it is practically useful for optimization; all other results use flow-matching training alone for simplicity.

\paragraph{3. Jeffreys-divergence computation.} \label{sec:forward-sampling}
We define the representational distance between conditions $a$ and $b$ as the Jeffreys divergence between their \textit{fitted} distributions \citep{Kullback1951Information,Amari2016InformationGeometry},
\begin{equation}
    D_{\mathcal V}(a,b)
    =
    D_{\mathrm{KL}}\!\left(q_a^{\mathcal V}\middle\|q_b^{\mathcal V}\right)
    +
    D_{\mathrm{KL}}\!\left(q_b^{\mathcal V}\middle\|q_a^{\mathcal V}\right),
    \label{eq:flow-jeffreys}
\end{equation}
where each directed KL divergence is
\begin{equation}
    D_{\mathrm{KL}}\!\left(q_a^{\mathcal V}\middle\|q_b^{\mathcal V}\right)
    =
    \mathbb E_{x\sim q_a^{\mathcal V}}
    \left[\log q_a^{\mathcal V}(x)-\log q_b^{\mathcal V}(x)\right].
\end{equation}

When the source distribution is Gaussian and the velocity class $\mathcal V$ is affine
(linear in the state), the fitted distributions $q_c^{\mathcal V}$ can
be obtained by solving deterministic ODEs for their means and covariances (Supplementary Section~\ref{app:affine-gaussian-evaluation}), from which the Jeffreys divergence can be computed.

In other cases, we estimate the Jeffreys divergence using Monte Carlo. We draw samples $X_i^{(a)}\sim q_a^{\mathcal V}$ and $X_i^{(b)}\sim q_b^{\mathcal V}$ for $i=1,\ldots,M$ by integrating the sampling ODE in Eq.~\eqref{eq:flow-ode} and compute
\begin{equation}
    \widehat D_{\mathcal V}(a,b)
    =
    \frac{1}{M}\sum_{i=1}^M
    \bigl[
        \log q_a^{\mathcal V}(X_i^{(a)})
        -\log q_b^{\mathcal V}(X_i^{(a)})
    \bigr]
    +
    \frac{1}{M}\sum_{i=1}^M
    \bigl[
        \log q_b^{\mathcal V}(X_i^{(b)})
        -\log q_a^{\mathcal V}(X_i^{(b)})
    \bigr].
    \label{eq:mc-jeffreys}
\end{equation}
These log densities are computed using the continuous change-of-variables identity in Eq.~\eqref{eq:cnf-log-density}.

\section{Flow matching with different velocity constraints lead to different representational metrics}
\label{sec:metric-correspondences}

We now present our main theoretical results, \textit{i.e., }\textit{various commonly used representational distance metrics arise from our flow matching framework as Jeffreys divergences under different velocity classes. }
These correspondences are summarized in Table~\ref{tab:constraints}.
For each velocity class, the proof follows a similar general strategy. We substitute the analytical form of the constrained velocity field into the population flow-matching loss and solve for its minimizer, which determines the induced approximate distribution. Evaluating the Jeffreys divergence between the resulting condition-specific approximations then yields the corresponding distance metric.
Complete proofs for all rows of Table~\ref{tab:constraints} are provided in the Supplementary Section~\ref{app:velocity-distance-correspondence-proofs}.%

\begin{table}[!htbp]
\centering

\caption{\small{Distance metrics and their corresponding velocity classes. We use $c$ for a categorical condition and $\theta$ for a continuous condition. All rows use the shared standard Gaussian source $\rho_0=\mathcal N(0,I)$ except the starred geometry-template distance, which uses a source distribution concentrated around a low-dimensional manifold (Supplementary Section~\ref{sec:geometry}). The cosine and correlation rows equal $2r^2$ times their conventional definitions. In the Mahalanobis row, $A_t$ is shared across conditions. In the linear-Fisher row, $\bar A_{\theta,t}$ is shared locally across nearby conditions (see Supplementary Section~\ref{app:velocity-distance-correspondence-proofs}} %
).}
\label{tab:constraints}
\footnotesize
\setlength{\tabcolsep}{4pt}
\begin{tabular}{@{}p{0.62\linewidth}p{0.36\linewidth}@{}}
\toprule
Velocity field & Induced distance \\
\midrule
$v_c(x,t)=\dot\beta_t \learnedparam{b_c}$ & Euclidean distance \\
$v_c(x,t)=\dot\beta_t \learnedparam{b_c},\ \|\learnedparam{b_c}\|=r$ & Cosine distance \\
$v_c(x,t)=\dot\beta_t \learnedparam{b_c},\ \mathbf 1^\top \learnedparam{b_c}=0,\ \|\learnedparam{b_c}\|=r$ & Correlation distance \\
$v_c(x,t)=\dot\beta_t \learnedparam{b_c}+\learnedparam{A_t}(x-\beta_t \learnedparam{b_c})$ & Mahalanobis distance \\
$v_c(x,t)=\dot\beta_t \learnedparam{b_c}+\learnedparam{A_{c,t}}(x-\beta_t \learnedparam{b_c})$ & Gaussian Jeffreys divergence \\
$v_\theta(x,t)=\dot\beta_t \learnedparam{b_\theta}+\learnedparam{\bar A_{\theta,t}}(x-\beta_t \learnedparam{b_\theta})$ & Linear Fisher information \\
$\learnedparam{v_c(x,t)},\ \learnedparam{v_\theta(x,t)} \in\mathcal V_{\mathrm{unconstrained}}$ & Jeffreys divergence; full Fisher information \\
$\learnedparam{v_c(x,t; \tau)},\ \learnedparam{v_\theta(x,t; \tau)}$  & Auxiliary $\tau$-dependent distance \\
$v_c(x,t)=\learnedparam{a_{c,t}}+\learnedparam{\Omega_{c,t}}(x-\learnedparam{o_{c,t}})+\learnedparam{\lambda_{c,t}}(x-\learnedparam{o_{c,t}}),\ \learnedparam{\Omega_{c,t}}^\top=-\learnedparam{\Omega_{c,t}}$ & Geometry-template Jeffreys divergence$^\ast$ \\
\bottomrule
\end{tabular}
\par\hfill{\scriptsize $^\ast$: new distance introduced in this paper;\quad \textcolor{learnedparamred}{dark red}: outputs from trainable neural networks}
\end{table}

Fig.~\ref{fig:velocity-class-distance-correspondence} illustrates this correspondence for three representative velocity classes. A velocity field that is constant in $x$ can only translate the Gaussian source. At the population optimum, this translation shifts the source mean to the target-distribution mean while leaving its covariance unchanged, so the Jeffreys divergence reduces to the squared Euclidean distance between target means. An affine-in-$x$ velocity field can both translate and linearly transform the Gaussian source. Its population optimum therefore produces the moment-matched Gaussian approximation to the target distribution, and the resulting Jeffreys divergence is the Gaussian Jeffreys divergence. Finally, in the population limit, an unconstrained velocity field recovers the target distribution itself, yielding the full Jeffreys divergence.

Beyond a unification of existing distance metrics, the flow-matching framework enables principled designs of new metrics by specifying the source distribution as a distributional ``template'' and choosing a velocity class that constrains how the template can be transformed. The resulting distance is the Jeffreys divergence between the two fitted target distributions. As a demonstration, we design a new distance metric using a velocity class restricted to similarity transformations, thereby preserving the geometry of the source distribution when approximating the target distributions (Supplementary Section~\ref{sec:geometry}). This metric can be useful when the geometry of the target distributions is known a priori or when one wishes to emphasize a particular geometric property.

\FloatBarrier
\section{Applications}

\label{sec:synthetic}

We evaluate the flow-matching framework across a number of applications. First, we use a simple synthetic categorical dataset to test whether different velocity constraints recover multiple representational distance metrics (Section~\ref{sec:categorical-application}). The remaining four settings involve continuous variables: Fisher information in simulations and neural recordings (Section~\ref{sec:fisher}), and time-resolved RDMs in simulations and neural recordings (Section~\ref{sec:time-resolved}). The later four settings highlight an important advantage of flow matching for complex distributions that depend on continuous variables.

\subsection{Flow matching recovers multiple distance metrics on a synthetic categorical dataset}
\label{sec:categorical-application}

We first test whether practically flow matching with different velocity constraints enables the estimation of commonly used representational distance metrics, as predicted by our analytical results.
We construct a five-condition Gaussian dataset in $\mathbb R^3$, and fit flow matching with the corresponding velocity constraint for each metric in Table~\ref{tab:constraints}. Accuracy, measured by the Pearson correlation between the unique off-diagonal entries of the estimated and ground-truth RDMs, improves with total sample size across all six metrics (Fig.~\ref{fig:discrete}A). The estimated RDMs closely resemble the corresponding ground-truth RDMs defined by the six metrics (Fig.~\ref{fig:discrete}B), supporting our main theoretical claim that flow matching can be used as a unified estimator for these metrics. 
Notably, most of these metrics depend only on first- and second-order moments, for which robust and efficient estimators such as Ledoit--Wolf shrinkage are available \citep{LedoitWolf2004Covariance}. Indeed, when comparing flow matching with these specialized estimators on this simple Gaussian dataset, we did not observe advantages in estimated accuracy for flow matching (Fig.~\ref{fig:supp-categorical-distance-errors}). 
In later applications, we will show that for datasets involving complex distributions and continuous variables, flow matching exhibits various advantages.

\begin{figure}[H]
  \centering
  \includegraphics[width=\linewidth]{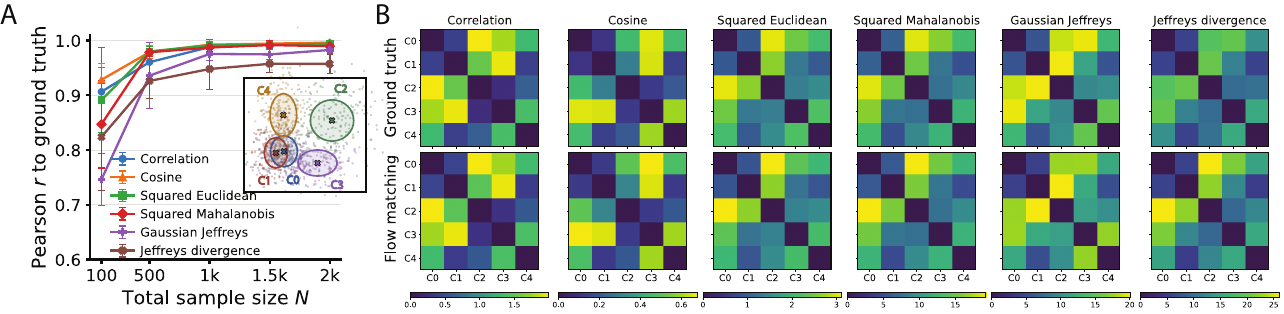}
  \caption{\textbf{Flow matching recovers distance metrics on a synthetic categorical dataset.} (\textbf{A}) Pearson correlation between the unique off-diagonal entries of the estimated and ground-truth RDMs as a function of total sample size $N$. Curves and error bars show means and sample standard deviations over five repetitions; the inset visualizes the five-condition dataset. (\textbf{B}) Ground-truth RDMs (top) and the corresponding flow-matching estimates using 2,000 total samples (bottom).}
  \label{fig:discrete}
\end{figure}

\subsection{Flow matching estimates Fisher information in simulations and neural data}
\label{sec:fisher}
\label{sec:fisher-synthetic}

Fisher information is a local measure of representational distance. \textit{Linear} Fisher information quantifies how precisely small stimulus changes can be estimated by an optimal, locally unbiased linear readout of neural responses \citep{Brunel1998PopulationCoding,Abbott1999CorrelatedVariability,Averbeck2006NeuralCorrelations,MorenoBote2014InformationLimiting,Kanitscheider2015MeasuringFisher,Le_2026}, whereas \textit{full} Fisher information measures the sensitivity of the entire response distribution through the log-likelihood score, including information accessible only through nonlinear readouts. Estimating Fisher information can be challenging, partly because it is difficult to model how response distributions depend on continuous conditions. In contrast, our flow-matching approach can, in principle, interpolate naturally across conditions by approximating the velocity field with a deep neural network, without requiring additional data binning or modeling assumptions. Below, we evaluate the performance of this flow-matching framework in estimating both quantities.

\paragraph{Simulations.}

We first estimate linear Fisher information using a toy dataset whose conditional response distribution is Gaussian given a continuous variable $\theta$. We fit a condition-dependent affine (linear-in-state) velocity field, $v_\theta(x,t)=\dot\beta_t b_\theta+\bar{A}_{\theta,t}(x-\beta_t b_\theta)$ (Table~\ref{tab:constraints}). As baselines, we consider several existing methods. Local linear estimator (OLE) estimates linear Fisher information by binning the data and performing linear classification on adjacent bins \citep{Kanitscheider2015MeasuringFisher}. Gaussian-process kernel regression (GKR; see Section~\ref{sec:gkr}) uses Gaussian-process regression to estimate the distributional mean as a function of $\theta$ and then uses kernel regression to estimate covariances \citep{YeWessel2025GridInformation}. The Wishart process jointly models the mean and covariance with Gaussian processes \citep{Nejatbakhsh2023Wishart} (Section~\ref{sec:wishart-process-upstream}).
On this Gaussian dataset, affine flow yields a smooth estimate that closely follows the ground-truth linear Fisher-information when varying the sample-size and number of neurons(Fig.~\ref{fig:independent-ou-fisher-jeffreys}B, C). GKR and the Wishart process perform fairly well on low-dimensional datasets, but their errors increase substantially with response dimension (Fig.~\ref{fig:independent-ou-fisher-jeffreys}C; also see Pearson correlations in Supplementary Figure~\ref{fig:supp-continuous-benchmark-correlations}A). 

\begin{figure}[!htbp]
    \centering
    \includegraphics[width=0.90\textwidth]{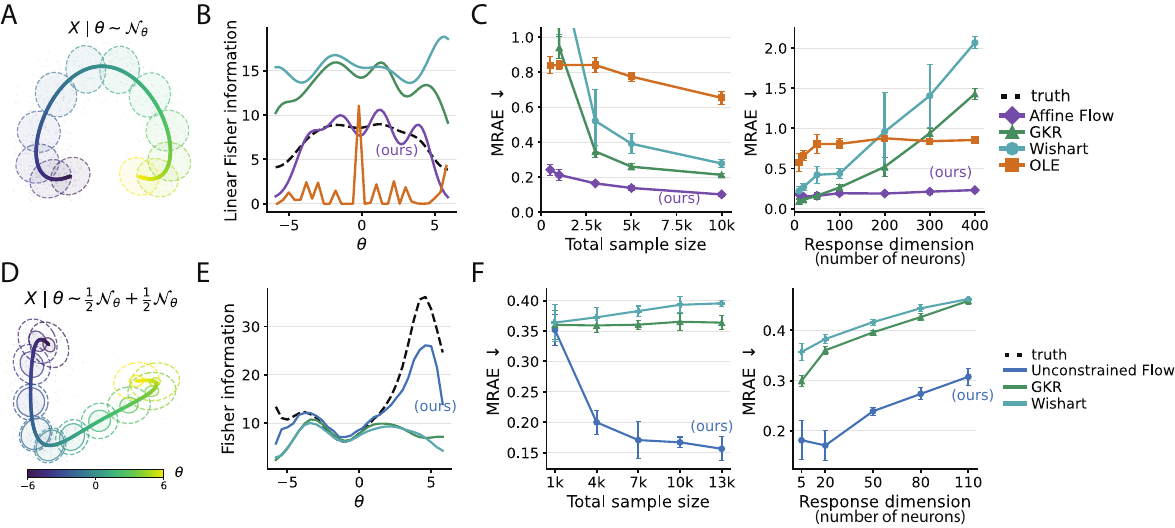}
    \caption{\textbf{Flow matching performs best in estimating Fisher information on toy datasets.} (\textbf{A}) Illustration of the Gaussian conditional-response dataset. (\textbf{B}) Representative linear Fisher-information estimates ($d = 300, N = 1000$). (\textbf{C}) Mean relative absolute error (MRAE) versus total sample size at $d=300$ (left) and versus response dimension at $N=1{,}000$ (right). (\textbf{D}) Equal-weight Gaussian scale-mixture dataset. (\textbf{E}) Representative full Fisher-information estimates. (\textbf{F}) MRAE versus total sample size at $d=20$ (left) and versus response dimension at $N=7{,}000$ (right). Points and error bars in (\textbf{C}) and (\textbf{F}) show means and standard deviations over five repetitions.}
    \label{fig:independent-ou-fisher-jeffreys}
\end{figure}

Furthermore, we evaluate full Fisher information on another Gaussian mixture dataset (Fig.~\ref{fig:independent-ou-fisher-jeffreys}D). GKR and the Wishart process fail in estimation because they assume Gaussian distributions, whereas unconstrained flow successfully targets the full Fisher information (Fig.~\ref{fig:independent-ou-fisher-jeffreys}E and F).

\paragraph{Neural data.}
\label{sec:fisher-neural-data}

Next, we apply flow matching to neural recording data. We analyze six mouse visual-cortex recording sessions from \citet{Stringer2021Geometry}, each containing responses from 10,000 to 25,000 neurons across about 4,000 static-grating trials (Fig.~\ref{fig:stringer}A). Direct estimation of Fisher information in this high-dimensional response space is statistically challenging and computationally demanding. Following common dimensionality-reduction practice \citep{Cunningham2014Dimensionality}, we performed PCA separately for each session using only the training trials and apply the fixed projection to the training, validation, and test responses (200 principal components, about 70\% of the variance). All models are fitted and evaluated in this dimensionality-reduced PCA space.

\begin{figure}[t]
  \centering
  \includegraphics[width=0.9\linewidth]{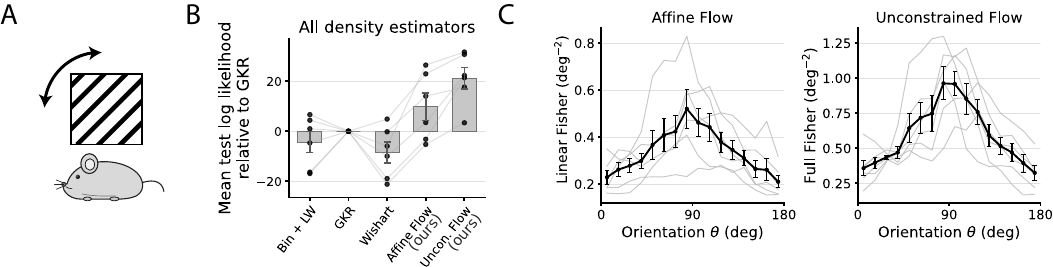}
  \caption{\textbf{Flow matching estimates Fisher information in mouse visual-cortex recordings.} (\textbf{A}) Experimental paradigm: visual-cortex responses are recorded while mice view static gratings. (\textbf{B}) Session-wise mean held-out log likelihood relative to GKR. Gray lines connect estimates from the same session. Bars and error bars show the across-session mean and standard error of the mean (SEM). (\textbf{C}) Linear Fisher information from affine flow and full Fisher information from unconstrained flow versus grating orientation. Gray lines show six individual sessions, each from a different mouse; black curves and error bars show the across-session mean and SEM.}
  \label{fig:stringer}
\end{figure}

For a given orientation range, visualizations of the PC1--PC2 subspace show that, while the Gaussian estimators (Bin + LW, GKR, the Wishart process, and affine flow) generally fit the held-out data well, some discrepancies remain (Supplementary Fig.~\ref{fig:supp-stringer-local-density}). By contrast, unconstrained flow captures a non-elliptical density contour and visually fits the data better. Quantitatively, across the six sessions, affine flow has the highest mean held-out log likelihood among the Gaussian-family estimators. Furthermore, unconstrained flow, which does not impose Gaussianity, has the highest held-out log-likelihood overall in every session (Fig.~\ref{fig:stringer}B). Finally, the affine-flow linear Fisher estimate and the unconstrained-flow full Fisher estimate both peak near $90^\circ$ (Fig.~\ref{fig:stringer}C; see similar Fisher curves estimated by the other methods in Supplementary Fig.~\ref{fig:supp-stringer-density-fisher}), suggesting greater discriminability of horizontal visual patterns. These results are different from the oblique effects reported in humans~\citep{Appelle1972,Furmanski2000ObliqueEffect} and macaques~\citep{Li2003oblique}, which show higher discriminability for both horizontal and vertical orientations, yet are consistent with a recent study in mice \citep{Dipoppa2024Adaptation}.

\subsection{Estimating time-resolved distances using flow matching}
\label{sec:time-resolved}
\label{sec:time-resolved-synthetic}

Neural population representations are not static, rather they evolve dynamically over time. Understanding such dynamics is important for understanding the computations performed by the brain. Time-resolved RDMs address this question by tracking how representational distances change over time. Conventional estimation typically involves dividing the data into discrete time bins. Our flow-matching framework does not require binning, as the velocity neural network learns to interpolate across time. In this section, we evaluate whether our method can estimate time-resolved RDMs.

\begin{figure}[t]
  \centering
  \includegraphics[width=0.9\linewidth]{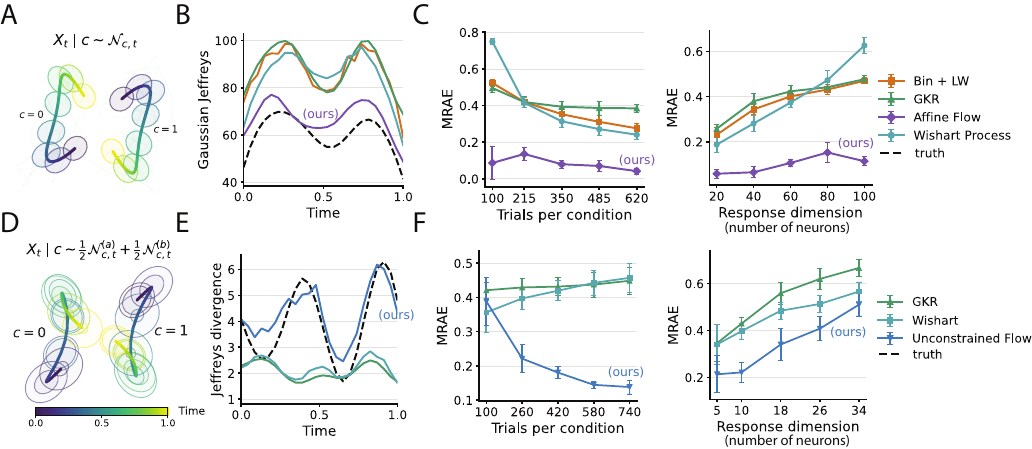}
  \caption{\textbf{Flow matching best estimates time-resolved distances on OU-trial toy datasets.} (\textbf{A}) Gaussian OU dataset with two conditions. (\textbf{B}) Representative time-resolved Gaussian Jeffreys-divergence estimates at $d=60$ and 260 trials per condition. (\textbf{C}) MRAE versus trials per condition at $d=60$ (left) and versus response dimension at 260 trials per condition (right). (\textbf{D}) Gaussian-mixture OU dataset. (\textbf{E}) Representative full Jeffreys-divergence estimates at $d=10$ and 260 trials per condition. (\textbf{F}) MRAE versus trials per condition at $d=10$ (left) and versus response dimension at 260 trials per condition (right). Points and error bars in (\textbf{C}) and (\textbf{F}) show means and standard deviations over five repetitions.}
  \label{fig:ou-time-resolved-jeffreys}
\end{figure}

\paragraph{Simulations.}

We construct two-condition toy datasets with temporally correlated Ornstein--Uhlenbeck (OU) trials \citep{UhlenbeckOrnstein1930Brownian} and either Gaussian or two-scale Gaussian-mixture marginals. 
We measure pairwise distances using either Gaussian Jeffreys divergence or Jeffreys divergence. Gaussian Jeffreys divergence generalizes Mahalanobis distance by allowing the two conditions to have different covariance matrices, whereas Jeffreys divergence makes no Gaussian assumption and can therefore capture differences between their full distributions.
We find that, on the Gaussian dataset, affine flow outperforms conventional methods in estimating Gaussian Jeffreys divergence throughout both the trial-count and response-dimension sweeps (Fig.~\ref{fig:ou-time-resolved-jeffreys}A--C)
On the mixture dataset, only the unconstrained flow-matching estimator captures the non-Gaussian structure and outperforms the other methods in estimating Jeffreys divergence (Fig.~\ref{fig:ou-time-resolved-jeffreys}D--F).

\paragraph{Neural data.}
\label{sec:time-resolved-neural-data}

We next apply time-resolved distance estimation to extracellular spiking activity from the Allen Visual Coding Neuropixels dataset \citep{Siegle2021Survey}. Each drifting-grating presentation is analyzed over four seconds: one second before stimulus onset, two seconds of stimulus presentation, and one second after stimulus offset (Fig.~\ref{fig:time-resolved-jeffreys}A). We treat each combination of motion direction and temporal frequency as a condition and examine how the pairwise representational distances evolve over time. We analyze 32 recording sessions, one for per mouse. Spikes are counted in nonoverlapping 100-ms bins, transformed elementwise using a variance-stablizing transformation as $k\mapsto\sqrt{k}$, and projected onto the top 70 principal components, which explains about 70\% variance.

Quantitatively, affine flow leads to the highest mean held-out log-likelihood among the Gaussian estimators, whereas unconstrained flow has the highest mean held-out log-likelihood overall (Fig.~\ref{fig:time-resolved-jeffreys}B; see example data visualizations in Supplementary Figs.~\ref{fig:supp-allen-ephys-density-trajectories} and~\ref{fig:supp-allen-ephys-density-fits}). Similar results hold when using 10 rather than 70 principal components (Supplementary Fig.~\ref{fig:supp-allen-pca10-five-session-likelihood}). Across all five temporal frequencies, both flow estimates of Jeffreys divergence rise sharply after stimulus onset and fall after stimulus offset. Higher temporal frequencies show smaller overall distances, suggesting poorer discrimination between grating directions and less directional information encoded at faster temporal frequencies (Fig.~\ref{fig:time-resolved-jeffreys}C; also see Supplementary Figs.~\ref{fig:supp-allen-ephys-jeffreys-timecourses} and~\ref{fig:supp-allen-ephys-all-condition-rdms}).
Individual RDMs exhibit semi-periodic structures, implying that therepresentational distances between drifting gratings of the opposite directions are closer compared to gratings with orthogonal directions (Fig.~\ref{fig:time-resolved-jeffreys}C bottom; also see Supplementary Fig.~\ref{fig:supp-allen-two-photon-summary} for a similar pattern from analyzing calcium imaging data from \citet{deVries2020VisualCoding}). %

Using the RDMs at different time points, we can also assess neural representational stability over time by computing time-to-time RDM correlations (Fig.~\ref{fig:time-resolved-jeffreys}D; see Supplementary Fig.~\ref{fig:supp-allen-rdm-time-correlations} for the other methods). The mean correlation is higher during stimulus presentation (0--2 s; mean = 0.73) than before ($-1$--0 s; mean = 0.055) or after (2--3 s; mean = 0.25). This suggests that the structure of neural representational geometry are relatively stable during stimulus presentation. Together with observations in Fig.~\ref{fig:time-resolved-jeffreys}C,
these results suggest that, while the neural manifolds expand and shrink over time, the shape of the manifolds is generally preserved.

We also find that methods such as GKR tend to oversmooth the distribution, perhaps because they use Gaussian processes and assume Gaussian response distributions, producing spuriously high correlations of RDMs even before stimulus presentation (Fig.~\ref{fig:time-resolved-jeffreys}D).

\begin{figure}[t]
  \centering
  \includegraphics[width=0.8\linewidth]{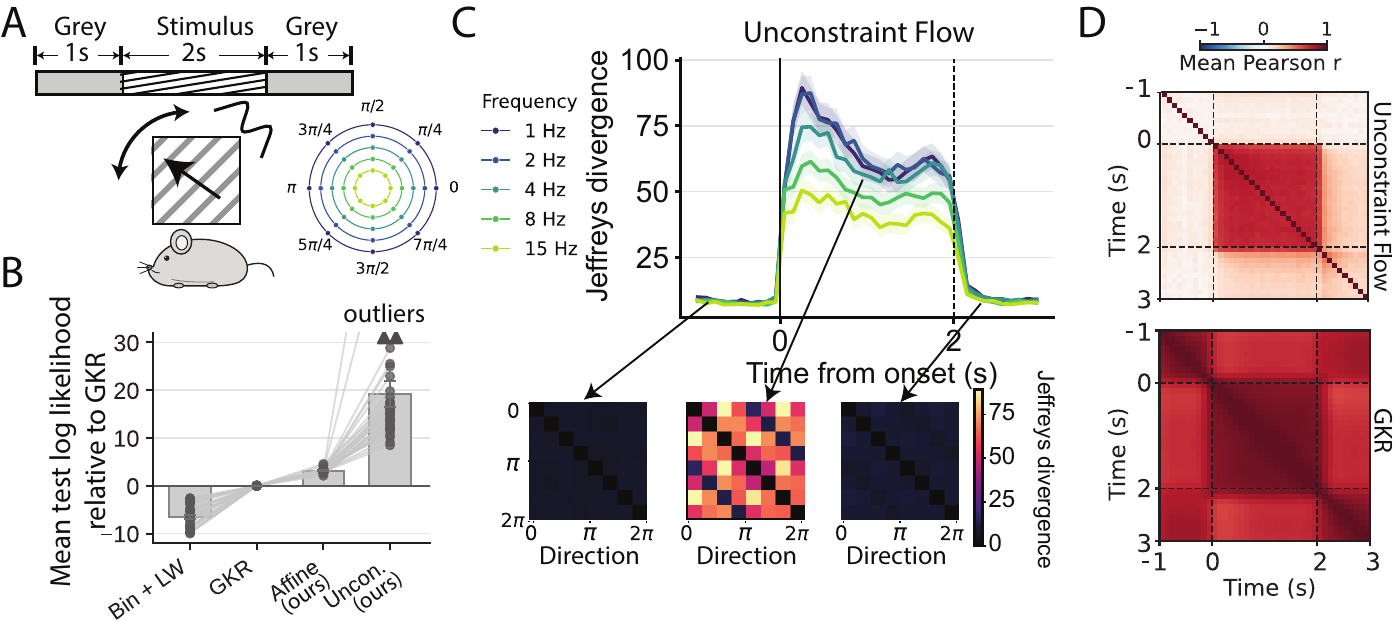}
  \caption{\textbf{Flow matching estimates time-resolved distances from Allen Neuropixels recordings.} (\textbf{A}) Four-second trial structure. (\textbf{B}) Mean test log-likelihood relative to the GKR method. Dots denote recording sessions, gray lines connect estimates from the same session, and bars show the mean across 32 sessions; triangles indicate two outliers above the displayed range. (\textbf{C}) Unconstrained-flow Jeffreys divergence averaged over the 28 unique pairs among the eight directions separately at each temporal frequency within each session (top), and 4-Hz direction RDMs averaged across sessions at the bin centers nearest $-0.5$, $1.0$, and $2.5$ seconds (bottom). Colors indicate temporal frequency. Curves and shaded bands show the across-session mean and SEM, respectively. (\textbf{D}) Time-to-time RDM correlations. Colors show mean Pearson correlations between RDMs at pairs of time points. 
  }
  \label{fig:time-resolved-jeffreys}
\end{figure}

\FloatBarrier
\section{Discussion and Conclusion}
\label{sec:discussion}

We have introduced a flow-matching framework that unifies neural representational dissimilarity metrics as Jeffreys divergences under different velocity-field constraints. It provides a common estimator with empirical advantages for high-dimensional data and complex distributions conditioned on continuous variables, and enables principled design of new metrics through source distributions and velocity constraints.
Our results has several limitations. Because the theoretical correspondences hold at global population optima, which neural-network training may not reach, the fitted estimates can therefore deviate from their intended metric interpretations. Also, our Validation has focused on visual-cortex recordings. Extending it to other brain systems and measurement techniques (such as fMRI and EEG), as well as deep neural network models, remains an important direction for future investigations. 

\subsection*{AI use statement}
AI-assisted tools were used for language editing and manuscript organization under author direction. The authors reviewed the scientific claims, derivations, citations, and conclusions and take responsibility for the final content.

\subsection*{Ethics statement}
This work analyzes synthetic data and previously collected animal recordings. No new animal or human-subject experiments were conducted for this study.

\subsection*{Reproducibility statement}
The Supplementary Material provides complete proofs, estimator specifications, dataset construction and preprocessing details, hyperparameter searches, and additional results. The real-data experiments use the visual-cortex recordings described by \citet{Stringer2021Geometry}, two-photon calcium-imaging sessions from the Allen Brain Observatory Visual Coding dataset, and Allen Visual Coding Neuropixels recordings \citep{Siegle2021Survey}. Code is available at \url{https://github.com/AgeYY/FlowRDM}.

\clearpage
\bibliographystyle{plainnat}
\bibliography{representational_flow_matching_arxiv}

\begin{thebibliography}{50}
\providecommand{\natexlab}[1]{#1}
\providecommand{\url}[1]{\texttt{#1}}
\expandafter\ifx\csname urlstyle\endcsname\relax
  \providecommand{\doi}[1]{doi: #1}\else
  \providecommand{\doi}{doi: \begingroup \urlstyle{rm}\Url}\fi

\bibitem[Abbott and Dayan(1999)]{Abbott1999CorrelatedVariability}
L.~F. Abbott and Peter Dayan.
\newblock The effect of correlated variability on the accuracy of a population code.
\newblock \emph{Neural Computation}, 11\penalty0 (1):\penalty0 91--101, 1999.
\newblock \doi{10.1162/089976699300016827}.

\bibitem[Albergo and Vanden-Eijnden(2023)]{Albergo2023StochasticInterpolants}
Michael~S. Albergo and Eric Vanden-Eijnden.
\newblock Building normalizing flows with stochastic interpolants.
\newblock In \emph{International Conference on Learning Representations}, 2023.

\bibitem[Amari(2016)]{Amari2016InformationGeometry}
Shun-ichi Amari.
\newblock \emph{Information Geometry and Its Applications}.
\newblock Springer Japan, 2016.
\newblock \doi{10.1007/978-4-431-55978-8}.

\bibitem[Appelle(1972)]{Appelle1972}
Stuart Appelle.
\newblock Perception and discrimination as a function of stimulus orientation: the" oblique effect" in man and animals.
\newblock \emph{Psychological bulletin}, 78\penalty0 (4):\penalty0 266, 1972.

\bibitem[Averbeck et~al.(2006)Averbeck, Latham, and Pouget]{Averbeck2006NeuralCorrelations}
Bruno~B. Averbeck, Peter~E. Latham, and Alexandre Pouget.
\newblock Neural correlations, population coding and computation.
\newblock \emph{Nature Reviews Neuroscience}, 7\penalty0 (5):\penalty0 358--366, 2006.
\newblock \doi{10.1038/nrn1888}.

\bibitem[Bruinsma et~al.(2020)Bruinsma, Perim, Tebbutt, Hosking, Solin, and Turner]{Bruinsma2020ScalableGP}
Wessel Bruinsma, Eric Perim, William Tebbutt, Scott Hosking, Arno Solin, and Richard Turner.
\newblock Scalable exact inference in multi-output gaussian processes.
\newblock In \emph{Proceedings of the 37th International Conference on Machine Learning}, volume 119 of \emph{Proceedings of Machine Learning Research}, pages 1190--1201, 2020.

\bibitem[Brunel and Nadal(1998)]{Brunel1998PopulationCoding}
Nicolas Brunel and Jean-Pierre Nadal.
\newblock Mutual information, fisher information, and population coding.
\newblock \emph{Neural Computation}, 10\penalty0 (7):\penalty0 1731--1757, 1998.
\newblock \doi{10.1162/089976698300017115}.

\bibitem[Chen et~al.(2018)Chen, Rubanova, Bettencourt, and Duvenaud]{Chen2018NeuralODE}
Ricky T.~Q. Chen, Yulia Rubanova, Jesse Bettencourt, and David~K. Duvenaud.
\newblock Neural ordinary differential equations.
\newblock In \emph{Advances in Neural Information Processing Systems}, volume~31, 2018.

\bibitem[Cichy et~al.(2014)Cichy, Pantazis, and Oliva]{Cichy2014Resolving}
Radoslaw~Martin Cichy, Dimitrios Pantazis, and Aude Oliva.
\newblock Resolving human object recognition in space and time.
\newblock \emph{Nature Neuroscience}, 17:\penalty0 455--462, 2014.
\newblock \doi{10.1038/nn.3635}.

\bibitem[Cunningham and Yu(2014)]{Cunningham2014Dimensionality}
John~P. Cunningham and Byron~M. Yu.
\newblock Dimensionality reduction for large-scale neural recordings.
\newblock \emph{Nature Neuroscience}, 17\penalty0 (11):\penalty0 1500--1509, 2014.
\newblock \doi{10.1038/nn.3776}.

\bibitem[Davtyan et~al.(2023)Davtyan, Sameni, and Favaro]{Davtyan2023VideoFlow}
Aram Davtyan, Sepehr Sameni, and Paolo Favaro.
\newblock Efficient video prediction via sparsely conditioned flow matching.
\newblock In \emph{Proceedings of the IEEE/CVF International Conference on Computer Vision}, pages 23263--23274, 2023.

\bibitem[de~Vries et~al.(2020)]{deVries2020VisualCoding}
Saskia E.~J. de~Vries et~al.
\newblock A large-scale standardized physiological survey reveals functional organization of the mouse visual cortex.
\newblock \emph{Nature Neuroscience}, 23:\penalty0 138--151, 2020.
\newblock \doi{10.1038/s41593-019-0550-9}.

\bibitem[Diedrichsen and Kriegeskorte(2017)]{Diedrichsen2017RepresentationalModels}
J{\"o}rn Diedrichsen and Nikolaus Kriegeskorte.
\newblock Representational models: A common framework for understanding encoding, pattern-component, and representational-similarity analysis.
\newblock \emph{PLOS Computational Biology}, 13\penalty0 (4):\penalty0 e1005508, 2017.
\newblock \doi{10.1371/journal.pcbi.1005508}.

\bibitem[Dipoppa et~al.(2024)Dipoppa, Nogueira, Bugeon, Friedman, Reddy, Harris, Ringach, Miller, Carandini, and Fusi]{Dipoppa2024Adaptation}
Mario Dipoppa, Ramon Nogueira, St{\'e}phane Bugeon, Yoni Friedman, Charu~B. Reddy, Kenneth~D. Harris, Dario~L. Ringach, Kenneth~D. Miller, Matteo Carandini, and Stefano Fusi.
\newblock Adaptation shapes the representational geometry in mouse {V1} to efficiently encode the environment.
\newblock \emph{bioRxiv}, 2024.
\newblock \doi{10.1101/2024.12.11.628035}.

\bibitem[Esser et~al.(2024)Esser, Kulal, Blattmann, Entezari, M{"u}ller, Saini, Levi, Lorenz, Sauer, Boesel, Podell, Dockhorn, English, and Rombach]{Esser2024RectifiedFlow}
Patrick Esser, Sumith Kulal, Andreas Blattmann, Rahim Entezari, Jonas M{"u}ller, Harry Saini, Yam Levi, Dominik Lorenz, Axel Sauer, Frederic Boesel, Dustin Podell, Tim Dockhorn, Zion English, and Robin Rombach.
\newblock Scaling rectified flow transformers for high-resolution image synthesis.
\newblock In \emph{Proceedings of the 41st International Conference on Machine Learning}, volume 235, pages 12606--12633, 2024.

\bibitem[Fisher(1922)]{Fisher1922Foundations}
Ronald~A. Fisher.
\newblock On the mathematical foundations of theoretical statistics.
\newblock \emph{Philosophical Transactions of the Royal Society of London. Series A}, 222:\penalty0 309--368, 1922.
\newblock \doi{10.1098/rsta.1922.0009}.

\bibitem[Furmanski and Engel(2000)]{Furmanski2000ObliqueEffect}
Christopher~S. Furmanski and Stephen~A. Engel.
\newblock An oblique effect in human primary visual cortex.
\newblock \emph{Nature Neuroscience}, 3\penalty0 (6):\penalty0 535--536, 2000.
\newblock \doi{10.1038/75702}.

\bibitem[Grathwohl et~al.(2019)Grathwohl, Chen, Bettencourt, Sutskever, and Duvenaud]{Grathwohl2019FFJORD}
Will Grathwohl, Ricky T.~Q. Chen, Jesse Bettencourt, Ilya Sutskever, and David~K. Duvenaud.
\newblock {FFJORD}: Free-form continuous dynamics for scalable reversible generative models.
\newblock In \emph{International Conference on Learning Representations}, 2019.

\bibitem[Hall(2015)]{Hall2015LieGroups}
Brian~C. Hall.
\newblock \emph{Lie Groups, Lie Algebras, and Representations: An Elementary Introduction}.
\newblock Springer, 2 edition, 2015.
\newblock \doi{10.1007/978-3-319-13467-3}.

\bibitem[Hu et~al.(2024)Hu, Wu, Asano, Mettes, Fernando, Ommer, and Snoek]{Hu2024TextFlow}
Vincent Hu, Di~Wu, Yuki Asano, Pascal Mettes, Basura Fernando, Bj{"o}rn Ommer, and Cees Snoek.
\newblock Flow matching for conditional text generation in a few sampling steps.
\newblock In \emph{Proceedings of the 18th Conference of the European Chapter of the Association for Computational Linguistics (Volume 2: Short Papers)}, pages 380--392. Association for Computational Linguistics, 2024.
\newblock \doi{10.18653/v1/2024.eacl-short.33}.

\bibitem[Kanitscheider et~al.(2015)Kanitscheider, Coen-Cagli, Kohn, and Pouget]{Kanitscheider2015MeasuringFisher}
Ingmar Kanitscheider, Ruben Coen-Cagli, Adam Kohn, and Alexandre Pouget.
\newblock Measuring fisher information accurately in correlated neural populations.
\newblock \emph{PLOS Computational Biology}, 11\penalty0 (6):\penalty0 e1004218, 2015.
\newblock \doi{10.1371/journal.pcbi.1004218}.

\bibitem[Kohn et~al.(2016)Kohn, Coen-Cagli, Kanitscheider, and Pouget]{Kohn2016CorrelationsInformation}
Adam Kohn, Ruben Coen-Cagli, Ingmar Kanitscheider, and Alexandre Pouget.
\newblock Correlations and neuronal population information.
\newblock \emph{Annual Review of Neuroscience}, 39:\penalty0 237--256, 2016.
\newblock \doi{10.1146/annurev-neuro-070815-013851}.

\bibitem[Kornblith et~al.(2019)Kornblith, Norouzi, Lee, and Hinton]{Kornblith2019Similarity}
Simon Kornblith, Mohammad Norouzi, Honglak Lee, and Geoffrey Hinton.
\newblock Similarity of neural network representations revisited.
\newblock In \emph{Proceedings of the 36th International Conference on Machine Learning}, pages 3519--3529, 2019.

\bibitem[Kriegeskorte(2009)]{Kriegeskorte2009PopulationCodes}
Nikolaus Kriegeskorte.
\newblock Relating population-code representations between man, monkey, and computational models.
\newblock \emph{Frontiers in Neuroscience}, 3\penalty0 (3):\penalty0 363--373, 2009.
\newblock \doi{10.3389/neuro.01.035.2009}.

\bibitem[Kriegeskorte and Kievit(2013)]{KriegeskorteKievit2013Geometry}
Nikolaus Kriegeskorte and Rogier~A. Kievit.
\newblock Representational geometry: Integrating cognition, computation, and the brain.
\newblock \emph{Trends in Cognitive Sciences}, 17\penalty0 (8):\penalty0 401--412, 2013.
\newblock \doi{10.1016/j.tics.2013.06.007}.

\bibitem[Kriegeskorte and Wei(2021)]{KriegeskorteWei2021TuningGeometry}
Nikolaus Kriegeskorte and Xue-Xin Wei.
\newblock Neural tuning and representational geometry.
\newblock \emph{Nature Reviews Neuroscience}, 22\penalty0 (11):\penalty0 703--718, 2021.
\newblock \doi{10.1038/s41583-021-00502-3}.

\bibitem[Kriegeskorte et~al.(2008{\natexlab{a}})Kriegeskorte, Mur, and Bandettini]{Kriegeskorte2008RSA}
Nikolaus Kriegeskorte, Marieke Mur, and Peter~A. Bandettini.
\newblock Representational similarity analysis: connecting the branches of systems neuroscience.
\newblock \emph{Frontiers in Systems Neuroscience}, 2:\penalty0 4, 2008{\natexlab{a}}.
\newblock \doi{10.3389/neuro.06.004.2008}.

\bibitem[Kriegeskorte et~al.(2008{\natexlab{b}})Kriegeskorte, Mur, Ruff, Kiani, Bodurka, Esteky, Tanaka, and Bandettini]{Kriegeskorte2008Categorical}
Nikolaus Kriegeskorte, Marieke Mur, Douglas~A. Ruff, Roozbeh Kiani, Jerzy Bodurka, Hossein Esteky, Keiji Tanaka, and Peter~A. Bandettini.
\newblock Matching categorical object representations in inferior temporal cortex of man and monkey.
\newblock \emph{Neuron}, 60\penalty0 (6):\penalty0 1126--1141, 2008{\natexlab{b}}.
\newblock \doi{10.1016/j.neuron.2008.10.043}.

\bibitem[Kullback and Leibler(1951)]{Kullback1951Information}
Solomon Kullback and Richard~A. Leibler.
\newblock On information and sufficiency.
\newblock \emph{The Annals of Mathematical Statistics}, 22\penalty0 (1):\penalty0 79--86, 1951.
\newblock \doi{10.1214/aoms/1177729694}.

\bibitem[Le and Wei(2026)]{Le_2026}
Dylan Le and Xue-Xin Wei.
\newblock Split-trial analysis reveals the information capacity of neural population codes.
\newblock \emph{eLife}, 2026.
\newblock \doi{10.7554/elife.111658.1}.

\bibitem[Ledoit and Wolf(2004)]{LedoitWolf2004Covariance}
Olivier Ledoit and Michael Wolf.
\newblock A well-conditioned estimator for large-dimensional covariance matrices.
\newblock \emph{Journal of Multivariate Analysis}, 88\penalty0 (2):\penalty0 365--411, 2004.
\newblock \doi{10.1016/S0047-259X(03)00096-4}.

\bibitem[Li et~al.(2003)Li, Peterson, and Freeman]{Li2003oblique}
Baowang Li, Matthew~R Peterson, and Ralph~D Freeman.
\newblock Oblique effect: a neural basis in the visual cortex.
\newblock \emph{Journal of neurophysiology}, 90\penalty0 (1):\penalty0 204--217, 2003.

\bibitem[Lipman et~al.(2023)Lipman, Chen, Ben-Hamu, Nickel, and Le]{Lipman2023FlowMatching}
Yaron Lipman, Ricky T.~Q. Chen, Heli Ben-Hamu, Maximilian Nickel, and Matt Le.
\newblock Flow matching for generative modeling.
\newblock In \emph{International Conference on Learning Representations}, 2023.

\bibitem[Loshchilov and Hutter(2019)]{LoshchilovHutter2019AdamW}
Ilya Loshchilov and Frank Hutter.
\newblock Decoupled weight decay regularization.
\newblock In \emph{International Conference on Learning Representations}, 2019.

\bibitem[Moreno-Bote et~al.(2014)Moreno-Bote, Beck, Kanitscheider, Pitkow, Latham, and Pouget]{MorenoBote2014InformationLimiting}
Rub{\'e}n Moreno-Bote, Jeffrey Beck, Ingmar Kanitscheider, Xaq Pitkow, Peter Latham, and Alexandre Pouget.
\newblock Information-limiting correlations.
\newblock \emph{Nature Neuroscience}, 17\penalty0 (10):\penalty0 1410--1417, 2014.
\newblock \doi{10.1038/nn.3807}.

\bibitem[Nejatbakhsh et~al.(2023)Nejatbakhsh, Garon, and Williams]{Nejatbakhsh2023Wishart}
Amin Nejatbakhsh, Isabel Garon, and Alex Williams.
\newblock Estimating noise correlations across continuous conditions with wishart processes.
\newblock In \emph{Advances in Neural Information Processing Systems}, volume~36, pages 54032--54045, 2023.

\bibitem[Nili et~al.(2014)Nili, Wingfield, Walther, Su, Marslen-Wilson, and Kriegeskorte]{Nili2014RSAToolbox}
Hamed Nili, Cai Wingfield, Alexander Walther, Li~Su, William Marslen-Wilson, and Nikolaus Kriegeskorte.
\newblock A toolbox for representational similarity analysis.
\newblock \emph{PLOS Computational Biology}, 10\penalty0 (4):\penalty0 e1003553, 2014.
\newblock \doi{10.1371/journal.pcbi.1003553}.

\bibitem[Papamakarios et~al.(2021)Papamakarios, Nalisnick, Rezende, Mohamed, and Lakshminarayanan]{Papamakarios2021NormalizingFlows}
George Papamakarios, Eric Nalisnick, Danilo~Jimenez Rezende, Shakir Mohamed, and Balaji Lakshminarayanan.
\newblock Normalizing flows for probabilistic modeling and inference.
\newblock \emph{Journal of Machine Learning Research}, 22\penalty0 (57):\penalty0 1--64, 2021.

\bibitem[Pouget et~al.(2000)Pouget, Dayan, and Zemel]{Pouget2000PopulationCodes}
Alexandre Pouget, Peter Dayan, and Richard Zemel.
\newblock Information processing with population codes.
\newblock \emph{Nature Reviews Neuroscience}, 1\penalty0 (2):\penalty0 125--132, 2000.
\newblock \doi{10.1038/35039062}.

\bibitem[Rao(1945)]{Rao1945Information}
C.~R. Rao.
\newblock Information and accuracy attainable in the estimation of statistical parameters.
\newblock \emph{Bulletin of the Calcutta Mathematical Society}, 37\penalty0 (3):\penalty0 81--91, 1945.

\bibitem[Rasmussen and Williams(2006)]{RasmussenWilliams2006GaussianProcesses}
Carl~Edward Rasmussen and Christopher K.~I. Williams.
\newblock \emph{Gaussian Processes for Machine Learning}.
\newblock MIT Press, 2006.
\newblock \doi{10.7551/mitpress/3206.001.0001}.

\bibitem[Rezende and Mohamed(2015)]{Rezende2015Variational}
Danilo~Jimenez Rezende and Shakir Mohamed.
\newblock Variational inference with normalizing flows.
\newblock In \emph{Proceedings of the 32nd International Conference on Machine Learning}, pages 1530--1538, 2015.

\bibitem[Rhodes et~al.(2020)Rhodes, Xu, and Gutmann]{Rhodes2020TRE}
Benjamin Rhodes, Kai Xu, and Michael~U. Gutmann.
\newblock Telescoping density-ratio estimation.
\newblock In \emph{Advances in Neural Information Processing Systems}, volume~33, 2020.

\bibitem[Seri{\`e}s et~al.(2009)Seri{\`e}s, Stocker, and Simoncelli]{Series2009Homunculus}
Peggy Seri{\`e}s, Alan~A. Stocker, and Eero~P. Simoncelli.
\newblock Is the homunculus ``aware'' of sensory adaptation?
\newblock \emph{Neural Computation}, 21\penalty0 (12):\penalty0 3271--3304, 2009.
\newblock \doi{10.1162/neco.2009.09-08-869}.

\bibitem[Seung and Sompolinsky(1993)]{SeungSompolinsky1993PopulationCodes}
H.~S. Seung and Haim Sompolinsky.
\newblock Simple models for reading neuronal population codes.
\newblock \emph{Proceedings of the National Academy of Sciences}, 90\penalty0 (22):\penalty0 10749--10753, 1993.
\newblock \doi{10.1073/pnas.90.22.10749}.

\bibitem[Siegle et~al.(2021)]{Siegle2021Survey}
Joshua~H. Siegle et~al.
\newblock Survey of spiking in the mouse visual system reveals functional hierarchy.
\newblock \emph{Nature}, 592:\penalty0 86--92, 2021.
\newblock \doi{10.1038/s41586-020-03171-x}.

\bibitem[Stringer et~al.(2021)Stringer, Michaelos, Tsyboulski, Lindo, and Pachitariu]{Stringer2021Geometry}
Carsen Stringer, Michalis Michaelos, Dmitri Tsyboulski, Sarah~E. Lindo, and Marius Pachitariu.
\newblock High-precision coding in visual cortex.
\newblock \emph{Cell}, 184\penalty0 (10):\penalty0 2767--2778.e15, 2021.
\newblock \doi{10.1016/j.cell.2021.03.042}.

\bibitem[Uhlenbeck and Ornstein(1930)]{UhlenbeckOrnstein1930Brownian}
George~E. Uhlenbeck and Leonard~S. Ornstein.
\newblock On the theory of the brownian motion.
\newblock \emph{Physical Review}, 36\penalty0 (5):\penalty0 823--841, 1930.
\newblock \doi{10.1103/PhysRev.36.823}.

\bibitem[Walther et~al.(2016)Walther, Nili, Ejaz, Alink, Kriegeskorte, and Diedrichsen]{Walther2016Reliability}
Alexander Walther, Hamed Nili, Naveed Ejaz, Arjen Alink, Nikolaus Kriegeskorte, and J{\"o}rn Diedrichsen.
\newblock Reliability of dissimilarity measures for multi-voxel pattern analysis.
\newblock \emph{NeuroImage}, 137:\penalty0 188--200, 2016.
\newblock \doi{10.1016/j.neuroimage.2015.12.012}.

\bibitem[Ye and Wessel(2025)]{YeWessel2025GridInformation}
Zeyuan Ye and Ralf Wessel.
\newblock Speed modulations in grid cell information geometry.
\newblock \emph{Nature Communications}, 16:\penalty0 7723, 2025.
\newblock \doi{10.1038/s41467-025-62856-x}.

\end{thebibliography}

\clearpage
\appendix
\begin{center}
  {\LARGE\bfseries Supplementary Material}
\end{center}
\vspace{1em}
\startcontents[appendix]
\section*{Supplementary Material Contents}
\printcontents[appendix]{}{1}{\setcounter{tocdepth}{3}}

\section{Mathematical Proofs of the Metric Correspondences under flow matching}
\label{app:derivations}

In this section, we show that flow-matching training under different velocity constraints yields Jeffreys divergences that correspond to different existing distance metrics. Optional log-likelihood fine-tuning recovers the same distance metrics as well, as shown in Supplementary Section~\ref{app:nll-derivations}.

\subsection{Assumptions and notation}

We use the term representational \emph{distance} in the sense common in neural data analysis: it is symmetric and nonnegative but need not obey the triangle inequality.

We prove the flow-matching population statements in Table~\ref{tab:constraints}. Unless stated otherwise, the source is $\rho_0=\mathcal N(0,I)$, $X_0\sim \rho_0$, $X_1^{(c)}\sim p_c$, and $X_0$ is independent of $X_1^{(c)}$. Each modeled condition has $\pi_{\mathcal C}(c)>0$, and each $p_c$ has finite second moments. Covariance matrices are invertible. The schedules are continuously differentiable and satisfy $\alpha_0=1$, $\alpha_1=0$, $\beta_0=0$, and $\beta_1=1$. For condition $c$, write
\begin{equation}
    X_t^{(c)}=\alpha_tX_0+\beta_tX_1^{(c)},\qquad
    U_t^{(c)}=\dot\alpha_tX_0+\dot\beta_tX_1^{(c)} .
\end{equation}
For a velocity class $\mathcal V$, let $v^{\mathcal V}=\arg\min_{v\in\mathcal V}\mathcal L_{\mathrm{FM}}(v)$, with $\mathcal L_{\mathrm{FM}}$ defined in Eq.~\eqref{eq:fm-objective}. Inserting $v^{\mathcal V}$ into the ODE produces $q_c^{\mathcal V}$ for each generic condition $c$. Pairwise results compare $q_a^{\mathcal V}$ and $q_b^{\mathcal V}$. The unknown response distributions $p_c$ need not be Gaussian.

Before proving the correspondences, we provide two formulas used repeatedly
below: the affine-flow endpoint distribution obtained by setting $t=1$ in
Eq.~\eqref{eq:affine-flow-gaussian-law}
(Subsection~\ref{app:affine-flow-gaussianity}) and the closed-form Jeffreys
divergence between Gaussian endpoints in
Eq.~\eqref{eq:affine-gaussian-jeffreys-evaluation}
(Subsection~\ref{app:affine-gaussian-evaluation}).

\subsection{Affine-flow solution}
\label{app:affine-flow-gaussianity}

\begin{lemma}[Affine flows preserve Gaussianity]
\label{lem:affine-flow-gaussian}
Let $A_t\in\mathbb R^{d\times d}$ and $r_t\in\mathbb R^d$ be continuous on $[0,1]$, and consider the affine ODE
\begin{equation}
    \dot X_t=A_tX_t+r_t,
    \qquad
    X_0\sim\mathcal N(m_0,S_0).
\end{equation}
Let $\Phi_t$ solve $\dot\Phi_t=A_t\Phi_t$ with $\Phi_0=I$, and let $\eta_t$ solve $\dot\eta_t=A_t\eta_t+r_t$ with $\eta_0=0$. Then
\begin{equation}
    X_t=\Phi_tX_0+\eta_t
    \label{eq:affine-flow-solution}
\end{equation}
and therefore
\begin{equation}
    X_t
    \sim
    \mathcal N\!\left(
        \Phi_tm_0+\eta_t,
        \Phi_tS_0\Phi_t^\top
    \right)
    \label{eq:affine-flow-gaussian-law}
\end{equation}
for every $t\in[0,1]$.
\end{lemma}

\begin{proof}
Differentiating the right-hand side of Eq.~\eqref{eq:affine-flow-solution} shows that it satisfies the affine ODE and the initial condition $X_0$. Uniqueness of the linear ODE therefore gives Eq.~\eqref{eq:affine-flow-solution}. This is an affine transformation of the Gaussian random vector $X_0$, so Eq.~\eqref{eq:affine-flow-gaussian-law} follows.
\end{proof}

\subsection{Deterministic Jeffreys evaluation for affine Gaussian flows}
\label{app:affine-gaussian-evaluation}

Suppose the fitted velocity for condition $c$ is affine in the state,
\begin{equation}
    v_c^{\mathcal V}(x,t)=A_{c,t}x+r_{c,t},
\end{equation}
and the source is $\rho_0=\mathcal N(m_0,S_0)$. By
Lemma~\ref{lem:affine-flow-gaussian}, the distribution remains Gaussian along
the flow. Its mean $m_{c,t}=\mathbb E[X_t]$ and covariance
$S_{c,t}=\operatorname{Cov}(X_t)$ satisfy
\begin{align}
    \dot m_{c,t}
    &=
    A_{c,t}m_{c,t}+r_{c,t},
    &m_{c,0}&=m_0, \notag\\
    \dot S_{c,t}
    &=
    A_{c,t}S_{c,t}+S_{c,t}A_{c,t}^{\top},
    &S_{c,0}&=S_0.
    \label{eq:affine-moment-odes}
\end{align}
The first equation follows by taking the expectation of
$\dot X_t=A_{c,t}X_t+r_{c,t}$. For the second, write
$Y_t=X_t-m_{c,t}$, so that $\dot Y_t=A_{c,t}Y_t$, and differentiate
$S_{c,t}=\mathbb E[Y_tY_t^\top]$. Integrating
Eq.~\eqref{eq:affine-moment-odes} from $t=0$ to $t=1$ therefore gives the
endpoint distribution directly:
\begin{equation}
    q_c^{\mathcal V}
    =
    \mathcal N(m_c,S_c),
    \qquad
    m_c=m_{c,1},
    \quad
    S_c=S_{c,1}.
\end{equation}

For two conditions $a$ and $b$, let $\delta_{ab}=m_a-m_b$. Their Jeffreys
divergence is then available in closed form:
\begin{equation}
    D_{\mathcal V}(a,b)
    =
    \frac12\left[
        \operatorname{tr}(S_b^{-1}S_a)
        +
        \operatorname{tr}(S_a^{-1}S_b)
        -2d
        +
        \delta_{ab}^{\top}
        (S_a^{-1}+S_b^{-1})
        \delta_{ab}
    \right].
    \label{eq:affine-gaussian-jeffreys-evaluation}
\end{equation}
In practice, we solve Eq.~\eqref{eq:affine-moment-odes} once per condition,
symmetrize the numerical endpoint covariance as
$(S_c+S_c^\top)/2$, and evaluate
Eq.~\eqref{eq:affine-gaussian-jeffreys-evaluation} using Cholesky-based linear
solves rather than explicit matrix inverses. All entries of an RDM can then be
formed from the stored endpoint moments. Unlike
Eq.~\eqref{eq:mc-jeffreys}, this procedure of Jeffreys divergence computation requires no generated samples or
sample-wise log-density evaluations and has no Monte Carlo sampling error.

\subsection{Proofs of velocity-to-distance metric correspondences}
\label{app:velocity-distance-correspondence-proofs}

\begin{proposition}[Translation-constrained flow matching recovers squared Euclidean distance]
\label{prop:app-fm-euclidean}
Let $\rho_0=\mathcal N(0,I)$, let $p_c$ have finite mean $\mu_c$, assume $\pi_{\mathcal C}(c)>0$, and let $\beta_t$ be continuously differentiable with $\beta_0=0$, $\beta_1=1$, and $K_\beta=\int_0^1\dot\beta_t^2\,dt>0$. For the translation-only velocity class
\begin{equation}
    \mathcal V_{\mathrm{Euc}}
    =
    \{v_c(x,t)=\dot\beta_t b_c : b_c \in \mathbb R^d\}.
\end{equation}
Write $\mathbf b=(b_c)_{c\in\mathcal C}$ and let $v_{\mathbf b}$ denote the corresponding condition-dependent velocity field. The global minimizer of the population flow-matching loss in Eq.~\eqref{eq:fm-objective},
\begin{equation}
    \mathbf b^\star
    =
    \arg\min_{\mathbf b}
    \mathcal L_{\mathrm{FM}}(v_{\mathbf b}),
\end{equation}
has condition-specific components $b_c^\star=\mu_c$ and produces $q_c^{\mathcal V_{\mathrm{Euc}}}=\mathcal N(\mu_c,I)$. Consequently,
\begin{equation}
    D_{\mathcal V_{\mathrm{Euc}}}(a,b)
    =
    \|\mu_a-\mu_b\|_2^2.
\end{equation}
\end{proposition}

\begin{proof}
Expanding the objective to the sampled condition $C$ gives
\begin{equation}
    \mathcal L_{\mathrm{FM}}(v_{\mathbf b})
    =
    \mathbb E_{C\sim\pi_{\mathcal C}}
    \left[
      \mathcal L_{\mathrm{FM},C}(b_C)
    \right]
    =
    \sum_{c\in\mathcal C}
    \pi_{\mathcal C}(c)\,
    \mathcal L_{\mathrm{FM},c}(b_c),
    \label{eq:fm-loss-condition-decomposition}
\end{equation}
where $\mathcal L_{\mathrm{FM},c}$ is the flow-matching loss conditional on $C=c$. Each $b_c$ appears in only one summand of Eq.~\eqref{eq:fm-loss-condition-decomposition}. Because $\pi_{\mathcal C}(c)>0$, the global minimization therefore separates as
\begin{equation}
    b_c^\star
    =
    \arg\min_{b_c\in\mathbb R^d}
    \mathcal L_{\mathrm{FM},c}(b_c)
    \qquad\text{for every }c\in\mathcal C.
    \label{eq:fm-conditionwise-minimization}
\end{equation}
We can now solve each condition-specific problem. Substituting $v_c(x,t)=\dot\beta_t b_c$ and $U_t^{(c)}=\dot\alpha_tX_0+\dot\beta_tX_1^{(c)}$ gives
\begin{align}
    \mathcal L_{\mathrm{FM},c}(b_c)
    &=
    \int_0^1
    \mathbb E\!\left[
      \left\|
        \dot\beta_t b_c
        -\dot\alpha_tX_0
        -\dot\beta_tX_1^{(c)}
      \right\|_2^2
      \middle|c
    \right]dt \\
    &=
    \int_0^1\Big[
      \dot\beta_t^2\|b_c\|_2^2
      -2\dot\alpha_t\dot\beta_t b_c^\top\mathbb E[X_0]
      -2\dot\beta_t^2 b_c^\top\mathbb E[X_1^{(c)}\mid c]
      \notag\\
    &\hspace{7em}
      +\mathbb E\!\left[
        \left\|\dot\alpha_tX_0+\dot\beta_tX_1^{(c)}\right\|_2^2
        \middle|c
      \right]
    \Big]dt.
    \label{eq:fm-euclidean-loss-expansion}
\end{align}
The final expectation in Eq.~\eqref{eq:fm-euclidean-loss-expansion} does not depend on $b_c$. Since $X_0\sim\mathcal N(0,I)$ and $X_1^{(c)}\sim p_c$, we have $\mathbb E[X_0]=0$ and $\mathbb E[X_1^{(c)}\mid c]=\mu_c$. Therefore
\begin{align}
    \mathcal L_{\mathrm{FM},c}(b_c)
    &=
    K_\beta \|b_c\|_2^2
    -2K_\beta b_c^\top\mu_c
    +\mathrm{const} \\
    &=
    K_\beta\|b_c-\mu_c\|_2^2
    +\mathrm{const},
    \label{eq:fm-translation-quadratic}
\end{align}
where $K_\beta=\int_0^1\dot\beta_t^2\,dt>0$. Hence $b_c^\star=\mu_c$. The time-one flow is the translation $x_1=x_0+\mu_c$, so Lemma~\ref{lem:affine-flow-gaussian} gives $q_c^{\mathcal V_{\mathrm{Euc}}}=\mathcal N(\mu_c,I)$. Applying this result to conditions $a$ and $b$, their Jeffreys divergence is
\begin{equation}
    D_{\mathcal V_{\mathrm{Euc}}}(a,b)
    =
    \|\mu_a-\mu_b\|_2^2 .
\end{equation}
Note that the response distribution $p_c$ need not be Gaussian: the Gaussian form of $q_c^{\mathcal V_{\mathrm{Euc}}}$ follows from translating the Gaussian source.
\end{proof}

\begin{proposition}[Fixed-norm translations give cosine distance]
\label{prop:app-fm-cosine}
Let $r>0$, suppose $K_\beta=\int_0^1\dot\beta_t^2\,dt>0$, and suppose $\mu_c\neq0$ for every condition. For the fixed-norm translation class
\begin{equation}
    \mathcal V_{\cos}
    =
    \{v_c(x,t)=\dot\beta_tb_c:\ \|b_c\|=r\}.
\end{equation}
Then the population flow-matching optimum is unique within each condition and satisfies $b_c^\star=r\mu_c/\|\mu_c\|$. It produces
\begin{equation}
    q_c^{\mathcal V_{\cos}}
    =
    \mathcal N\!\left(r\frac{\mu_c}{\|\mu_c\|},I\right),
\end{equation}
and the induced distance is
\begin{equation}
    D_{\mathcal V_{\cos}}(a,b)
    =
    2r^2
    \left(
    1-\frac{\mu_a^\top\mu_b}{\|\mu_a\|\|\mu_b\|}
    \right),
\end{equation}
which is $2r^2$ times cosine distance between the condition means.
\end{proposition}

\begin{proof}
As in Proposition~\ref{prop:app-fm-euclidean}, the global objective separates over conditions. Restricting Eq.~\eqref{eq:fm-translation-quadratic} to $\|b_c\|=r$ makes $K_\beta\|b_c\|^2$ constant, so minimizing the condition-specific loss is equivalent to
\begin{equation}
    \max_{\|b_c\|=r} b_c^\top\mu_c.
\end{equation}
By Cauchy--Schwarz, $b_c^\top\mu_c\le r\|\mu_c\|$, with equality only when $b_c$ points in the direction of $\mu_c$. Hence
\begin{equation}
    b_c^\star
    =
    r\frac{\mu_c}{\|\mu_c\|}.
\end{equation}
The endpoint ODE integrates to $x_1=x_0+b_c^\star$. By Lemma~\ref{lem:affine-flow-gaussian},
\begin{equation}
    q_c^{\mathcal V_{\cos}}
    =
    \mathcal N\!\left(r\frac{\mu_c}{\|\mu_c\|},I\right).
\end{equation}
Writing $u_c=\mu_c/\|\mu_c\|$, Jeffreys divergence between the fitted distributions for $a$ and $b$ is
\begin{align}
    D_{\mathcal V_{\cos}}(a,b)
    &=
    \|ru_a-ru_b\|^2
    =
    2r^2(1-u_a^\top u_b) \\
    &=
    2r^2
    \left(
    1-\frac{\mu_a^\top\mu_b}{\|\mu_a\|\|\mu_b\|}
    \right).
\end{align}
Thus the constraint removes mean-magnitude information and retains only direction. If $\mu_c=0$, every point on the radius-$r$ sphere is optimal, so neither the fitted direction nor cosine distance is defined uniquely.
\end{proof}

\begin{proposition}[Centered fixed-norm translations give correlation distance]
\label{prop:app-fm-correlation}
Let $H=I-\mathbf 1\mathbf 1^\top/d$, let $r>0$, suppose $K_\beta>0$, and suppose $H\mu_c\neq0$ for every condition. For the centered fixed-norm class
\begin{equation}
    \mathcal V_{\mathrm{corr}}
    =
    \{v_c(x,t)=\dot\beta_tb_c:\ \mathbf 1^\top b_c=0,\ \|b_c\|=r\}.
    \label{eq:corr-fixed-norm-class}
\end{equation}
Then the population flow-matching optimum is unique within each condition and satisfies $b_c^\star=rH\mu_c/\|H\mu_c\|$. It produces
\begin{equation}
    q_c^{\mathcal V_{\mathrm{corr}}}
    =
    \mathcal N\!\left(r\frac{H\mu_c}{\|H\mu_c\|},I\right),
\end{equation}
and the induced distance is
\begin{equation}
    D_{\mathcal V_{\mathrm{corr}}}(a,b)
    =
    2r^2
    \left(
    1-\frac{(H\mu_a)^\top(H\mu_b)}
    {\|H\mu_a\|\|H\mu_b\|}
    \right),
\end{equation}
which is $2r^2$ times Pearson correlation distance between the condition means.
\end{proposition}

\begin{proof}
The matrix $H$ is the orthogonal projector onto the centered subspace $\{z:\mathbf 1^\top z=0\}$: $H^\top=H$, $H^2=H$, and $H\mathbf 1=0$. Proposition~\ref{prop:app-fm-cosine} shows that, under a fixed-norm translation constraint, minimizing the flow-matching loss is equivalent to maximizing $b_c^\top\mu_c$. The additional constraint in Eq.~\eqref{eq:corr-fixed-norm-class} restricts this maximization to the centered sphere. Every feasible $b_c$ satisfies $b_c=Hb_c$, so
\begin{equation}
    b_c^\top\mu_c=b_c^\top H\mu_c .
\end{equation}
Therefore
\begin{equation}
    b_c^\top\mu_c
    =
    b_c^\top H\mu_c
    \le
    r\|H\mu_c\|,
\end{equation}
with equality only when $b_c$ is aligned with $H\mu_c$. Thus
\begin{equation}
    b_c^\star
    =
    r\frac{H\mu_c}{\|H\mu_c\|}.
\end{equation}
Applying Lemma~\ref{lem:affine-flow-gaussian} gives
\begin{equation}
    q_c^{\mathcal V_{\mathrm{corr}}}
    =
    \mathcal N\!\left(r\frac{H\mu_c}{\|H\mu_c\|},I\right).
\end{equation}
Writing $\widetilde u_c=H\mu_c/\|H\mu_c\|$, the Jeffreys divergence is
\begin{align}
    D_{\mathcal V_{\mathrm{corr}}}(a,b)
    &=
    \|r\widetilde u_a-r\widetilde u_b\|^2 \\
    &=
    2r^2
    \left(
    1-\frac{(H\mu_a)^\top(H\mu_b)}
    {\|H\mu_a\|\|H\mu_b\|}
    \right).
\end{align}
The inner product is Pearson correlation after feature-wise centering and normalization. If $H\mu_c=0$, the constrained optimum is nonunique and correlation distance is undefined for that condition.
\end{proof}

\begin{proposition}[A shared centered affine flow gives Mahalanobis distance]
\label{prop:app-fm-mahalanobis}
Let $\Sigma_c=\operatorname{Cov}_{p_c}(X)$ and $\Sigma_{\mathrm{pool}}=\mathbb E_{C\sim\pi_{\mathcal C}}[\Sigma_C]$. Consider
\begin{equation}
    \mathcal V_{\mathrm{Mah}}
    =
    \{v_c(x,t)=\dot\beta_tb_c+A_t(x-\beta_tb_c)\},
\end{equation}
where $b_c$ is condition specific and $A_t$ is shared across conditions. Define
\begin{equation}
    G(A)
    =
    \int_0^1
    (\dot\beta_tI-\beta_tA_t)^\top
    (\dot\beta_tI-\beta_tA_t)\,dt.
\end{equation}
Assume the pooled path covariance $\bar S_t=\alpha_t^2I+\beta_t^2\Sigma_{\mathrm{pool}}$ is invertible for every $t$, and assume $G(A^\star)$ is positive definite at a population optimum. Then every population optimum has $b_c^\star=\mu_c$ and induces
\begin{equation}
    q_c^{\mathcal V_{\mathrm{Mah}}}
    =
    \mathcal N(\mu_c,\Sigma_{\mathrm{pool}}).
\end{equation}
Consequently,
\begin{equation}
    D_{\mathcal V_{\mathrm{Mah}}}(a,b)
    =
    (\mu_a-\mu_b)^\top
    \Sigma_{\mathrm{pool}}^{-1}
    (\mu_a-\mu_b),
\end{equation}
which is the squared Mahalanobis distance based on pooled within-condition covariance.
\end{proposition}

\begin{proof}
Write
\begin{equation}
    X_1^{(c)}=\mu_c+\varepsilon_c,\qquad
    \mathbb E[\varepsilon_c\mid c]=0,\qquad
    \operatorname{Cov}(\varepsilon_c\mid c)=\Sigma_c .
\end{equation}
where $\varepsilon_c$ does not need to be Gaussian. Define the centered path variables
\begin{equation}
    Y_t^{(c)}=\alpha_tX_0+\beta_t\varepsilon_c,\qquad
    W_t^{(c)}=\dot\alpha_tX_0+\dot\beta_t\varepsilon_c .
\end{equation}
Then $X_t^{(c)}=\beta_t\mu_c+Y_t^{(c)}$ and $U_t^{(c)}=\dot\beta_t\mu_c+W_t^{(c)}$, with both $Y_t^{(c)}$ and $W_t^{(c)}$ centered. Substitution gives
\begin{align}
    v_c(X_t^{(c)},t)-U_t^{(c)}
    =
    A_tY_t^{(c)}-W_t^{(c)}
    +
    (\dot\beta_tI-\beta_tA_t)(b_c-\mu_c).
    \label{eq:mahalanobis-velocity-error-decomposition}
\end{align}
Expanding the squared velocity error using Eq.~\eqref{eq:mahalanobis-velocity-error-decomposition} produces the cross term
\begin{equation}
    2\big(A_tY_t^{(c)}-W_t^{(c)}\big)^\top
    (\dot\beta_tI-\beta_tA_t)(b_c-\mu_c).
\end{equation}
Conditioned on $c$ and $t$, the second factor is deterministic, whereas $\mathbb E[A_tY_t^{(c)}-W_t^{(c)}\mid c,t]=0$ because both $Y_t^{(c)}$ and $W_t^{(c)}$ are centered. The expected cross term therefore vanishes, and the population loss decomposes into
\begin{align}
    \mathcal L_{\mathrm{FM}}(A,\mathbf b)
    &=
    \mathbb E_{\substack{C\sim\pi_{\mathcal C}\\t\sim\operatorname{Unif}[0,1]}}
    \Big[
    \mathbb E\!\left[
      \|A_tY_t^{(C)}-W_t^{(C)}\|^2
      \mid C,t
    \right] \notag\\
    &\qquad+
    \|(\dot\beta_tI-\beta_tA_t)(b_C-\mu_C)\|^2
    \Big].
\end{align}
For fixed $A$, the first term does not depend on $b_c$. The second term is nonnegative; choosing $b_c=\mu_c$ makes it exactly zero. Thus $b_c^\star=\mu_c$ minimizes the loss over $b_c$ for every fixed $A$.

It remains to identify the shared covariance learned by $A_t$. Let
\begin{equation}
    \Sigma_{\mathrm{pool}}=\mathbb E_{C\sim\pi_{\mathcal C}}[\Sigma_C],\qquad
    \bar S_t=\mathbb E_{C\sim\pi_{\mathcal C}}
    [\operatorname{Cov}(Y_t^{(C)}\mid C)]
    =
    \alpha_t^2I+\beta_t^2\Sigma_{\mathrm{pool}},
\end{equation}
and
\begin{equation}
    \bar R_t=\mathbb E_{C\sim\pi_{\mathcal C}}
    [\operatorname{Cov}(W_t^{(C)},Y_t^{(C)}\mid C)]
    =
    \alpha_t\dot\alpha_tI+\beta_t\dot\beta_t\Sigma_{\mathrm{pool}}.
\end{equation}
For each fixed $t$, differentiating the centered least-squares objective with respect to $A_t$ gives the normal equation
\begin{equation}
    A_t^\star\bar S_t=\bar R_t .
\end{equation}
Because $\bar S_t$ is invertible, $A_t^\star=\bar R_t\bar S_t^{-1}$ is unique.

Let $m_{c,t}=\mathbb E[x_t\mid c]$ denote the mean generated by the learned affine flow. Substituting $b_c^\star=\mu_c$ and $A_t^\star$ into Eq.~\eqref{eq:flow-ode} and taking the conditional expectation of both sides gives
\begin{equation}
    \dot m_{c,t}
    =
    \dot\beta_t\mu_c+A_t^\star(m_{c,t}-\beta_t\mu_c),
    \qquad m_{c,0}=0.
\end{equation}
Here we used $\mathbb E[A_t^\star(x_t-\beta_t\mu_c)\mid c]=A_t^\star(m_{c,t}-\beta_t\mu_c)$ because $A_t^\star$ and $\mu_c$ are fixed for a given $t$ and condition $c$.
The function $m_{c,t}=\beta_t\mu_c$ solves this ODE, so uniqueness gives $m_{c,1}=\mu_c$. The generated covariance $Q_t$ satisfies
\begin{equation}
    \frac{dQ_t}{dt}=A_t^\star Q_t+Q_tA_t^{\star\top},\qquad Q_0=I.
\end{equation}
But
\begin{equation}
    \frac{d\bar S_t}{dt}
    =
    2\alpha_t\dot\alpha_tI+2\beta_t\dot\beta_t\Sigma_{\mathrm{pool}}
    =
    \bar R_t+\bar R_t^\top
    =
    A_t^\star\bar S_t+\bar S_tA_t^{\star\top}.
\end{equation}
Thus $\bar S_t$ solves the same covariance ODE with the same initial condition. By the uniqueness theorem for linear ODEs, $Q_t=\bar S_t$ and $Q_1=\Sigma_{\mathrm{pool}}$. Lemma~\ref{lem:affine-flow-gaussian} shows that the endpoint distribution is Gaussian; together with the mean and covariance derived above, this gives
\begin{equation}
    q_c^{\mathcal V_{\mathrm{Mah}}}
    =
    \mathcal N(\mu_c,\Sigma_{\mathrm{pool}}).
\end{equation}
Applying the Gaussian KL formula to conditions $a$ and $b$ gives
\begin{equation}
    D_{\mathcal V_{\mathrm{Mah}}}(a,b)
    =
    (\mu_a-\mu_b)^\top
    \Sigma_{\mathrm{pool}}^{-1}
    (\mu_a-\mu_b).
\end{equation}
This conclusion depends only on the first two moments of each $p_c$; the true response distributions need not be Gaussian.
\end{proof}

\begin{proposition}[Condition-specific affine flows give Gaussian Jeffreys divergence]
\label{prop:app-fm-gaussian-jeffreys}
Let $\Sigma_c=\operatorname{Cov}_{p_c}(X)$ be positive definite. For the condition-specific centered affine class
\begin{equation}
    \mathcal V_{\mathrm{Aff}}
    =
    \{v_c(x,t)=\dot\beta_tb_c+A_{c,t}(x-\beta_tb_c)\},
\end{equation}
define
\begin{equation}
    S_{c,t}=\alpha_t^2I+\beta_t^2\Sigma_c,
    \qquad
    G_c(A_c)=\int_0^1
    (\dot\beta_tI-\beta_tA_{c,t})^\top
    (\dot\beta_tI-\beta_tA_{c,t})\,dt.
\end{equation}
Assume $S_{c,t}$ is invertible for every $(c,t)$ and $G_c(A_c^\star)$ is positive definite at a population optimum. Then $b_c^\star=\mu_c$ and
\begin{equation}
    q_c^{\mathcal V_{\mathrm{Aff}}}
    =
    \mathcal N(\mu_c,\Sigma_c).
\end{equation}
For $\delta_{ab}=\mu_a-\mu_b$, the induced distance is
\begin{equation}
    D_{\mathcal V_{\mathrm{Aff}}}(a,b)
    =
    \frac12\Big[
    \operatorname{tr}(\Sigma_b^{-1}\Sigma_a)
    +
    \operatorname{tr}(\Sigma_a^{-1}\Sigma_b)
    -2d
    +\delta_{ab}^\top(\Sigma_a^{-1}+\Sigma_b^{-1})\delta_{ab}
    \Big].
\end{equation}
\end{proposition}

\begin{proof}
Because both $b_c$ and $A_{c,t}$ are condition specific, Eq.~\eqref{eq:fm-loss-condition-decomposition} separates the global optimization into independent condition-specific problems. Write $X_1^{(c)}=\mu_c+\varepsilon_c$ and define $Y_t^{(c)}$ and $W_t^{(c)}$ as in the Mahalanobis proof. Then
\begin{equation}
    v_c(X_t^{(c)},t)-U_t^{(c)}
    =
    A_{c,t}Y_t^{(c)}-W_t^{(c)}
    +(\dot\beta_tI-\beta_tA_{c,t})(b_c-\mu_c).
\end{equation}
The condition-specific loss decomposes into
\begin{equation}
    \mathcal L_{\mathrm{FM},c}(A_c,b_c)
    =
    \int_0^1
    \mathbb E\!\left[
      \|A_{c,t}Y_t^{(c)}-W_t^{(c)}\|^2
      \mid c
    \right]dt
    +
    (b_c-\mu_c)^\top G_c(A_c)(b_c-\mu_c).
\end{equation}
For fixed $A_c$, the first term is independent of $b_c$, whereas the second term is nonnegative and vanishes at $b_c=\mu_c$. Hence $b_c=\mu_c$ minimizes the condition-specific loss for every fixed $A_c$. Positive definiteness of $G_c(A_c^\star)$ makes this minimizer unique at the population optimum, so $b_c^\star=\mu_c$.

For the centered problem, let
\begin{equation}
    R_{c,t}
    =
    \operatorname{Cov}(W_t^{(c)},Y_t^{(c)}\mid c)
    =
    \alpha_t\dot\alpha_tI
    +\beta_t\dot\beta_t\Sigma_c.
\end{equation}
Differentiating the least-squares loss with respect to $A_{c,t}$ gives
\begin{equation}
    A_{c,t}^\star S_{c,t}=R_{c,t},
    \qquad
    A_{c,t}^\star=R_{c,t}S_{c,t}^{-1}.
\end{equation}
As above, the generated mean is $m_{c,t}=\beta_t\mu_c$. Its covariance $Q_{c,t}$ solves
\begin{equation}
    \dot Q_{c,t}
    =
    A_{c,t}^\star Q_{c,t}+Q_{c,t}A_{c,t}^{\star\top},
    \qquad Q_{c,0}=I.
    \label{eq:affine-covariance-ode}
\end{equation}
Meanwhile,
\begin{equation}
    \dot S_{c,t}
    =
    R_{c,t}+R_{c,t}^\top
    =
    A_{c,t}^\star S_{c,t}+S_{c,t}A_{c,t}^{\star\top},
    \qquad S_{c,0}=I.
\end{equation}
Uniqueness of the covariance ODE gives $Q_{c,t}=S_{c,t}$, and hence $Q_{c,1}=\Sigma_c$. By Lemma~\ref{lem:affine-flow-gaussian}, the endpoint distribution is Gaussian, so the affine flow maps $\rho_0$ to
\begin{equation}
    q_c^{\mathcal V_{\mathrm{Aff}}}
    =
    \mathcal N(\mu_c,\Sigma_c).
\end{equation}
Therefore,
\begin{equation}
    D_{\mathcal V_{\mathrm{Aff}}}(a,b)
    =
    \frac12
    \big[
    \operatorname{tr}(\Sigma_b^{-1}\Sigma_a)
    +
    \operatorname{tr}(\Sigma_a^{-1}\Sigma_b)
    -2d
    +
    \delta_{ab}^\top(\Sigma_a^{-1}+\Sigma_b^{-1})\delta_{ab}
    \big].
\end{equation}
When $\Sigma_a=\Sigma_b=\Sigma$, the trace terms cancel and the result reduces to $(\mu_a-\mu_b)^\top\Sigma^{-1}(\mu_a-\mu_b)$.
\end{proof}

\begin{proposition}[Stimulus-dependent affine flows give Gaussian Fisher information]
\label{prop:app-fm-gaussian-fisher}
Let $\theta$ be a scalar condition and suppose
\begin{equation}
    q_\theta^{\mathcal V_{\mathrm{Aff}}}=\mathcal N(\mu_\theta,\Sigma_\theta).
\end{equation}
Assume $\mu_\theta$ and $\Sigma_\theta$ are twice continuously differentiable in a neighborhood of $\theta$ and $\Sigma_\theta$ is positive definite there. Then, as $h\to0$,
\begin{equation}
    D_{\mathrm J}\!\left(q_\theta^{\mathcal V_{\mathrm{Aff}}},q_{\theta+h}^{\mathcal V_{\mathrm{Aff}}}\right)
    =
    J_{\mathrm{Gauss}}(\theta)h^2+O(h^3),
\end{equation}
where
\begin{equation}
    J_{\mathrm{Gauss}}(\theta)
    =
    \mu_\theta'{}^\top\Sigma_\theta^{-1}\mu_\theta'
    +\frac12\operatorname{tr}\!\left(
      \Sigma_\theta^{-1}\Sigma_\theta'
      \Sigma_\theta^{-1}\Sigma_\theta'
    \right).
\end{equation}
\end{proposition}

\begin{proof}
Applying Proposition~\ref{prop:app-fm-gaussian-jeffreys} to $\theta$ and
$\theta+h$, write
$D_h=D_{\mathrm J}(q_\theta^{\mathcal V_{\mathrm{Aff}}},
q_{\theta+h}^{\mathcal V_{\mathrm{Aff}}})$ and
$\Delta\mu_h=\mu_\theta-\mu_{\theta+h}$. Then
\begin{equation}
    D_h=
    \underbrace{
      \frac12\left[
        \operatorname{tr}(\Sigma_{\theta+h}^{-1}\Sigma_\theta)
        +\operatorname{tr}(\Sigma_\theta^{-1}\Sigma_{\theta+h})
        -2d
      \right]
    }_{\text{covariance contribution}}
    +
    \underbrace{
      \frac12
      \Delta\mu_h^\top
      (\Sigma_\theta^{-1}+\Sigma_{\theta+h}^{-1})
      \Delta\mu_h
    }_{\text{mean-difference contribution}}.
\end{equation}
Taylor expansion gives
\begin{equation}
    \mu_{\theta+h}
    =
    \mu_\theta+\mu_\theta'h+O(h^2),\qquad
    \Sigma_{\theta+h}
    =
    \Sigma_\theta+\Sigma_\theta'h+O(h^2).
\end{equation}
Because $\Sigma_{\theta+h}^{-1}=\Sigma_\theta^{-1}+O(h)$, the
mean-difference contribution is
\begin{equation}
    \mu_\theta'{}^\top\Sigma_\theta^{-1}\mu_\theta'\,h^2
    +
    O(h^3).
\end{equation}
For the covariance contribution, define
\begin{equation}
    B=
    \Sigma_\theta^{-1/2}
    (\Sigma_{\theta+h}-\Sigma_\theta)
    \Sigma_\theta^{-1/2}
    =
    \Sigma_\theta^{-1/2}\Sigma_\theta'\Sigma_\theta^{-1/2}h
    +O(h^2).
\end{equation}
From the definition of $B$, we have $\Sigma_{\theta+h}=\Sigma_\theta^{1/2}(I+B)\Sigma_\theta^{1/2}$. Hence $\Sigma_\theta^{-1}\Sigma_{\theta+h}=\Sigma_\theta^{-1/2}(I+B)\Sigma_\theta^{1/2}$ is similar to $I+B$, while $\Sigma_{\theta+h}^{-1}\Sigma_\theta=\Sigma_\theta^{-1/2}(I+B)^{-1}\Sigma_\theta^{1/2}$ is similar to $(I+B)^{-1}$. Similarity preserves trace, so the covariance trace can be evaluated using $I+B$ and $(I+B)^{-1}$. Since $(I+B)^{-1}=I-B+B^2+O(\|B\|^3)$, the first-order terms cancel and
\begin{align}
    \frac12
    \operatorname{tr}\left(
    \Sigma_{\theta+h}^{-1}\Sigma_\theta
    +
    \Sigma_\theta^{-1}\Sigma_{\theta+h}
    -2I
    \right)
    =
    \frac12
    \operatorname{tr}\left(
    \Sigma_\theta^{-1}\Sigma_\theta'
    \Sigma_\theta^{-1}\Sigma_\theta'
    \right)h^2
    +
    O(h^3).
\end{align}
Adding the mean-difference and covariance contributions gives the stated $J_{\mathrm{Gauss}}(\theta)$. For a vector-valued condition, replacing $h$ by a vector increment gives the Fisher information matrix
\begin{equation}
    [J_{\mathrm{Gauss}}(\theta)]_{ij}
    =
    (\partial_i\mu_\theta)^\top\Sigma_\theta^{-1}(\partial_j\mu_\theta)
    +\frac12\operatorname{tr}\!\left(
      \Sigma_\theta^{-1}(\partial_i\Sigma_\theta)
      \Sigma_\theta^{-1}(\partial_j\Sigma_\theta)
    \right).
\end{equation}
\end{proof}

\begin{proposition}[Averaging local affine coefficients gives linear Fisher information]
\label{prop:app-fm-linear-fisher}
Under the setup of Proposition~\ref{prop:app-fm-gaussian-fisher}, suppose the fitted conditional affine velocity is
\begin{equation}
    v_\xi(x,t)
    =
    \dot\beta_t\mu_\xi
    +A_{\xi,t}(x-\beta_t\mu_\xi),
    \qquad
    \xi\in\{\theta,\theta+h\},
\end{equation}
and assume $A_{\xi,t}$ is jointly continuous in $(\xi,t)$ and continuously differentiable in $\xi$ on a neighborhood of $\theta$ times $[0,1]$. Average the two fitted linear coefficients,
\begin{equation}
    \bar A_{\theta,h,t}
    =
    \frac12\left(A_{\theta,t}+A_{\theta+h,t}\right),
\end{equation}
and construct two new velocities
\begin{align}
    \widetilde v_\theta^{(h)}(x,t)
    &=
    \dot\beta_t\mu_\theta
    +\bar A_{\theta,h,t}(x-\beta_t\mu_\theta),\\
    \widetilde v_{\theta+h}^{(h)}(x,t)
    &=
    \dot\beta_t\mu_{\theta+h}
    +\bar A_{\theta,h,t}(x-\beta_t\mu_{\theta+h}).
\end{align}
Let $\widetilde q_\theta^{(h)}$ and $\widetilde q_{\theta+h}^{(h)}$ denote their endpoint distributions from the common source $\rho_0=\mathcal N(0,I)$. Then
\begin{equation}
    D_{\mathrm J}\!\left(
        \widetilde q_\theta^{(h)},
        \widetilde q_{\theta+h}^{(h)}
    \right)
    =
    \mu_\theta'{}^\top\Sigma_\theta^{-1}\mu_\theta'\,h^2
    +O(h^3),
\end{equation}
and therefore
\begin{equation}
    \lim_{h\to0}
    \frac{D_{\mathrm J}\!\left(
        \widetilde q_\theta^{(h)},
        \widetilde q_{\theta+h}^{(h)}
    \right)}{h^2}
    =
    \mu_\theta'{}^\top\Sigma_\theta^{-1}\mu_\theta',
\end{equation}
which is linear Fisher information.
\end{proposition}

\begin{proof}
Applying Lemma~\ref{lem:affine-flow-gaussian} to the two constructed velocities shows that their endpoint distributions are Gaussian. Because the velocities have the same linear coefficient $\bar A_{\theta,h,t}$ and start from the same Gaussian source, their endpoint covariance is identical; denote it by $\bar\Sigma_{\theta,h}$. Their respective endpoint means are $\mu_\theta$ and $\mu_{\theta+h}$, so
\begin{equation}
    \widetilde q_\theta^{(h)}
    =
    \mathcal N(\mu_\theta,\bar\Sigma_{\theta,h}),
    \qquad
    \widetilde q_{\theta+h}^{(h)}
    =
    \mathcal N(\mu_{\theta+h},\bar\Sigma_{\theta,h}).
\end{equation}
For Gaussians with a common covariance, Jeffreys divergence is exactly
\begin{equation}
    D_{\mathrm J}\!\left(
        \widetilde q_\theta^{(h)},
        \widetilde q_{\theta+h}^{(h)}
    \right)
    =
    (\mu_{\theta+h}-\mu_\theta)^\top
    \bar\Sigma_{\theta,h}^{-1}
    (\mu_{\theta+h}-\mu_\theta).
    \label{eq:linear-fisher-common-covariance-jeffreys}
\end{equation}
Since $\mu_{\theta+h}-\mu_\theta=\mu_\theta'h+O(h^2)$, Eq.~\eqref{eq:linear-fisher-common-covariance-jeffreys} has the claimed expansion provided that
\begin{equation}
    \bar\Sigma_{\theta,h}
    =
    \Sigma_\theta+O(h),
    \label{eq:linear-fisher-covariance-stability}
\end{equation}
because matrix inversion is smooth in a neighborhood of the positive-definite matrix $\Sigma_\theta$. Thus, it remains only to prove Eq.~\eqref{eq:linear-fisher-covariance-stability}.

Because $A_{\theta+h,t}=A_{\theta,t}+O(h)$ uniformly on $t\in[0,1]$, we have $\bar A_{\theta,h,t}=A_{\theta,t}+O(h)$. Let $\bar Q_t$ and $Q_{\theta,t}$ be the covariance trajectories generated by $\bar A_{\theta,h,t}$ and $A_{\theta,t}$, respectively. Both satisfy the affine covariance ODE in Eq.~\eqref{eq:affine-covariance-ode}, with the corresponding linear coefficient. Define $E_t=\bar Q_t-Q_{\theta,t}$ and $\Delta A_t=\bar A_{\theta,h,t}-A_{\theta,t}$. Subtracting these two instances of Eq.~\eqref{eq:affine-covariance-ode} gives
\begin{equation}
    \dot E_t
    =
    \bar A_{\theta,h,t}E_t
    +E_t\bar A_{\theta,h,t}^\top
    +\Delta A_tQ_{\theta,t}
    +Q_{\theta,t}\Delta A_t^\top,
    \qquad E_0=0.
    \label{eq:linear-fisher-covariance-error}
\end{equation}
The coefficients and $Q_{\theta,t}$ are bounded on $t \in [0,1]$, while $\sup_t\|\Delta A_t\|=O(h)$. Therefore Eq.~\eqref{eq:linear-fisher-covariance-error} implies $\|E_t\|\le C_1\int_0^t\|E_s\|\,ds+C_2h$, and Gronwall's inequality gives $\sup_{t\in[0,1]}\|E_t\|=O(h)$. At $t=1$, $\bar Q_1=\bar\Sigma_{\theta,h}$ and $Q_{\theta,1}=\Sigma_\theta$, so
\begin{equation}
    \bar\Sigma_{\theta,h}-\Sigma_\theta
    =
    E_1
    =
    O(h),
\end{equation}
which proves Eq.~\eqref{eq:linear-fisher-covariance-stability}. Positive definiteness of $\Sigma_\theta$ then gives $\bar\Sigma_{\theta,h}^{-1}=\Sigma_\theta^{-1}+O(h)$. Substituting this relation and $\mu_{\theta+h}-\mu_\theta=\mu_\theta'h+O(h^2)$ into Eq.~\eqref{eq:linear-fisher-common-covariance-jeffreys} yields
\begin{equation}
    D_{\mathrm J}\!\left(
        \widetilde q_\theta^{(h)},
        \widetilde q_{\theta+h}^{(h)}
    \right)
    =
    \mu_\theta'{}^\top\Sigma_\theta^{-1}\mu_\theta'\,h^2
    +O(h^3),
\end{equation}
\end{proof}

\begin{proposition}[Unconstrained flows give Jeffreys divergence and full Fisher information]
\label{prop:app-fm-unconstrained}
Consider an unconstrained flow trained by flow matching. At the population optimum, the standard flow-matching result \citep{Lipman2023FlowMatching} gives
\begin{equation}
    q_c^{\mathrm{uncon}}=p_c.
\end{equation}
Hence the induced distance is the true Jeffreys divergence:
\begin{equation}
    D_{\mathrm{uncon}}(a,b)
    =
    D_{\mathrm J}(p_a,p_b).
\end{equation}
For a continuous parameter $\theta\in\mathbb R^k$, further assume $p_\theta$ is positive on a common support, $\log p_\theta(x)$ is three times continuously differentiable in $\theta$, and differentiation can be exchanged with integration. Then, for $h\in\mathbb R^k$,
\begin{equation}
    D_{\mathrm J}(p_\theta,p_{\theta+h})
    =
    h^\top J(\theta)h
    +o(\|h\|^2),
\end{equation}
where $J(\theta)$ is the full Fisher information matrix.
\end{proposition}

\begin{proof}
The population endpoint identity $q_c^{\mathrm{uncon}}=p_c$ is the standard flow-matching result stated above and is not reproved here. Substitution into Eq.~\eqref{eq:flow-jeffreys} gives $D_{\mathrm{uncon}}(a,b)=D_{\mathrm J}(p_a,p_b)$.

For the continuous family, define the score $s_i(x;\theta)=\partial_{\theta_i}\log p_\theta(x)$ and the Fisher information matrix
\begin{equation}
    J_{ij}(\theta)=
    \mathbb E_{p_\theta}[s_i(X;\theta)s_j(X;\theta)].
\end{equation}
The stated regularity conditions imply $\mathbb E_{p_\theta}[s_i(X;\theta)]=0$ and the information identity
\begin{equation}
    J_{ij}(\theta)
    =
    -\mathbb E_{p_\theta}
    [\partial_{\theta_i}\partial_{\theta_j}\log p_\theta(X)].
\end{equation}
Expand $\ell_{\theta+h}(x)=\log p_{\theta+h}(x)$ around $\theta$:
\begin{equation}
    \ell_{\theta+h}(x)
    =
    \ell_\theta(x)
    +
    h^\top\nabla_\theta\ell_\theta(x)
    +
    \frac12h^\top\nabla_\theta^2\ell_\theta(x)h
    +
    o(\|h\|^2).
\end{equation}
Taking expectation under $p_\theta$ yields
\begin{equation}
    D_{\mathrm{KL}}(p_\theta\|p_{\theta+h})
    =
    \frac12h^\top J(\theta)h+o(\|h\|^2).
\end{equation}
Applying the same expansion at $\theta+h$ with increment $-h$ gives
\begin{equation}
    D_{\mathrm{KL}}(p_{\theta+h}\|p_\theta)
    =
    \frac12h^\top J(\theta+h)h+o(\|h\|^2)
    =
    \frac12h^\top J(\theta)h+o(\|h\|^2),
\end{equation}
where the last equality uses continuity of $J$. Adding both directions yields
\begin{equation}
    D_{\mathrm J}(p_\theta,p_{\theta+h})
    =
    h^\top J(\theta)h
    +
    o(\|h\|^2).
\end{equation}
\end{proof}

Proposition~\ref{prop:app-fm-gaussian-fisher} should not be interpreted as a special case of Proposition~\ref{prop:app-fm-unconstrained}. The former constrains the velocity to be affine and can be applied even when the target distributions are non-Gaussian. The latter instead assumes an unconstrained flow that recovers the target distributions themselves.

\begin{proposition}[Similarity-constrained velocities generate similarity transformations]
\label{prop:app-fm-geometric}
Let $T$ be the time-one map of
\begin{equation}
    \frac{dx}{dt}
    =
    a_t+\Omega_t(x-o_t)+\lambda_t(x-o_t),
    \qquad
    \Omega_t^\top=-\Omega_t.
\end{equation}
If the coefficients are continuous on $[0,1]$, then
\begin{equation}
    T(z)=sRz+u,
    \qquad
    s>0,\quad R\in\mathrm{SO}(d),\quad u\in\mathbb R^d.
\end{equation}
Thus the flow transforms every data point by an isotropic scaling, a rotation, and a translation, without shear or nonlinear deformation.
\end{proposition}

\begin{proof}
Rewrite the flow as $\dot x=A_tx+r_t$, where
\begin{equation}
    A_t=\Omega_t+\lambda_tI,
    \qquad
    r_t=a_t-(\Omega_t+\lambda_tI)o_t.
\end{equation}
Let $\Phi_t$ solve $\dot\Phi_t=A_t\Phi_t$ with $\Phi_0=I$. Since $\Omega_t$ is skew-symmetric,
\begin{align}
    \frac{d}{dt}(\Phi_t^\top\Phi_t)
    &=
    \Phi_t^\top(A_t^\top+A_t)\Phi_t \\
    &=
    2\lambda_t\Phi_t^\top\Phi_t.
\end{align}
Let $M_t=\Phi_t^\top\Phi_t$. Because $M_0=I$, the matrix ODE $\dot M_t=2\lambda_tM_t$ has the solution
\begin{equation}
    M_t
    =
    \exp\!\left(2\int_0^t\lambda_\tau\,d\tau\right)I
    =
    s_t^2I,
    \qquad
    s_t=\exp\!\left(\int_0^t\lambda_\tau\,d\tau\right)>0.
\end{equation}
Define $R_t=s_t^{-1}\Phi_t$. Then
\begin{equation}
    R_t^\top R_t
    =
    s_t^{-2}\Phi_t^\top\Phi_t
    =I,
\end{equation}
so $R_t$ is orthogonal. Moreover, $R_0=I$ and $R_t$ varies continuously with $t$. Since an orthogonal matrix has determinant $+1$ or $-1$, continuity prevents the determinant from changing sign; hence $\det R_t=1$ and $R_t\in\mathrm{SO}(d)$.

Finally, variation of constants gives the trajectory starting from $z$ as
\begin{equation}
    T_t(z)
    =
    \Phi_tz
    +
    \Phi_t\int_0^t\Phi_\tau^{-1}r_\tau\,d\tau
    =
    s_tR_tz+u_t,
\end{equation}
where $u_t=\Phi_t\int_0^t\Phi_\tau^{-1}r_\tau\,d\tau$ is independent of $z$. Evaluating at $t=1$ and writing $s=s_1$, $R=R_1$, and $u=u_1$ yields $T(z)=sRz+u$, as claimed.
\end{proof}

\section{Likelihood Fine-Tuning and Metric Correspondences}
\label{app:nll-derivations}

\subsection{Optional negative-log-likelihood fine-tuning}
\label{sec:nll-fine-tuning}

Evaluating the Jeffreys divergence requires log densities of
$q_c^{\mathcal V}$. Flow matching optimizes a regression objective that is
fast to train but does not directly maximize likelihood. Many previous works emploied negative log likelihood (NLL) to improve fitting. In this manuscript, we applied fine-tuning for our defined geometric-preserved distance (Supplementary Figure~\ref{fig:geom}).

Below we describe the fine-tuning procedure, and later prove that optimizing NLL also recovers the distance correspondences in Table~\ref{tab:constraints} at their population optima.

The log density is evaluated using the continuous change-of-variables identity
\begin{equation}
    \log q_c^{\mathcal V}(x_1)
    =
    \log \rho_0(x_0)
    -
    \int_0^1 \nabla_x\cdot v_c^{\mathcal V}(x_t,t)\,dt,
    \qquad
    x_0=(T_{c,1}^{\mathcal V})^{-1}(x_1).
    \label{eq:cnf-log-density}
\end{equation}
Here, $\nabla_x\cdot v_c^{\mathcal V}$ measures how the flow locally expands
or contracts volume. For an observed response $x_1$, we integrate backward to
recover $x_0$ and evaluate $\log q_c^{\mathcal V}(x_1)$. We fine-tune the constrained velocity field by minimizing the NLL
\begin{equation}
    \mathcal L_{\mathrm{NLL}}
    =
    -\frac{1}{N}\sum_c\sum_{i=1}^{n_c}
    \log q_c^{\mathcal V}\!\left(x_i^{(c)}\right),
    \qquad N=\sum_c n_c.
    \label{eq:nll}
\end{equation}
This stage encourages $q_c^{\mathcal V}$ to assign higher density to the
recorded responses, improving its finite-sample approximation to $p_c$ while
keeping the velocity constrained to $\mathcal V$.

\subsection{Likelihood-based proofs}

This section proves the counterparts of Table~\ref{tab:constraints} when $q_c^{\mathcal V}$ is fitted by population negative log likelihood. These statements describe the global optimum of Eq.~\eqref{eq:nll}. By Lemma~\ref{lem:affine-flow-gaussian}, affine velocity classes transport the Gaussian source within a Gaussian family. We assume finite second moments and positive-definite covariances wherever an affine Gaussian density is used.

For the translation families below, let
$\mathbf b=(b_c)_{c\in\mathcal C}$ and define the population objective
\begin{align}
    \mathcal L_{\mathrm{NLL}}(\mathbf b)
    &=
    \mathbb E_{C\sim\pi_{\mathcal C}}
    \mathbb E_{X\sim p_C}
    [-\log q_C^{\mathcal V}(X)] \notag\\
    &=
    \sum_{c\in\mathcal C}
    \pi_{\mathcal C}(c)\,
    \mathcal L_{\mathrm{NLL},c}(b_c),
    \label{eq:nll-condition-decomposition}\\
    \mathcal L_{\mathrm{NLL},c}(b_c)
    &=
    -\mathbb E_{X\sim p_c}
    [\log q_c^{\mathcal V}(X)].
    \notag
\end{align}

\begin{proposition}[Likelihood-trained translations give Euclidean, cosine, and correlation distances]
\label{prop:app-nll-translations}
\label{prop:nll-euclidean}
Let $\rho_0=\mathcal N(0,I)$ and
$v_c(x,t)=\dot\beta_t b_c$, with $\beta_0=0$ and $\beta_1=1$. Then the
time-one map is $T_{c,1}(x)=x+b_c$, and
Lemma~\ref{lem:affine-flow-gaussian} gives
$q_c^{\mathcal V}=\mathcal N(b_c,I)$. Assume
$\pi_{\mathcal C}(c)>0$ and that each $p_c$ has finite second moments. Let
$r>0$ and $H=I-\mathbf 1\mathbf 1^\top/d$. Population likelihood training
yields:
\begin{align}
    b_c^\star&=\mu_c
    &&\text{without an additional constraint},\label{eq:nll-trans-euc-opt}\\
    b_c^\star&=r\frac{\mu_c}{\|\mu_c\|}
    &&\text{under }\|b_c\|=r,\quad \mu_c\ne0,\label{eq:nll-trans-cos-opt}\\
    b_c^\star&=r\frac{H\mu_c}{\|H\mu_c\|}
    &&\text{under }\|b_c\|=r,\ \mathbf1^\top b_c=0,\quad H\mu_c\ne0.
    \label{eq:nll-trans-corr-opt}
\end{align}
The corresponding Jeffreys divergences are squared Euclidean distance, $2r^2$ times cosine distance, and $2r^2$ times Pearson correlation distance, respectively.
\end{proposition}

\begin{proof}
Equation~\eqref{eq:nll-condition-decomposition} shows that the global likelihood objective is a positive weighted sum of condition-specific losses. Because $b_c$ is condition specific, its minimization separates over $c$. Direct evaluation of the Gaussian negative log likelihood gives
\begin{align}
    \mathcal L_{\mathrm{NLL},c}(b_c)
    &=
    \frac12\mathbb E_{X\sim p_c}\!\left[\|X-b_c\|_2^2\right]
    +\mathrm{const} \\
    &=
    \frac12\|b_c-\mu_c\|_2^2
    +\frac12\operatorname{tr}(\Sigma_c)
    +\mathrm{const} \\
    &=
    \frac12\|b_c\|^2-b_c^\top\mu_c+\mathrm{const},
    \label{eq:nll-translation-projection}
\end{align}
where terms independent of $b_c$ are absorbed into the constant. The unconstrained quadratic $\frac12\|b_c-\mu_c\|_2^2$ is uniquely minimized at $b_c=\mu_c$, proving Eq.~\eqref{eq:nll-trans-euc-opt}.

On the sphere $\|b_c\|=r$, the norm term is constant, so Cauchy--Schwarz reduces the problem to maximizing $b_c^\top\mu_c$ and gives Eq.~\eqref{eq:nll-trans-cos-opt}. On the centered sphere, $b_c=Hb_c$, so $b_c^\top\mu_c=b_c^\top H\mu_c$; applying Cauchy--Schwarz within this subspace gives Eq.~\eqref{eq:nll-trans-corr-opt}.

All three fitted distributions have identity covariance, so Jeffreys divergence equals the squared distance between their means. Substitution gives
\begin{align}
    D_{\mathcal V_{\mathrm{Euc}}}(a,b)
    &=\|\mu_a-\mu_b\|^2,\\
    D_{\mathcal V_{\cos}}(a,b)
    &=2r^2\left(1-\frac{\mu_a^\top\mu_b}{\|\mu_a\|\|\mu_b\|}\right),\\
    D_{\mathcal V_{\mathrm{corr}}}(a,b)
    &=2r^2\left(1-\frac{(H\mu_a)^\top(H\mu_b)}{\|H\mu_a\|\|H\mu_b\|}\right).
\end{align}
If a required mean vector is zero, the corresponding constrained optimum is nonunique, exactly matching the usual domain restriction of cosine or correlation distance.
The response distributions $p_c$ need not be Gaussian.
\end{proof}

\begin{proposition}[Likelihood-trained affine flows give Mahalanobis and Gaussian Jeffreys divergence]
\label{prop:app-nll-affine}
Assume $\pi_{\mathcal C}(c)>0$ and each $p_c$ has mean $\mu_c$ and positive-definite covariance $\Sigma_c$. Let $\Sigma_{\mathrm{pool}}=\mathbb E_{C\sim\pi_{\mathcal C}}[\Sigma_C]$.

If the reachable Gaussian family has condition-specific means and one shared covariance,
\begin{equation}
    q_c^{\mathcal V}=\mathcal N(m_c,S),
\end{equation}
then the population likelihood optimum is $m_c^\star=\mu_c$, $S^\star=\Sigma_{\mathrm{pool}}$, and
\begin{equation}
    D_{\mathcal V}(a,b)
    =
    (\mu_a-\mu_b)^\top\Sigma_{\mathrm{pool}}^{-1}(\mu_a-\mu_b).
\end{equation}
If both means and covariances are condition specific, the optimum is $q_c^{\mathcal V}=\mathcal N(\mu_c,\Sigma_c)$ and the induced distance is the Gaussian Jeffreys divergence in Proposition~\ref{prop:app-fm-gaussian-jeffreys}.
\end{proposition}

\begin{proof}
For the shared-covariance family, direct evaluation of the Gaussian negative log likelihood and averaging over $C\sim\pi_{\mathcal C}$ gives
\begin{align}
    \mathcal L_{\mathrm{NLL}}(\mathbf m,S)
    =
    \frac12\left[
      \log\det S
      +\mathbb E_C\operatorname{tr}(S^{-1}\Sigma_C)
      +\mathbb E_C(\mu_C-m_C)^\top S^{-1}(\mu_C-m_C)
    \right]
    +\mathrm{const}.
\end{align}
For every fixed $S\succ0$, positivity of $\pi_{\mathcal C}(c)$ and positive definiteness of $S^{-1}$ give the unique mean optimum $m_c^\star=\mu_c$. The remaining covariance objective is
\begin{equation}
    \frac12\left[
      \log\det S
      +\operatorname{tr}(S^{-1}\Sigma_{\mathrm{pool}})
    \right]
    +\mathrm{const}.
\end{equation}
The Gaussian cross-entropy is uniquely minimized by matching the pooled covariance, so $S^\star=\Sigma_{\mathrm{pool}}$. Hence
\begin{equation}
    q_c^{\mathcal V}=\mathcal N(\mu_c,\Sigma_{\mathrm{pool}}),
    \qquad
    D_{\mathcal V}(a,b)
    =
    (\mu_a-\mu_b)^\top\Sigma_{\mathrm{pool}}^{-1}(\mu_a-\mu_b),
\end{equation}
where the distance follows from the common-covariance Gaussian KL formula.

If the covariance is condition specific, the global loss separates over $c$. Minimizing the Gaussian cross-entropy conditionwise gives $(m_c^\star,S_c^\star)=(\mu_c,\Sigma_c)$ for every condition. Adding the two directed Gaussian KL divergences cancels the log determinants and yields the formula in Proposition~\ref{prop:app-fm-gaussian-jeffreys}.
\end{proof}

\begin{proposition}[Likelihood-trained affine flows give Gaussian and linear Fisher information]
\label{prop:app-nll-fisher}
Let $\theta$ be scalar, and suppose likelihood optimization is pointwise-separable in $\theta$ and can realize $q_\theta^{\mathcal V}=\mathcal N(m_\theta,S_\theta)$ at every condition. Assume $\mu_\theta$ and $\Sigma_\theta$ are twice continuously differentiable and $\Sigma_\theta$ is positive definite. At the population likelihood optimum, $m_\theta^\star=\mu_\theta$ and $S_\theta^\star=\Sigma_\theta$, and
\begin{equation}
    \lim_{h\to0}
    \frac{D_{\mathcal V}(\theta,\theta+h)}{h^2}
    =
    \mu_\theta'{}^\top\Sigma_\theta^{-1}\mu_\theta'
    +\frac12\operatorname{tr}\!\left(
      \Sigma_\theta^{-1}\Sigma_\theta'
      \Sigma_\theta^{-1}\Sigma_\theta'
    \right).
\end{equation}
If the two nearby fitted distributions are instead constrained to share a covariance $\bar\Sigma_{\theta,h}=\Sigma_\theta+O(h)$ while retaining their separate means, the limit is $\mu_\theta'{}^\top\Sigma_\theta^{-1}\mu_\theta'$, the linear Fisher information.
\end{proposition}

\begin{proof}
At each fixed $\theta$, minimizing the Gaussian cross-entropy gives the unique pointwise optimum
\begin{equation}
    q_\theta^{\mathcal V}=\mathcal N(\mu_\theta,\Sigma_\theta).
\end{equation}
Proposition~\ref{prop:app-fm-gaussian-fisher} then applies directly and gives
\begin{equation}
    D_{\mathcal V}(\theta,\theta+h)
    =
    \left[
      \mu_\theta'{}^\top\Sigma_\theta^{-1}\mu_\theta'
      +\frac12\operatorname{tr}\!\left(
        \Sigma_\theta^{-1}\Sigma_\theta'
        \Sigma_\theta^{-1}\Sigma_\theta'
      \right)
    \right]h^2
    +O(h^3).
\end{equation}
Under the shared-local-covariance constraint, Proposition~\ref{prop:app-fm-linear-fisher} gives
\begin{equation}
    D_{\mathcal V}(\theta,\theta+h)
    =
    \mu_\theta'{}^\top\Sigma_\theta^{-1}\mu_\theta'\,h^2
    +O(h^3).
\end{equation}
Dividing by $h^2$ and taking $h\to0$ proves both limits. The pointwise-realizability assumption is essential: if a finite neural parameterization couples distant values of $\theta$, its population optimum need not equal the separate Gaussian moment projection at every stimulus.
\end{proof}

\section{A new distance metric that preserves the geometry of a template}
\label{sec:geometry}

A framework for distance metrics makes it possible to design new metrics in a principled way. A design consists of two choices: a source distribution, which specifies the basic structure used to approximate each response distribution; and a velocity class, which specifies the permitted transformations from source to target.

As a demonstration, we introduce a metric that estimates the distance between two distributional approximations constrained to preserve the geometry of a prescribed template. Such a metric can be useful when the geometry of the target distribution is known a priori, or when one wishes to emphasize a particular geometric property of the distribution.

\paragraph{Geometric-template source distribution.}
Let $\mathcal U\subseteq\mathbb R^m$ be a bounded parameter domain, typically with $m<d$, and let $f:\mathcal U\to\mathbb R^d$ be an a priori specified template embedding. Thus, $u\in\mathcal U$ indexes a location on the template; for example, $\mathcal U=[0,2\pi)$ and $f(u)=(\cos u,\sin u)$ trace a circle. Define the condition-indexed source distribution $\rho_{\mathrm{geom},c}$ by
\begin{equation}
    {
    Z_c=f(U)+\varepsilon_c,
    \qquad
    U\sim\operatorname{Unif}(\mathcal U),
    \qquad
    \varepsilon_c\sim\mathcal N(0,\sigma_c^2I).
    }
    \label{eq:geom-source}
\end{equation}
The additive Gaussian noise makes the source full dimensional, so it has a well-defined density in $\mathbb R^d$, while $\sigma_c$ controls the thickness of the distribution around the template. The template geometry $f$ is fixed across conditions; only its admissible transformation and its noise thickness are fitted.

\paragraph{Similarity-constrained velocity field.}
We want each fitted transformation to preserve the template's shape while allowing its position, orientation, and overall scale to vary. We therefore restrict the time-one map Eq.~\eqref{eq:fitted-distribution} to the orientation-preserving similarity group \citep{Hall2015LieGroups},
\begin{equation}
    {
    T_c(z)=s_cR_cz+t_c,
    \qquad
    s_c>0,\quad R_c\in\mathrm{SO}(d).
    }
\end{equation}
Here $t_c$ translates the template, $R_c$ rotates it, and $s_c$ changes its scale uniformly in every direction. These similarity transformations exclude reflection, shear, and nonlinear deformation.

To parameterize these transformations continuously, we constrain the velocity to the Lie algebra of this group:
\begin{equation}
    {
    \frac{dx}{dt}
    =
    a_{c,t}
    +\Omega_{c,t}(x-o_{c,t})
    +\lambda_{c,t}(x-o_{c,t}),
    \qquad
    \Omega_{c,t}^\top=-\Omega_{c,t}.
    }
    \label{eq:similarity}
\end{equation}
The quantities $a_{c,t}$, $\Omega_{c,t}$, $o_{c,t}$, and $\lambda_{c,t}$ are learnable parameters produced by a neural network. The vector $a_{c,t}$ is the infinitesimal generator of translation, the matrix $\Omega_{c,t}$ generates rotation about $o_{c,t}$, and the scalar $\lambda_{c,t}$ generates isotropic expansion or contraction. Integrating these infinitesimal generators over time produces a similarity transformation of the form above, as proved in Proposition~\ref{prop:app-fm-geometric}.

\paragraph{Training and induced geometry-template dissimilarity.}
Training proceeds in two stages. First, we fix each noise scale $\sigma_c$ at its initial value and optimize the neural-network parameters of the similarity-constrained velocity field using the flow-matching objective in Eq.~\eqref{eq:fm-objective}. Second, we make the positive scale $\sigma_c$ learnable and jointly fine-tune $\sigma_c$ and the velocity parameters by minimizing the NLL objective in Eq.~\eqref{eq:nll}. This stage adapts the template noiseless to the observations. After training, we define the geometry-template dissimilarity between conditions $a$ and $b$ as their Jeffreys divergence as before.

\subsection{Demonstration of a new distance metric that preserves the geometry of a template}
\label{sec:geometry-experiment}
Besides estimation, having a framework allows one to design new distance metrics. In this section, we demonstrate the geometry-preserving metric from Section~\ref{sec:geometry} on a toy dataset. This metric fits each target by a distribution that preserves the geometry of a chosen source template.

The synthetic dataset contains two conditions, each represented by a noisy square with a different orientation. The source distribution is a unit square with a small amount of noise. We fit each condition using either a similarity-constrained velocity field, which permits only translation, rotation, and isotropic scaling of the source distribution, or an unconstrained nonlinear velocity field. The similarity-constrained model better recovers the ground truth and preserves the square geometry, whereas the unconstrained model distorts it (Figure~\ref{fig:geom}). This combination of a geometry-informed source template and a similarity-constrained velocity field is therefore useful when the data geometry is known or when a particular geometric structure should be preserved.

\section{Experimental Details}
\label{app:experiments}

We organize the experimental details into four groups: estimator specifications, simulation datasets, real neural datasets, and the joint estimator--dataset specification for the geometry-template Jeffreys distance.

\subsection{Estimators}

\subsubsection{Flow matching training}
\label{sec:flow-matching}

For affine flow which contains a velocity with $A$ matrix, for example, $v_c(x,t)=\dot\beta_t b_c+A_{c,t}(x-\beta_t b_c)$, the symmetric matrix $A_{c,t}$ contains $d(d+1)/2=O(d^2)$ learned entries. This parameter growth can make high-dimensional fits unstable with finite samples, so we use the regularized loss:
{
\begin{equation*}
\mathcal L_{\lambda}
=\mathcal L_{\mathrm{FM}}
+\lambda\,\mathbb E_{c,t}\!\left[\frac{1}{d}\lVert A_{c,t}\rVert_F^2\right].
\end{equation*}
}
We compare six regularization strengths, $\lambda\in\{0,0.005,0.05,0.5,5,50\}$. We first fit an unregularized model and save its checkpoint. Starting from this shared checkpoint, we then continue fitting separately for each value of $\lambda$. All resulting candidate checkpoints are evaluated on the same validation set, and we select the candidate with the minimum unpenalized validation Flow Matching loss (Section~\ref{sec:flow-network-architecture}).

For unconstrained velocity, we don't use regularization. We instead divide the fitting data into training and validation sets, apply early stopping with a patience of 1000 epochs, and retain the checkpoint with the lowest validation loss.

\paragraph{Additional training details.} Unless otherwise specified, source samples $X_0$ are drawn independently from $\mathcal N(0,I_d)$ and independently of the observed responses. At each optimization step, we pair source samples with observed responses, draw the flow time uniformly as $t\sim\operatorname{Unif}[5\times10^{-4},1-5\times10^{-4}]$, and minimize the loss along a cosine probability path. We optimize using AdamW \citep{LoshchilovHutter2019AdamW} and clip the gradient norm at $10$. The network architecture, batch size, and learning rate are experiment-specific and summarized in Table~\ref{tab:flow-implementation-settings}.

\subsubsection{Velocity network architecture}
\label{sec:flow-network-architecture}

\paragraph{Condition embedding.}
Raw conditions are converted to fixed feature vectors before they enter the velocity model. A categorical condition $c\in\{0,\ldots,K-1\}$ is represented by the one-hot vector $\phi_{\mathrm{cat}}(c)=e_c\in\mathbb R^K$. 

Throughout this subsection, $J=8$ denotes the number of Gaussian radial-basis centers used for both nonperiodic and periodic scalar embeddings, as detailed below.

For a nonperiodic scalar condition $s$, let $\bar s_{\mathrm{tr}}$ and $\sigma_{\mathrm{tr}}$ denote the mean and root-mean-square scale of the training conditions. The standardized coordinate $z(s)$ and the linear-plus-RBF embedding are
\begin{equation*}
\begin{aligned}
z(s)&=\frac{s-\bar s_{\mathrm{tr}}}{\sigma_{\mathrm{tr}}},\\
m_j&=z_{\min}+\frac{j-1}{J-1}(z_{\max}-z_{\min}),
\qquad j=1,\ldots,J,\\
h&=\frac{z_{\max}-z_{\min}}{J-1},
\qquad
\psi_j(s)=\exp\!\left[-\frac{1}{2}
\left(\frac{z(s)-m_j}{h}\right)^2\right],\\
\phi_{\mathrm{np}}(s)&=
\bigl(z(s),\psi_1(s),\ldots,\psi_J(s)\bigr)
\in\mathbb R^{J+1},
\qquad J=8.
\end{aligned}
\end{equation*}
Here $z_{\min}$ and $z_{\max}$ are the extrema of the standardized training conditions. Each Gaussian feature is largest near its center $m_j$ and decays smoothly with distance, while retaining $z(s)$ preserves the global linear ordering. The continuous Fisher benchmarks therefore use a nine-dimensional condition vector.

For a period-$P$ scalar, centers instead lie uniformly on a circle. With origin $o$, the periodic embedding uses the shortest wrapped displacement $\delta_P$:
\nopagebreak[4]
\begin{equation*}
\begin{aligned}
m_j&=o+\frac{j-1}{J}P,
\qquad h=\frac{P}{J},
\qquad j=1,\ldots,J,\\
\delta_P(s,m_j)&=\operatorname{mod}\!\left(s-m_j+\frac{P}{2},P\right)-\frac{P}{2},\\
\psi_j^{\mathrm{per}}(s)&=
\exp\!\left[-\frac{1}{2}
\left(\frac{\delta_P(s,m_j)}{h}\right)^2\right],\\
\phi_{\mathrm{per}}(s)&=
\bigl(\psi_1^{\mathrm{per}}(s),\ldots,\psi_J^{\mathrm{per}}(s)\bigr)
\in\mathbb R^J,
\qquad J=8.
\end{aligned}
\end{equation*}
Consequently $\phi_{\mathrm{per}}(s+P)=\phi_{\mathrm{per}}(s)$, so nearby orientations across the periodic boundary remain nearby in feature space.

In the Stringer experiment, both affine and unconstrained flow use the eight-dimensional periodic embedding $\phi_{\mathrm{per}}(s)$ with $P=\pi$. For a mixed condition consisting of a categorical variable $c$ and a scalar $s$, which may be periodic or non-periodic, we concatenate the two representations as $\phi_{\mathrm{mix}}(c,s)=[e_c,\phi(s)]$. We use this representation for the time-resolved OU experiments, where $c$ denotes category and $s$ denotes physical time. For both the Allen two-photon calcium-imaging and Neuropixels experiments, $c$ indexes the 40 joint drifting-direction--temporal-frequency conditions, while $s$ denotes physical time and is represented by $\phi_{\mathrm{np}}(s)$.

\paragraph{Network architecture.}
Responses are standardized using statistics fitted on the fitting responses. The scalar flow time is used directly without any embeddings. Writing $h_t(u)=[t,\phi(u)]$, where $\phi(u)$ denotes the appropriate categorical, scalar, periodic, or mixed embedding of the experimental condition described above. Neural architectures are feed-forward multilayer perceptrons with SiLU activation with the width and number of hidden layers listed in Table~\ref{tab:flow-implementation-settings}. Neural networks output velocity parameters:

\begin{equation*}
\begin{aligned}
b(c)&=\operatorname{MLP}_b\!\left(\phi(c)\right)\in\mathbb R^d
\qquad\text{(endpoint mean)},\\
a_{c,t}&=\operatorname{MLP}_A\!\left(h_t(c)\right)
\in\mathbb R^{d(d+1)/2},
\qquad A_{c,t}=\operatorname{unpack}_{\mathrm{sym}}(a_{c,t})=A_{c,t}^{\top},\\
v_c^{\mathrm{aff}}(x_t,t)
&=\dot\beta_t b(c)+A_{c,t}\bigl(x_t-\beta_t b(c)\bigr),\\
v_c^{\mathrm{uncon}}(x_t,t)
&=\operatorname{MLP}_v\!\left([x_t,t,\phi(c)]\right)\in\mathbb R^d.
\end{aligned}
\end{equation*}
Translation velocity sets $A_{c,t}=0$, and shared affine replaces $A_{c,t}$ by a condition-independent $A_t$. The square-geometry models use the analogous parameterizations.

\begin{table}[H]
\centering

\caption{Flow Matching training settings. Width $\times$ depth gives the hidden width and number of hidden layers.}
\label{tab:flow-implementation-settings}
\setlength{\tabcolsep}{3pt}
\scriptsize
\begin{tabular}{@{}p{0.22\linewidth}p{0.20\linewidth}p{0.15\linewidth}p{0.10\linewidth}p{0.25\linewidth}@{}}
\toprule
Experiment & Flow family & Width $\times$ depth & Batch & Learning rate \\
\midrule
Categorical RDM & constrained and nonlinear & $128\times3$ & $3000$ & $10^{-4}$ for correlation/cosine; $10^{-3}$ otherwise \\
Independent Gaussian Fisher & affine & $256\times3$ & $256$ & $10^{-4}$ \\
Independent mixture Fisher & affine and nonlinear & $256\times3$ & $256$ & $10^{-4}$ \\
Gaussian OU & affine & $256\times3$ & $256$ & $10^{-4}$ \\
Mixture OU & affine and nonlinear & $256\times3$ & $256$ & $10^{-4}$ \\
Stringer & affine and nonlinear & $256\times3$ & $256$ & $10^{-4}$ \\
Allen & affine and nonlinear & $256\times3$ & $2048$ & $10^{-4}$ \\
Square geometry & similarity-constrained and nonlinear & $64\times2$ & $256$ & $10^{-3}$ FM; $10^{-4}$ NLL \\
\bottomrule
\end{tabular}
\end{table}

\paragraph{Jeffreys divergence estimation.} When the velocity has the affine (linear-in-state) form $v_c^{\mathcal V}(x,t)=A_{c,t}x+r_{c,t}$, we integrate the mean and covariance ODEs in Eq.~\eqref{eq:affine-moment-odes} and evaluate the corresponding Gaussian distance from the resulting moments; Gaussian Jeffreys divergence uses Eq.~\eqref{eq:affine-gaussian-jeffreys-evaluation}. For an unconstrained velocity---used to estimate Jeffreys divergence and full Fisher information---we instead generate endpoint samples by solving Eq.~\eqref{eq:flow-ode} and evaluate $\log q_c^{\mathcal V}(x)$ through the continuous change-of-variables identity in Eq.~\eqref{eq:cnf-log-density}. Jeffreys divergence is then estimated with Eq.~\eqref{eq:mc-jeffreys}.

\subsubsection{Gaussian-process with kernel regression}
\label{sec:gkr}

Our implementation follows the GKR code of \citet{YeWessel2025GridInformation}, available at \url{https://github.com/AgeYY/speed_grid_cell_information}.

GKR models the smooth conditional response distribution as
\begin{equation}
    q_{\mathrm{GKR}}(x\mid u)
    =
    \mathcal N\!\left(x\mid\widehat\mu(u),\widehat\Sigma(u)\right),
\end{equation}
Here $u$ denotes the conditioning input to each fitted GKR model: $u=\theta$ for the Fisher datasets, $u=(c,t)$ for the synthetic OU datasets, and $u$ is grating orientation for Stringer. For the Allen two-photon calcium-imaging and Allen Neuropixels ephys datasets, we fit a separate GKR model for each joint drifting-direction--temporal-frequency condition, with $u=t$ within each model. We use an RBF kernel for nonperiodic inputs and a periodic kernel for grating orientation.

Let $r_i=x_i-\widehat\mu(u_i)$ be the fitted residuals. GKR estimates the conditional covariance by normalized kernel regression of their outer products,
\begin{equation}
    \widehat\Sigma(u)
    =
    \frac{\sum_{i=1}^{n}k_{\lambda}(u_i,u)r_i r_i^\top}
         {\sum_{i=1}^{n}k_{\lambda}(u_i,u)}
    +\epsilon_{\Sigma}I.
    \label{eq:gkr-kernel-covariance}
\end{equation}
For the nonperiodic scalar conditions, inputs are standardized and
\begin{equation}
    k_{\lambda}(u_i,u)
    =
    \exp\!\left[-\frac{(\widetilde u_i-\widetilde u)^2}
                           {2(\lambda+\epsilon_{\lambda})}\right].
\end{equation}
For an orientation with period $P$, the periodic covariance kernel is instead
\begin{equation}
    k_{\lambda}^{\mathrm{per}}(u_i,u)
    =
    \exp\!\left[-\frac{\sin^2\!\left(\pi(u_i-u)/P\right)}
                           {\lambda+\epsilon_{\lambda}}\right].
\end{equation}
The default configuration uses $200$ inducing points; $300$ mean-GP iterations at learning rate $0.05$; and $30$ covariance-kernel epochs at learning rate $0.1$, with a validation fraction of $1/3$.

\subsubsection{Wishart-process estimator}
\label{sec:wishart-process-upstream}

Our Wishart process implementation reuses the original JAX/NumPyro model, kernels, variational inference, and posterior-mode prediction from \url{https://github.com/neurostatslab/wishart-process}, pinned to commit \href{https://github.com/neurostatslab/wishart-process/commit/5607d144ce8d27f0f5ef12f905abce7e868c6259}{\texttt{5607d14}}.

The Wishart process models both the mean and covariance of a smoothly condition-dependent Gaussian distribution \citep{Nejatbakhsh2023Wishart}. For a conditioning input $u$ and response $X\in\mathbb R^d$,
\begin{equation}
    {
    X\mid u
    \sim\mathcal N\!\left(\mu(u),\Sigma(u)\right),
    \qquad
    \mu_j(\cdot)\sim\operatorname{GP}(0,k_\mu),
    \qquad
    F_{jp}(\cdot)\sim\operatorname{GP}(0,k_\Sigma).
    }
    \label{eq:wishart-process-upstream-observation}
\end{equation}
The latent functions $\mu_j$ model the conditional mean, while the factors $F_{jp}$ model the conditional covariance; all mean coordinates and covariance factors have independent GP priors. Let $F(u)\in\mathbb R^{d\times d}$ collect the active covariance factors used by the upstream implementation, which defines
\begin{equation}
    {
    \Sigma(u)
    =L\left[F(u)F(u)^\top+\delta I_d\right]L^\top,
    \qquad \delta=0.1.
    }
    \label{eq:wishart-process-upstream-covariance}
\end{equation}
The Gram matrix $F(u)F(u)^\top$ is positive semidefinite, and adding $\delta I_d$ makes it positive definite before scaling by $L$. This Gram construction motivates the name \emph{Wishart process}: at each condition, outer products of Gaussian latent factors form a covariance matrix, while the GP correlations across $u$ make that covariance vary smoothly with the condition.
The shared matrix $L$ is initialized as the Cholesky factor of $(\widehat\Sigma_{\mathrm{pool}}+\delta I_d)/(d+1)$, where $\widehat\Sigma_{\mathrm{pool}}$ is the pooled within-bin empirical covariance. It is then optimized during SVI.

The mean and covariance-factor GPs use separate product kernels. For $r\in\{\mu,\Sigma\}$ and condition coordinate $q$,
\begin{align}
    k_r(u,u')
    &=\prod_q\exp\!\left[-\frac{d_q(u_q,u_q')^2}{\lambda_{r,q}}\right]
      +10^{-3}\mathbf 1\{u=u'\},
    \\
    d_q(a,b)
    &=\begin{cases}
        a-b, & \text{for a nonperiodic coordinate},\\
        \sin\!\left(\pi(a-b)/P_q\right), & \text{for a coordinate with period $P_q$}.
      \end{cases}
    \label{eq:wishart-process-upstream-kernel}
\end{align}
Here $\lambda_{\mu,q}$ and $\lambda_{\Sigma,q}$ separately control how smoothly the mean and covariance factors vary along condition coordinate $q$: larger widths produce stronger sharing across nearby conditions. The Fisher datsets have one nonperiodic coordinate $u=\theta$, whereas the Stringer recordings have one periodic orientation coordinate $u=\theta$ with period $P=\pi$.

The authors' code interface requires an equal-trial tensor with shape $K\times C\times d$. After fixing the train, validation, and test splits, our wrapper bins only the training observations, sets $K=\min_c n_c$, and subsamples every occupied bin to $K$ responses. All validation and test observations are retained. Mean and covariance kernel widths ($\lambda_{\mu,q}$ and $\lambda_{\Sigma, q}$) are selected by held-out Gaussian log likelihood.

\begin{table}[H]
    \centering

    \small
    \setlength{\tabcolsep}{5pt}
    \begin{tabular}{@{}p{0.29\linewidth}p{0.27\linewidth}p{0.34\linewidth}@{}}
        \toprule
        Benchmark & Mean-kernel widths & Covariance-kernel widths \\
        \midrule
        Linear Fisher
        & $\lambda_\mu\in\{5,10,20\}$
        & $\lambda_\Sigma\in\{20{,}480,40{,}960,81{,}920\}$ \\
        Gaussian-mixture full Fisher
        & $\lambda_\mu\in\{2.25,9,36\}$
        & $\lambda_\Sigma\in\{144,576,2304\}$ \\
        Stringer recordings (initial grid)
        & $\lambda_\mu\in\{0.00125\cdot2^j:j=0,\ldots,8\}$
        & $\lambda_\Sigma\in\{0.05\cdot2^j:j=0,\ldots,7\}\cup\{\infty\}$ \\
        Gaussian OU
        & $\lambda_{\mu,t}\in\{1,4,16,64,256\}$
        & $\lambda_{\Sigma,t}\in\{256\cdot4^j:j=0,\ldots,7\}\cup\{\infty\}$ \\
        Two-scale Gaussian-mixture OU
        & $\lambda_{\mu,t}\in\{1,4,16,64,256\}$
        & $\lambda_{\Sigma,t}\in\{16\cdot4^j:j=0,\ldots,8\}\cup\{\infty\}$ \\
        Allen Neuropixels PCA-10\newline
        & $\lambda_\mu\in\{0.015625\cdot2^j:j=0,\ldots,6\}$\newline$\cup\{4,\infty\}$
        & $\lambda_\Sigma\in\{0.25,1,4,16,64\}$\newline$\cup\{1024,16384,262144,\infty\}$ \\
        \bottomrule
    \end{tabular}
    \caption{Kernel-width grids used for the Wishart-process method. An infinite covariance width denotes a covariance that is constant over the continuous variable (e.g. orientation or trial time) within each condition.}
    \label{tab:upstream-wishart-kernel-grids}
\end{table}

To search kernel widths, we evaluate the grids in Table~\ref{tab:upstream-wishart-kernel-grids}, adaptively expanding the search space when needed. We continue the expansion until the validation curve exhibits an inverted-U shape, indicating a local optimum, or until the kernel width becomes sufficiently large to produce nearly constant smoothing. For example, supplementary Figure~\ref{fig:supp-stringer-wishart-grid-audit} shows the resulting search profiles for the Stringer recordings.

For the PCA-70 Allen Neuropixels analysis in Figure~\ref{fig:time-resolved-jeffreys}, we found it practically difficult to identify optimal kernel widths because validation often favored very large values, resulting in estimates that were nearly constant over time. This behavior may reflect the high dimensionality and non-Gaussianity of the neural responses.

Therefore we instead reducing the responses to 10 principal components where the Wishart process produced more reasonable fits, as shown in Supplementary Figure~\ref{fig:supp-allen-pca10-five-session-likelihood}. Affine flow matching remains best among Gaussian methods.

\subsubsection{Cross-fitted optimal linear estimator (OLE)}
\label{sec:ole}

We mainly follow \citet{Kanitscheider2015MeasuringFisher} for the locally optimal linear estimator and its finite-sample bias correction. OLE fits a linear decoder on a training set and projects held-out responses onto the resulting axis. Our reported achieved-information estimate applies the bias correction shown below and is clipped at zero.

Specifically, we estimate this quantity with five-fold cross-fitting independently for each adjacent pair $\theta_L<\theta_R$ on a 31-endpoint grid, where $h=\theta_R-\theta_L=0.4$. Responses are assigned to disjoint radius-$h/2$ windows around the two endpoints. If either window contains fewer than eight responses, we instead use the nearest disjoint groups with eight responses per endpoint. For fold $f$, four folds estimate the endpoint means and Ledoit--Wolf covariances. We then form
\begin{align}
    \widehat\mu_f'
    &=
    \frac{\widehat\mu_{R,f}-\widehat\mu_{L,f}}{h},
    &
    \widehat\Sigma_f
    &=
    \frac{\widehat\Sigma_{L,f}+\widehat\Sigma_{R,f}}{2},
    \\
    \widehat w_f
    &=
    \frac{\widehat\Sigma_f^{-1}\widehat\mu_f'}
         {\widehat\mu_f'{}^\top\widehat\Sigma_f^{-1}\widehat\mu_f'},
    &
    \widehat b_f
    &=
    \frac{\theta_L+\theta_R}{2}
    -\widehat w_f^\top
     \frac{\widehat\mu_{L,f}+\widehat\mu_{R,f}}{2}.
\end{align}
The fitted decoder $\widehat\theta=\widehat w_f^\top X+\widehat b_f$ is frozen and applied only to the held-out fold. Pooling the five held-out folds therefore evaluates each selected response exactly once out of fold. Let $\bar y_L,\bar y_R$, $s_L^2,s_R^2$, and $n_L,n_R$ denote the means, sample variances, and counts of the pooled held-out projections. The reported estimate is
\begin{align}
    \widehat I_{\mathrm{OLE}}^{\mathrm{raw}}
    &=
    \frac{
      (\bar y_R-\bar y_L)^2-s_L^2/n_L-s_R^2/n_R
    }{
      h^2(s_L^2+s_R^2)/2
    },
    &
    \widehat I_{\mathrm{OLE}}
    &=
    \max\!\left\{\widehat I_{\mathrm{OLE}}^{\mathrm{raw}},0\right\}.
\end{align}

\subsubsection{TRE density-ratio estimation for categorical Jeffreys divergence}
\label{sec:tre}

For the full Jeffreys divergence comparison in Supplementary Figure~\ref{fig:supp-categorical-distance-errors}, we use telescoping density-ratio estimation (TRE) \citep{Rhodes2020TRE} with $B\in\{1,8,16\}$ bridges. For each unordered pair of categorical conditions $a$ and $b$, TRE draws independent samples $X_a\sim p_a$ and $X_b\sim p_b$ and forms $B+1$ angular waymarks
\begin{equation}
    {
    Z_m=\cos(\phi_m)X_a+\sin(\phi_m)X_b,
    \qquad
    \phi_m=\frac{m\pi}{2B},\qquad m=0,\ldots,B.
    }
\end{equation}
Then we constructed $B$ adjacent-ratio classifiers share a three-hidden-layer MLP trunk with width $128$ and use separate linear output heads. Each condition-pair model is trained for at most $1000$ epochs with batch size $512$, AdamW learning rate $10^{-3}$, zero weight decay, and gradient clipping at norm $10$. We select the checkpoint with the lowest validation loss. If $q_m$ denotes the density of $Z_m$, classifier head $m$ is trained with equal class priors to distinguish $q_m$ from $q_{m+1}$. The estimated logits then telescope to
\begin{equation}
    {
    \widehat r_{ab}(x)
    =\sum_{m=0}^{B-1}\widehat\ell_{ab,m}(x)
    \approx\sum_{m=0}^{B-1}\log\frac{q_m(x)}{q_{m+1}(x)}
    =\log\frac{p_a(x)}{p_b(x)}.
    }
\end{equation}

We estimate the Jeffreys divergence for each condition pair by averaging this log ratio under the two endpoint distributions and clip negative finite-sample estimates at zero:
\begin{equation}
    {
    \widehat D_{\mathrm J}(p_a,p_b)
    =\max\!\left\{0,
    \frac{1}{n_a}\sum_{i=1}^{n_a}\widehat r_{ab}(X_{a,i})
     -\frac{1}{n_b}\sum_{j=1}^{n_b}\widehat r_{ab}(X_{b,j})
    \right\}.
    }
\end{equation}

\subsection{Simulation datasets}

\subsubsection{Categorical RDM dataset}

Here, we describe the categorical RDM experiment shown in Figure~\ref{fig:discrete}.

We define the conditional response distribution as
\begin{equation}
    z\mid Y=k\sim\mathcal N\!\left(\mu_k,\operatorname{diag}(v_k)\right),
\end{equation}
where $Y$ is uniformly distributed over $\{0, \ldots, 4\}$ and $z$ is three-dimensional. We generate the parameters $\mu_k$ and $\operatorname{diag}(v_k)$ once and hold them fixed across all five repetitions. Thus, the repetitions share the same category-conditional distributions and differ only in their sampled observations.

Specifically, we generate $\mu_k$ and $\operatorname{diag}(v_k)$ as follows. We draw each mean coordinate as $\mu_{k,j}=G_{k,j}Z_{k,j}$, where $G_{k,j}\sim\operatorname{Unif}(0.2,2.0)$ and $Z_{k,j}\sim\operatorname{Unif}(0,1)$. We resample the candidate means until every pair is separated by at least $\sqrt{3}/2$ (half the diagonal of the unit cube), thereby keeping the components well separated. We set the conditional variances to
\begin{equation}
    v_{k,j}=3(0.15)^2+0.15|\mu_{k,j}|+10^{-5}.
\end{equation}

We then independently draw five datasets from this fixed population. When fitting the structured models---that is, all velocity families in Table~\ref{tab:constraints} except the unconstrained family---we use all observations. Because the response dimension is low, we set $\lambda=0$ for affine flow, corresponding to no regularization. The nonlinear Jeffreys model instead uses an 80/20 training--validation split. Figure~\ref{fig:discrete}A shows accuracy as a function of sample size, and Figure~\ref{fig:discrete}B shows the corresponding RDM estimates at 2,000 total samples.

Figure~\ref{fig:supp-categorical-distance-errors} compares flow matching against plug-in mean estimators for correlation, cosine, and squared Euclidean distance; pooled-covariance Mahalanobis estimators with and without Ledoit--Wolf (LW) shrinkage; a condition-specific LW covariance estimator for Gaussian Jeffreys divergence; and TRE estimators using 1, 8, or 16 bridges for full Jeffreys divergence.

We compute all six ground-truth RDMs analytically from the known population means $\mu_k$ and diagonal covariances $\operatorname{diag}(v_k)$.

We evaluate estimation performance using the mean relative absolute error (MRAE) over the unique off-diagonal category pairs. Let $\mathcal P=\{(i,j):0\leq i<j\leq C-1\}$ and $M=|\mathcal P|=\binom{C}{2}$. For an estimated RDM $\widehat D$ and its ground-truth RDM $D$,
\begin{equation}
    \operatorname{MRE}(\widehat D,D)
    =
    \frac{1}{M}\sum_{(i,j)\in\mathcal P}
    \frac{\left|\widehat D_{ij}-D_{ij}\right|}
         {\max\!\left\{\left|D_{ij}\right|,10^{-12}\right\}}.
\end{equation}
With $C=5$, the average includes $M=10$ unique category pairs.

\subsubsection{Fisher datasets}

Here we describe the datasets used for estimating Fisher information in Figure~\ref{fig:independent-ou-fisher-jeffreys}A--F. We first specify the Gaussian dataset used for linear Fisher information, then the mixture dataset used for full Fisher information.

\paragraph{Gaussian responses and linear Fisher information.}
For the Gaussian dataset (Figure~\ref{fig:independent-ou-fisher-jeffreys}A--C), $\theta\in[-6,6]$, while the response dimension $d$ follows the sweep protocol below:
\begin{equation}
    p^{\mathrm G}(x\mid\theta)
    =\mathcal N\!\left(x\mid\mu(\theta),\Sigma(\theta)\right).
\end{equation}
The mean is a Gaussian bump with equal-spacing centers $\vartheta_j=-6+12(j-1)/(d-1)$
\begin{equation}
    \mu_j(\theta)
    =a_j\exp\!\left[-0.2\left(\theta-\vartheta_j\right)^2\right],
\end{equation}
where $a_j\overset{\mathrm{iid}}{\sim}\operatorname{Unif}(0.2,2.0)$. and the condition-dependent diagonal covariance is
\begin{equation}
    \left[\Sigma(\theta)\right]_{j\ell}
    =\delta_{j\ell}
      \left[
        d(0.07)^2
        +\frac{d}{300}\,1.30\lvert\mu_j(\theta)\rvert
        +10^{-8}
      \right].
\end{equation}
The $d/300$ factor preserves the relative covariance modulation across response dimensions.

The target is linear Fisher information, $\mu'(\theta)^\top\Sigma(\theta)^{-1}\mu'(\theta)$, estimated by affine flow (Section~\ref{sec:flow-matching} and Proposition~\ref{prop:app-fm-gaussian-fisher}), GKR (Section~\ref{sec:gkr}), the Wishart process (Section~\ref{sec:wishart-process-upstream}), and five-fold cross-fitted OLE (Section~\ref{sec:ole}). The sample-size sweep fixes $d=300$ and uses $N\in\{500,1000,3000,5000,10000\}$; the dimension sweep fixes $N=1000$ and uses $d\in\{10,20,50,100,200,300,400\}$.

For the affine-flow readout, each adjacent pair $(\theta_k,\theta_{k+1})$, with midpoint $\overline\theta_k=(\theta_k+\theta_{k+1})/2$, shares a covariance obtained by averaging its two affine matrices at every flow time:
\begin{align}
    \overline A_{k,t}
    &=\tfrac12\left(A_{\theta_k,t}+A_{\theta_{k+1},t}\right),
    &
    \dot{\overline\Sigma}_{k,t}
    &=\overline A_{k,t}\overline\Sigma_{k,t}
      +\overline\Sigma_{k,t}\overline A_{k,t}^{\top},
    &
    \overline\Sigma_{k,0}&=I_d,\\
    \widehat I_{\mathrm{lin}}(\overline\theta_k)
    &=\frac{\Delta b_k^{\top}\overline\Sigma_{k,1}^{-1}\Delta b_k}
            {(\theta_{k+1}-\theta_k)^2},
    &
    \Delta b_k&=b(\theta_{k+1})-b(\theta_k).
\end{align}

For estimated Fisher information $\widehat I(\theta_k)$ and ground truth $I(\theta_k)$ at $K$ evaluation points, we report MRAE
\begin{equation}
    \label{eq:mrae-fisher-information}
    \operatorname{MRAE}(\widehat I,I)
    =
    \frac{1}{K}\sum_{k=1}^{K}
    \frac{\left|\widehat I(\theta_k)-I(\theta_k)\right|}
         {\max\!\left\{\left|I(\theta_k)\right|,10^{-12}\right\}}.
\end{equation}
We additionally report the Pearson correlation between the estimated and ground-truth Fisher curves.

\paragraph{Mixture responses and full Fisher information.}
For the non-Gaussian dataset (Figure~\ref{fig:independent-ou-fisher-jeffreys}D--F), we retain the Gaussian dataset's $\mu(\theta)$ and $\Sigma(\theta)$, but introduce an equally likely latent scale branch $S$:
\begin{equation}
    \label{eq:independent-fisher-gaussian-mixture}
    \Pr(S=-1)=\Pr(S=+1)=\tfrac12,
    \qquad
    X\mid\theta,S
    \sim
    \mathcal N\!\left(
        \mu(\theta),
        [1+S\rho(\theta)]\Sigma(\theta)
    \right).
\end{equation}
where $\rho(\theta)=0.5+0.4\sin\!\left[\frac{\pi}{6}(\theta+6)+2.20\right]$. Because $\mathbb E[S]=0$, the marginal mean and covariance remain exactly $\mu(\theta)$ and $\Sigma(\theta)$; $\rho(\theta)$ therefore changes higher-order structure without changing either of the first two moments, keeping the linear Fisher information the same. The target for the common error panels is the exact mixture full Fisher information,
\begin{equation}
    I_{\mathrm{full}}(\theta)
    =
    \mathbb E_{X\sim p^{\mathrm{mix}}(\cdot\mid\theta)}
    \!\left[
        \left(\partial_\theta\log p^{\mathrm{mix}}(X\mid\theta)\right)^2
    \right].
\end{equation}
We use spacing $\Delta\theta=0.4$ and compare Unconstrained Flow, Affine Flow, GKR, and Wishart. The sample-size sweep fixes $d=20$ and uses $N\in\{1000,4000,7000,10000,13000\}$; the dimension sweep fixes $N=7000$ and uses $d\in\{5,20,50,80,110\}$.

\subsubsection{Time-resolved Jeffreys datasets}
\label{sec:time-resolved-jeffreys-benchmark}

Here we describe the Gaussian and non-Gaussian OU datasets in Figure~\ref{fig:ou-time-resolved-jeffreys}A--F.

\paragraph{Gaussian OU dataset: time-local marginals.}
The time-resolved dataset uses two conditions, $c\in\{0,1\}$, response dimension $d$, and $T=24$ time-bin centers
$t_k=-6+\left(k-\tfrac12\right)\Delta t$, $\Delta t=0.5$, $k=1,\ldots,T$.
For each tested $d$, the coordinatewise tuning-curve amplitudes are drawn once as
$a_j\overset{\mathrm{iid}}{\sim}\operatorname{Unif}(0.2,2.0)$, and the preferred time points are evenly spaced over the stimulus interval,
$\vartheta_j=-6+\frac{12(j-1)}{d-1}$, $j=1,\ldots,d$.
The common unsigned time-tuning curve is therefore
\begin{equation}
    \mu_j(t)
    =a_j\exp\!\left[-0.2\left(t-\vartheta_j\right)^2\right],
    \qquad
    \mu(t)=\bigl(\mu_1(t),\ldots,\mu_d(t)\bigr)^\top.
\end{equation}
Let $s_0=1$ and $s_1=-1$. At any fixed time $t$, the marginal distribution is:
\begin{equation}
    p_{\mathrm G}(x\mid c,t)
    =\mathcal N\!\left(x\mid s_c\,0.7\mu(t),\Sigma_{\mathrm G}(t)\right),
\end{equation}
\begin{equation}
    \left[\Sigma_{\mathrm G}(t)\right]_{j\ell}
    =\delta_{j\ell}
      \left[
        d(0.07)^2
        +\frac{d}{300}\,1.30\lvert\mu_j(t)\rvert
        +10^{-8}
      \right].
\end{equation}

\paragraph{Gaussian OU dataset: temporally correlated trials.}
To introduce within-trial dependence without changing the time-local marginals, let $L_{\mathrm G}(t)L_{\mathrm G}(t)^\top=\Sigma_{\mathrm G}(t)$. For each condition $c$ and trial $r$, standardized residuals follow the stationary OU recursion
\begin{align}
    z_{c,r,k}
    &=
    \rho z_{c,r,k-1}
    +\sqrt{1-\rho^2}\,\eta_{c,r,k},
    \qquad
    k=2,\ldots,T,
    \\
    \rho
    &=
    e^{-\Delta t/2}=e^{-0.5/2},
    \qquad
    z_{c,r,1},\eta_{c,r,k}
    \overset{\mathrm{ind}}{\sim}\mathcal N(0,I_d).
\end{align}
The observations are
\begin{equation}
    x_{c,r,k}=s_c\,0.7\mu(t_k)+L_{\mathrm G}(t_k)z_{c,r,k}.
\end{equation}
Each trial is an independent time sequence. Affine Flow uses a condition-balanced, trajectory-disjoint 80/20 training/validation split, so all time points from a trial remain in the same split. Bin\,+\,LW and GKR use all available trials.

Error is summarized by
\begin{equation}
    \operatorname{MRAE}
    =
    \frac1T\sum_{k=1}^T
    \frac{\left|\widehat D_{\mathrm J}(t_k)-D_{\mathrm J}(t_k)\right|}
         {\left|D_{\mathrm J}(t_k)\right|}.
\end{equation}
We compare Bin\,+\,LW, GKR, Affine Flow, and the Wishart process. The representative case has $d=60$ and $R=260$ trials per condition. The trial sweep fixes $d=60$ and uses $R\in\{80,215,350,485,620\}$; the dimension sweep fixes $R=260$ and uses $d\in\{20,40,60,80,100\}$.

\paragraph{Non-Gaussian two-scale OU dataset.}
For each condition $c$ and trial $r$, a narrow or wide covariance branch $S_{c,r}\in\{-1,+1\}$ is drawn with equal probability and held fixed over the complete trajectory. Its time-local conditional distribution and base covariance are
\begin{align}
    X\mid c,t,S_{c,r}
    &\sim
    \mathcal N\!\left(
        s_c\,0.4\mu(t),
        [1+S_{c,r}\rho_c(t)]\Sigma_{\mathrm M}(t)
    \right),
    \\
    \left[\Sigma_{\mathrm M}(t)\right]_{j\ell}
    &=4\delta_{j\ell}
      \left[
        d\left(\frac{0.2}{\sqrt 2}\right)^2
        +\frac{d}{50}\,0.65\lvert\mu_j(t)\rvert
        +10^{-8}
      \right].
\end{align}
The two conditions have opposite time modulation,
\begin{equation}
    \rho_0(t)=0.5+0.4\sin\!\left[\omega(t+6)+2.2\right],
    \qquad
    \rho_1(t)=0.5-0.4\sin\!\left[\omega(t+6)+2.2\right],
    \qquad
    \omega=\frac{2\pi}{12}.
\end{equation}
Thus, with $L_{\mathrm M}(t)L_{\mathrm M}(t)^\top=\Sigma_{\mathrm M}(t)$, the complete trajectory is
\begin{equation}
    x_{c,r,k}
    =s_c\,0.4\mu(t_k)
     +\sqrt{1+S_{c,r}\rho_c(t_k)}\,L_{\mathrm M}(t_k)z_{c,r,k}.
\end{equation}
Equal branch probabilities preserve the base mean and covariance, while the opposite modulations make higher-order moments condition dependent. The exact full-mixture Jeffreys reference is Monte Carlo audited.

We compare GKR, Affine Flow, the Wishart process, and Unconstrained Flow. The representative case has $d=10$ and $R=260$ trials per condition. The trial sweep fixes $d=10$ and uses $R\in\{100,260,420,580,740\}$; the dimension sweep fixes $R=260$ and uses $d\in\{2,10,18,26,34\}$. Both sweeps use five repeats.

\subsection{Real neural datasets}

\subsubsection{Mouse V1 static-grating dataset}

The mouse V1 static-grating dataset contains six two-photon calcium-imaging sessions (GT1, GT2, GT3, TX38, TX39, and TX40) collected during full-field static-grating presentation \citep{Stringer2021Geometry}. Each trial contributes one population-response vector and one grating orientation, treated modulo $\pi$; each session contains about 4000 trials and roughly 20000 simultaneously recorded neurons. Within each session, we form a deterministic, orientation-stratified 64/16/20\% training/validation/test split using seed 7. We fit a centered, unwhitened PCA using only the training trials, retain 200 components, and apply the fixed projection to validation and test trials. We chose 200 because it is close to the across-session median of 209 components required to reach a 70\% cumulative explained-variance ratio.

We compare five conditional density estimators on the same splits: binning with Ledoit--Wolf covariance (Bin+LW), GKR, the Wishart process, affine flow, and unconstrained flow. Bin+LW uses 16 equal orientation bins of width $\pi/16$, GKR follows Section~\ref{sec:gkr}, and the Wishart process follows Section~\ref{sec:wishart-process-upstream}, using the adaptive 16-bin fitting protocol and a full-rank covariance factorization. Both flow models use eight periodic radial-basis features for orientation and a three-hidden-layer MLP of width 256; their fitting follows Section~\ref{sec:flow-matching}.

We evaluate held-out log likelihood on the test split. For Bin+LW, GKR, the Wishart process, and affine flow, we use the fitted mean and covariance to compute Gaussian log likelihood; for unconstrained flow, we use the continuous change-of-variables identity in Eq.~\eqref{eq:cnf-log-density}. This comparison measures conditional density fit rather than direct Fisher-information accuracy.

\subsubsection{Allen Brain Observatory two-photon calcium-imaging dataset with drifting gratings}
\label{sec:allen-brain-institute}

We use two-photon calcium-imaging responses from 11 Allen Brain Observatory visual-cortex recording sessions selected to contain more than 300 recorded neurons \citep{deVries2020VisualCoding}. Mice viewed drifting gratings spanning eight motion directions ($0^\circ,45^\circ,\ldots,315^\circ$) and five temporal frequencies ($1,2,4,8,$ and $15$ Hz), giving 40 direction--frequency conditions. Each condition has 14--15 repeated trials, corresponding to approximately 560--600 nonblank trials per session. A trial contains 120 frames sampled at 30 Hz and spans four seconds: one second before stimulus onset, two seconds of drifting-grating presentation, and one second after stimulus offset, which also forms the first prestimulus second of the next trial.

For each session, we pool all nonblank observations (conditions $\times$ trials $\times$ time points) and randomly split into 60\% training, 20\% validation, and 20\% test sets. Trial identity is not used in this split, so frames from the same trial can occur in multiple subsets. We fit a centered, unwhitened PCA to the raw training observations only, retain 95 components, and apply this fixed map to all observations. PCA-95 retains 60.1--89.6\% of the training variance across sessions (median 70.2\%).

The resulting 95-dimensional observations are conditioned on a categorical index for the 40 direction--frequency combinations and on continuous time. We compare Bin\,+\,LW, GKR, the Wishart process, affine flow, and unconstrained flow. Training responses are used for fitting, validation responses for checkpoint or hyperparameter selection, and test responses only for final evaluation.

Supplementary Figure~\ref{fig:supp-allen-two-photon-summary} summarizes the responses, held-out likelihood, and time-resolved distance estimates across the 11 two-photon sessions.

\subsubsection{Allen Visual Coding Neuropixels dataset}
\label{sec:allen-ephys}

We analyze all 32 recording sessions from the Allen Visual Coding Neuropixels release \citep{Siegle2021Survey}. These sessions share 40 nonblank drifting-grating conditions formed by eight motion directions and five temporal frequencies. For each session, we retain presentations whose $[-1,3)$-second analysis window does not overlap an invalid interval. We retain units whose channel lies in VISp, VISl, VISrl, VISam, or VISpm and require presence ratio $\geq0.95$, inter-spike-interval violations $\leq0.5$, and amplitude cutoff $\leq0.1$. The resulting datasets contain 92--373 units and 421--600 presentations per session, with 8--15 presentations per condition.

We count spikes in 40 nonoverlapping 100-ms bins spanning $[-1,3)$ seconds relative to stimulus onset and apply the elementwise square-root transform $k\mapsto\sqrt{k}$. Within each session, we fit a centered, unwhitened PCA to all concatenated condition--presentation--time observations before splitting and retain 70 components. These components explain 62.4--95.2\% of the variance across sessions (median 72.9\%). This is the space we are working on for the rest of analysis. We then randomly divide observation rows into 80\% training and 20\% test sets; the flow models further divide the outer training rows into 80\% training and 20\% validation rows.

The models condition on the 40-level direction--frequency index and continuous time. We compare condition-specific Bin\,+\,LW, GKR, affine flow, and unconstrained flow. Held-out likelihood is averaged within each condition and then across the 40 conditions, with recording session as the unit of aggregation. For Supplementary Figure~\ref{fig:supp-allen-ephys-jeffreys-timecourses}, we average the 28 unique Jeffreys divergences among the eight directions separately at each temporal frequency within each recording session. Curves and shaded bands show the mean and SEM across the 32 sessions, respectively, with recording session as the aggregation unit. The 4-Hz direction RDMs in Figure~\ref{fig:time-resolved-jeffreys}C are evaluated at the bin centers nearest $-0.5$, $1.0$, and $2.5$ seconds (realized centers $-0.55$, $0.95$, and $2.45$ seconds) and averaged across sessions.

\subsection{Estimator and dataset for the geometry-template Jeffreys distance}

\paragraph{Synthetic dataset.}
Let $s_\ell(U)$ parameterize the boundary of a centered square with side length $\ell$, where $U\sim\operatorname{Unif}[0,4]$. For condition $c$, we draw
\begin{equation}
    {
    X_c=R_{\theta_c}s_2(U)+\eta,
    \qquad
    \eta\sim\mathcal N(0,0.2^2I_2),
    \qquad
    (\theta_1,\theta_2)=(0,\pi/4).
    }
\end{equation}
Only $X_c$ and its condition angle $\theta_c$ are included in the observed dataset; the perimeter coordinate $U$ is latent.
We generate 200 observations per condition and split each condition into 128 training, 32 validation, and 40 test observations (64/16/20\%).

\paragraph{Flow-matching fitting.}
Both models start from the same side-length-one square boundary $s_1(U)$ convolved with isotropic Gaussian noise of initial scale $0.1$; thus, the clean source geometry is identical across conditions. We compare two velocity families, each parameterized by a two-hidden-layer MLP of width 64. The similarity-constrained model outputs translation, angular velocity, isotropic scale, and center parameters, so its endpoint map can only translate, rotate, and uniformly rescale the square. The unconstrained model instead maps the response, flow time, and condition directly to a free two-dimensional velocity and may deform the template nonlinearly. We first train each model by Flow Matching along a cosine interpolation path using batch size 256, learning rate $10^{-3}$, gradient clipping at norm 10, and validation early stopping with patience 1000, while holding the source noise fixed. We then fine-tune by negative log likelihood for 1200 epochs at learning rate $10^{-4}$, allow a learnable source noise scale $\sigma_c$, and retain the checkpoint with the lowest validation NLL.

\clearpage
\section{Supplementary Figures}

\setcounter{figure}{0}
\renewcommand{\thefigure}{S\arabic{figure}}
\renewcommand{\theHfigure}{S\arabic{figure}}

\begin{figure}[H]
  \centering
  \includegraphics[width=\linewidth]{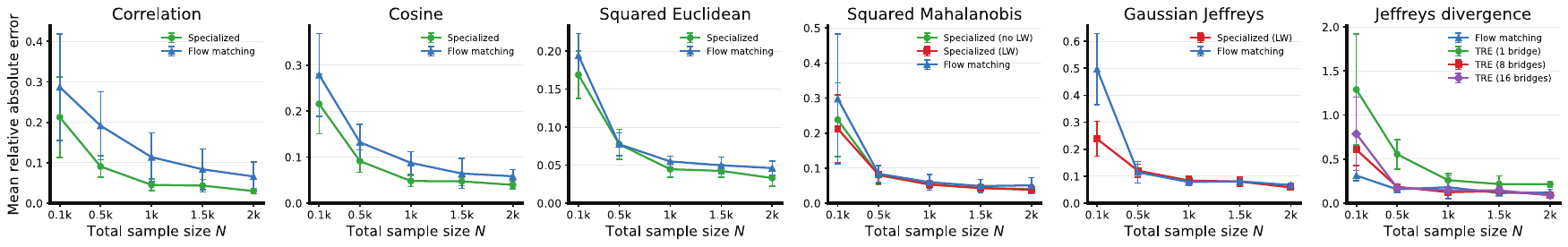}
  \caption{\textbf{Estimator accuracy across six representational distances on the five-condition categorical benchmark.} Mean relative absolute RDM error is shown against total sample size for correlation, cosine, squared Euclidean, squared Mahalanobis, Gaussian Jeffreys divergence, and full Jeffreys divergence. Specialized estimators are compared with flow matching; the full-Jeffreys panel also includes telescoping density-ratio estimation (TRE) with 1, 8, or 16 bridges~\citep{Rhodes2020TRE}. Markers and error bars show the mean and one sample standard deviation across five repetitions.}
  \label{fig:supp-categorical-distance-errors}
\end{figure}
\clearpage

\begin{figure}[H]
  \centering
  \includegraphics[width=\linewidth]{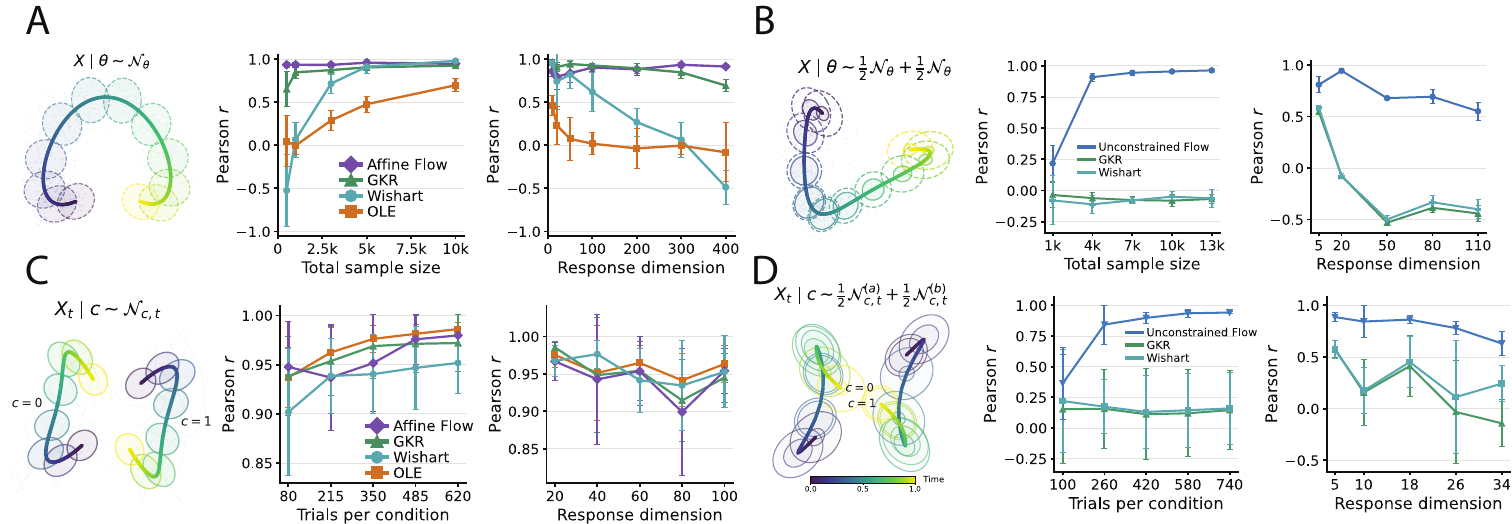}
  \caption{\textbf{Pearson correlation between the estimated and ground-truth curves.} (\textbf{A}) Independent Gaussian responses and linear Fisher information. (\textbf{B}) Independent two-scale Gaussian-mixture responses and full Fisher information. (\textbf{C}) Gaussian Ornstein--Uhlenbeck trials and time-resolved Gaussian Jeffreys divergence. (\textbf{D}) Two-scale Gaussian-mixture Ornstein--Uhlenbeck trials and time-resolved full Jeffreys divergence. Method colors follow Figures~\ref{fig:independent-ou-fisher-jeffreys} and~\ref{fig:ou-time-resolved-jeffreys}; points and error bars show the mean and one sample standard deviation across five repetitions. Figure relates to Figures~\ref{fig:independent-ou-fisher-jeffreys} and~\ref{fig:ou-time-resolved-jeffreys}}
  \label{fig:supp-continuous-benchmark-correlations}
\end{figure}
\clearpage

\begin{figure}[H]
  \centering
  \includegraphics[width=\linewidth]{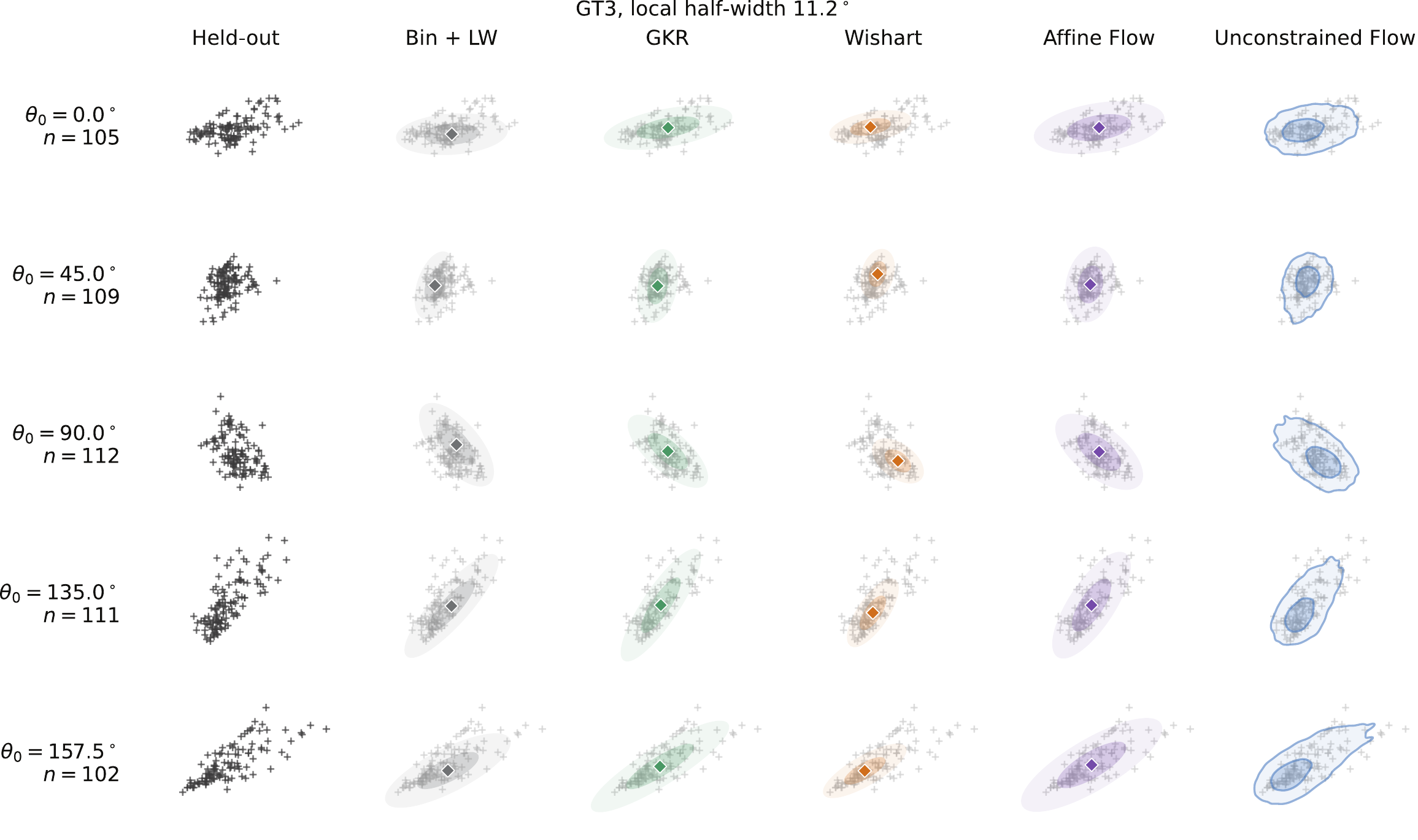}
  \caption{\textbf{Local conditional density fits across orientations in the representative Stringer session GT3.} Rows show held-out responses in the first two principal components within circular $\pm11.25^\circ$ windows centered at the displayed orientations. Columns compare held-out data, Bin\,+\,LW, GKR, the Wishart process, affine flow, and unconstrained flow. Crosses denote held-out observations; inner and outer ellipses or contours enclose 50\% and 95\% probability mass, respectively.}
  \label{fig:supp-stringer-local-density}
\end{figure}
\clearpage

\begin{figure}[H]
  \centering
  \includegraphics[width=\linewidth]{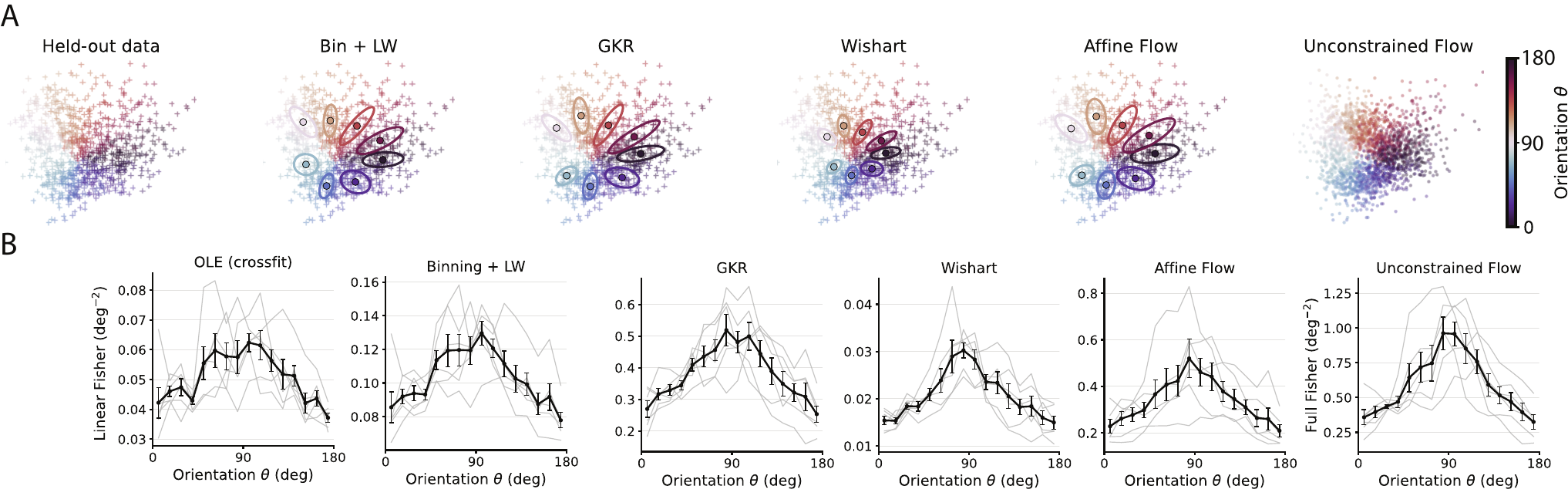}
  \caption{\textbf{Conditional density and Fisher-information estimates for the Stringer recordings.} (\textbf{A}) Held-out responses and fitted conditional densities in the first two principal components for an example session, colored by grating orientation. Gaussian estimators show fitted means and probability ellipses, whereas the unconstrained-flow panel shows generated samples. (\textbf{B}) Fisher-information estimates across the six sessions. GKR, Bin\,+\,LW, cross-fitted OLE, the Wishart process, and affine flow estimate linear Fisher information; unconstrained flow estimates full Fisher information. Thin gray curves show individual sessions, and colored curves and error bars show the pointwise mean and session-level standard error of the mean.}
  \label{fig:supp-stringer-density-fisher}
\end{figure}
\clearpage

\begin{figure}[H]
  \centering
  \includegraphics[width=\linewidth]{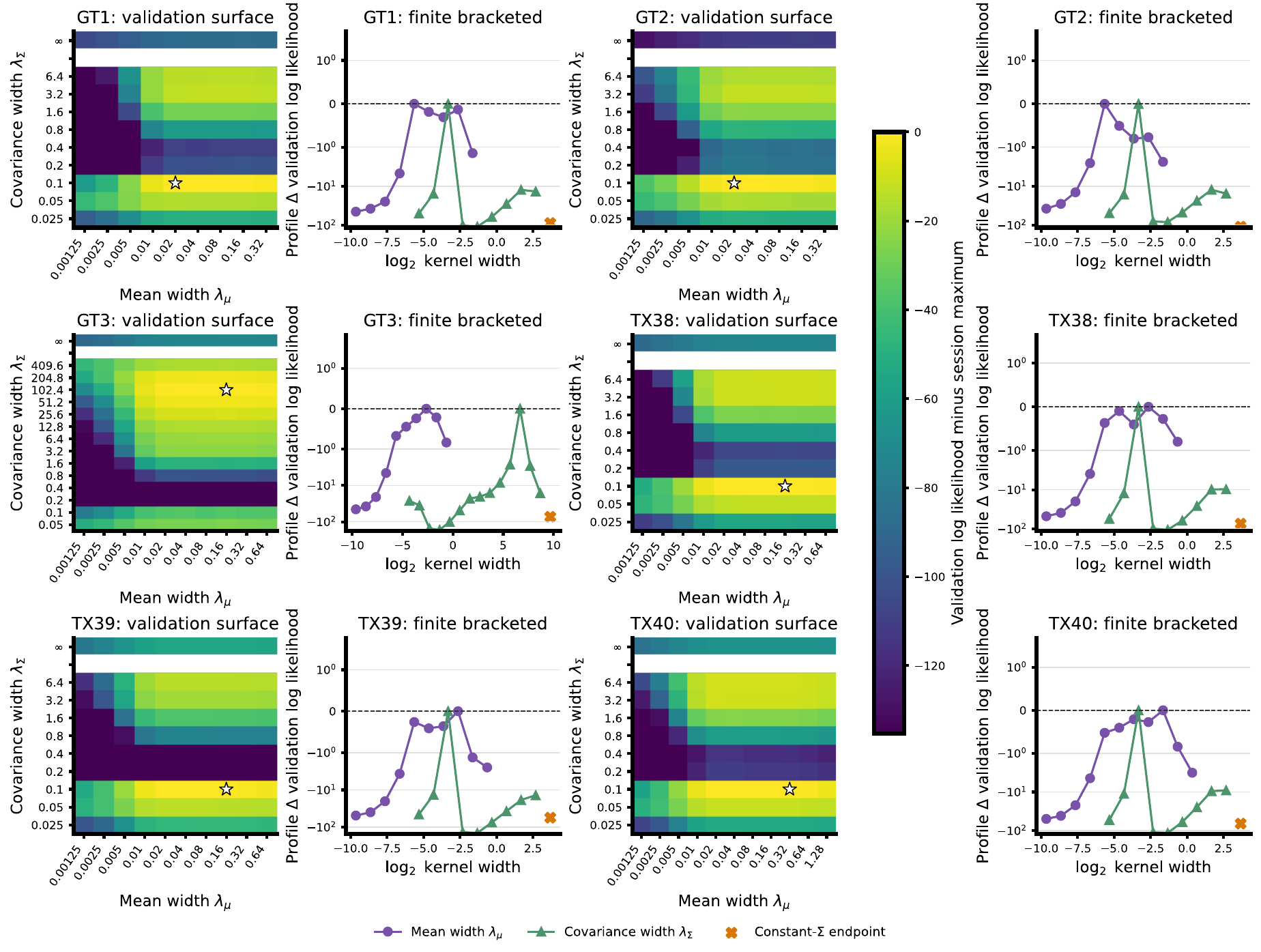}
  \caption{\textbf{Adaptive Wishart kernel-width searches for the Stringer recordings.} For each session, the validation surface shows Gaussian log likelihood relative to the session maximum over the mean width $\lambda_\mu$ and covariance width $\lambda_\Sigma$; the white star marks the selected pair. The adjacent profile panel shows relative validation likelihood across mean widths (purple circles) and finite covariance widths (green triangles), with the constant-covariance endpoint shown by the orange cross.}
  \label{fig:supp-stringer-wishart-grid-audit}
\end{figure}
\clearpage

\begin{figure}[H]
  \centering
  \includegraphics[width=\linewidth]{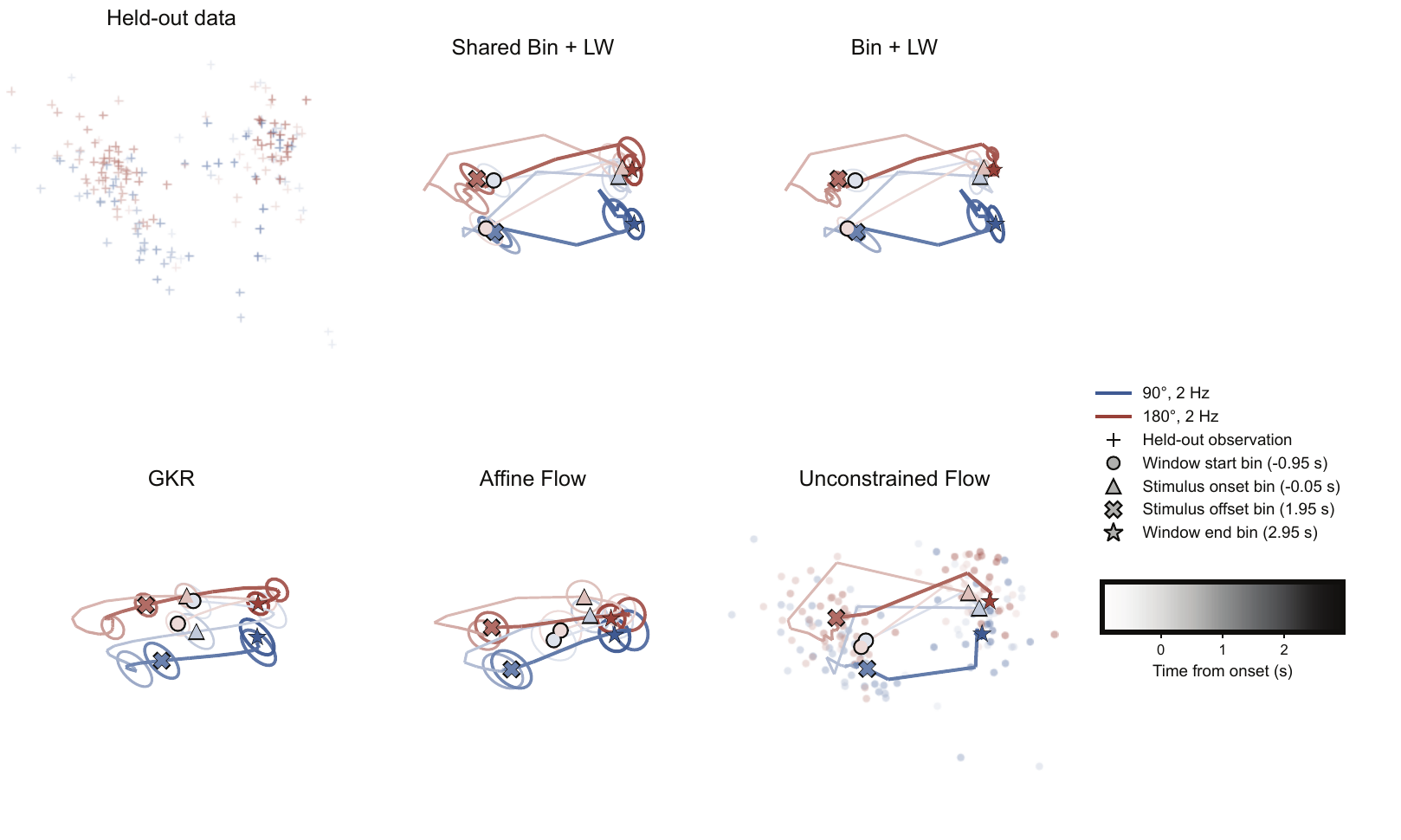}
  \caption{\textbf{Time-resolved conditional density trajectories for two drifting-grating conditions in the Allen Neuropixels recordings.} Held-out responses and fitted distributions are shown in the first two principal components for the $(90^\circ,2\,\mathrm{Hz})$ and $(180^\circ,2\,\mathrm{Hz})$ conditions. Columns show held-out data, shared Bin\,+\,LW, Bin\,+\,LW, GKR, affine flow, and unconstrained flow. Marker shape identifies representative time bins from the beginning to the end of the analysis window, and color lightness indicates time from stimulus onset.}
  \label{fig:supp-allen-ephys-density-trajectories}
\end{figure}
\clearpage

\begin{figure}[H]
  \centering
  \includegraphics[width=0.7\linewidth]{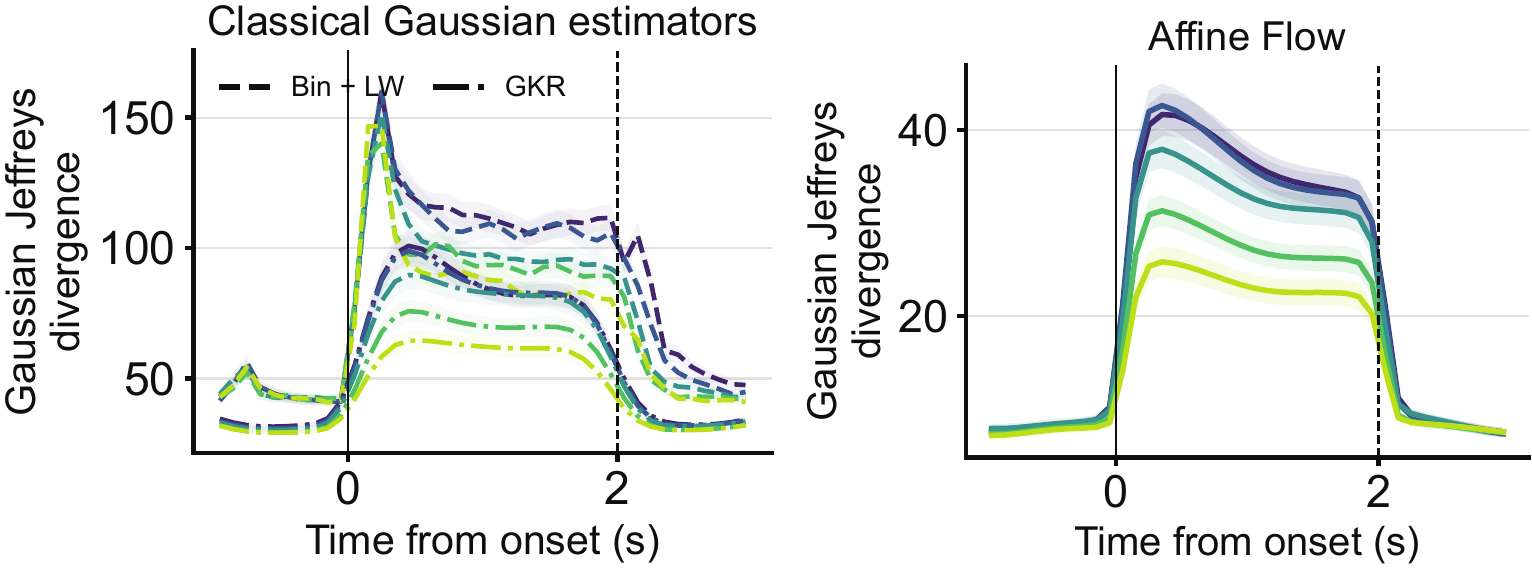}
  \caption{\textbf{Frequency-specific time-resolved Gaussian Jeffreys divergence from classical estimators and affine flow across 32 Allen Neuropixels sessions.} At each temporal frequency, Gaussian Jeffreys divergence is averaged over the 28 unique pairs among the eight directions within each session. Colors indicate the five temporal frequencies ($1$, $2$, $4$, $8$, and $15$ Hz). The left panel shows Bin\,+\,LW (dashed curves) and GKR (dash-dotted curves); the right panel shows affine flow (solid curves). Curves show across-session means, shaded regions show the SEM, and solid and dashed vertical lines mark stimulus onset and offset. Bin\,+\,LW uses a 0.2-s time window.}
  \label{fig:supp-allen-ephys-jeffreys-timecourses}
\end{figure}
\clearpage

\begin{figure}[H]
  \centering
  \includegraphics[width=\linewidth]{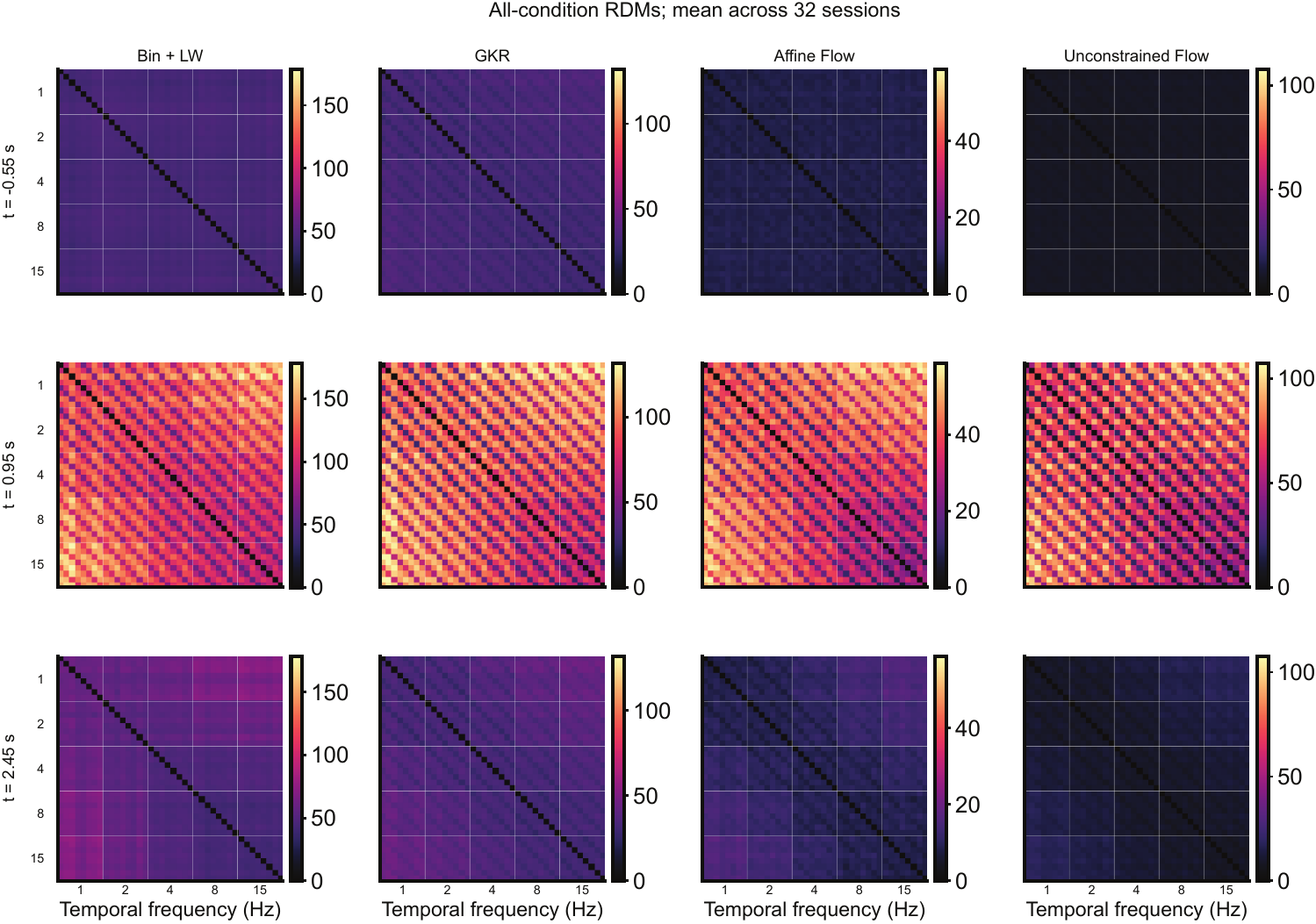}
  \caption{\textbf{Time-resolved all-condition RDMs across 32 Allen Neuropixels sessions.} Each matrix is the elementwise mean across sessions for all 40 direction--frequency conditions at $-0.55$, $0.95$, or $2.45$ seconds relative to stimulus onset. Columns show Gaussian Jeffreys divergence from Bin\,+\,LW, GKR, and affine flow, followed by full Jeffreys divergence from unconstrained flow. Conditions are grouped by temporal frequency, with eight motion directions within each block.}
  \label{fig:supp-allen-ephys-all-condition-rdms}
\end{figure}
\clearpage

\begin{figure}[H]
  \centering
  \includegraphics[width=\linewidth]{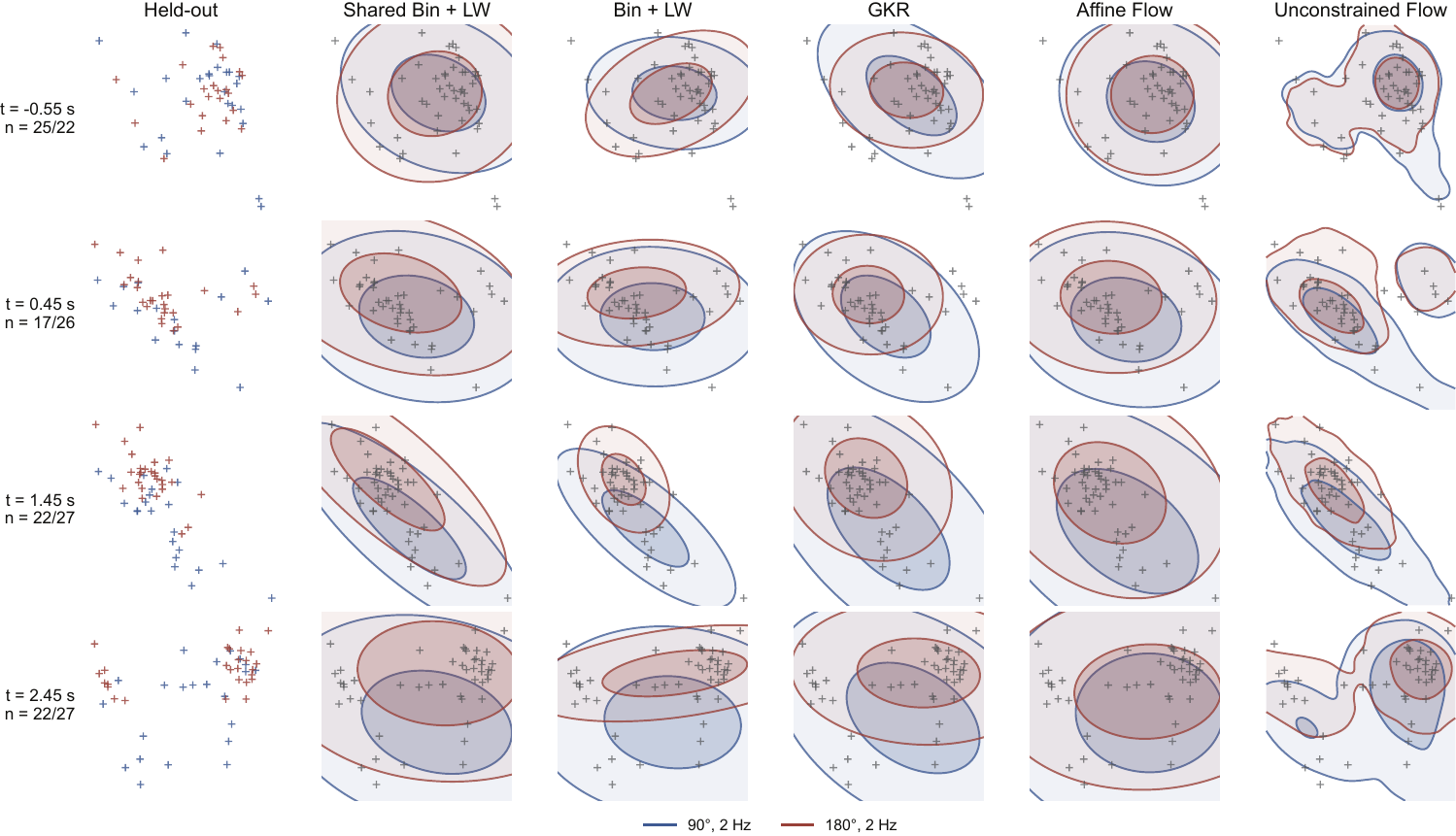}
  \caption{\textbf{Conditional density fits at four example times in an example Allen Neuropixels session.} Held-out responses and fitted conditional densities are shown in the first two principal components for the $(90^\circ,2\,\mathrm{Hz})$ and $(180^\circ,2\,\mathrm{Hz})$ conditions. Rows correspond to windows centered at $-0.55$, $0.45$, $1.45$, and $2.45$ seconds relative to stimulus onset; the displayed counts give the held-out observations for the two conditions. Columns compare held-out data, shared Bin\,+\,LW, Bin\,+\,LW, GKR, affine flow, and unconstrained flow.}
  \label{fig:supp-allen-ephys-density-fits}
\end{figure}
\clearpage

\begin{figure}[H]
  \centering
  \includegraphics[width=0.5\linewidth]{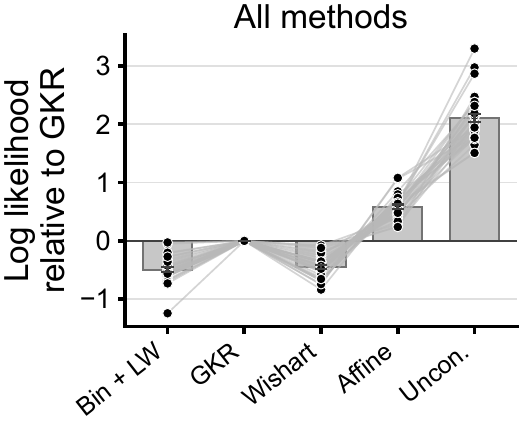}
  \caption{\textbf{Held-out likelihood across five Allen Neuropixels sessions after projection onto 10 principal components.} Held-out log likelihood relative to GKR is shown for Bin\,+\,LW, GKR, the Wishart process, affine flow, and unconstrained flow. Gray lines connect estimates from the same session, black dots show individual sessions, and bars and error bars show the across-session mean and standard error of the mean.}
  \label{fig:supp-allen-pca10-five-session-likelihood}
\end{figure}
\clearpage

\begin{figure}[H]
  \centering
  \includegraphics[width=\linewidth]{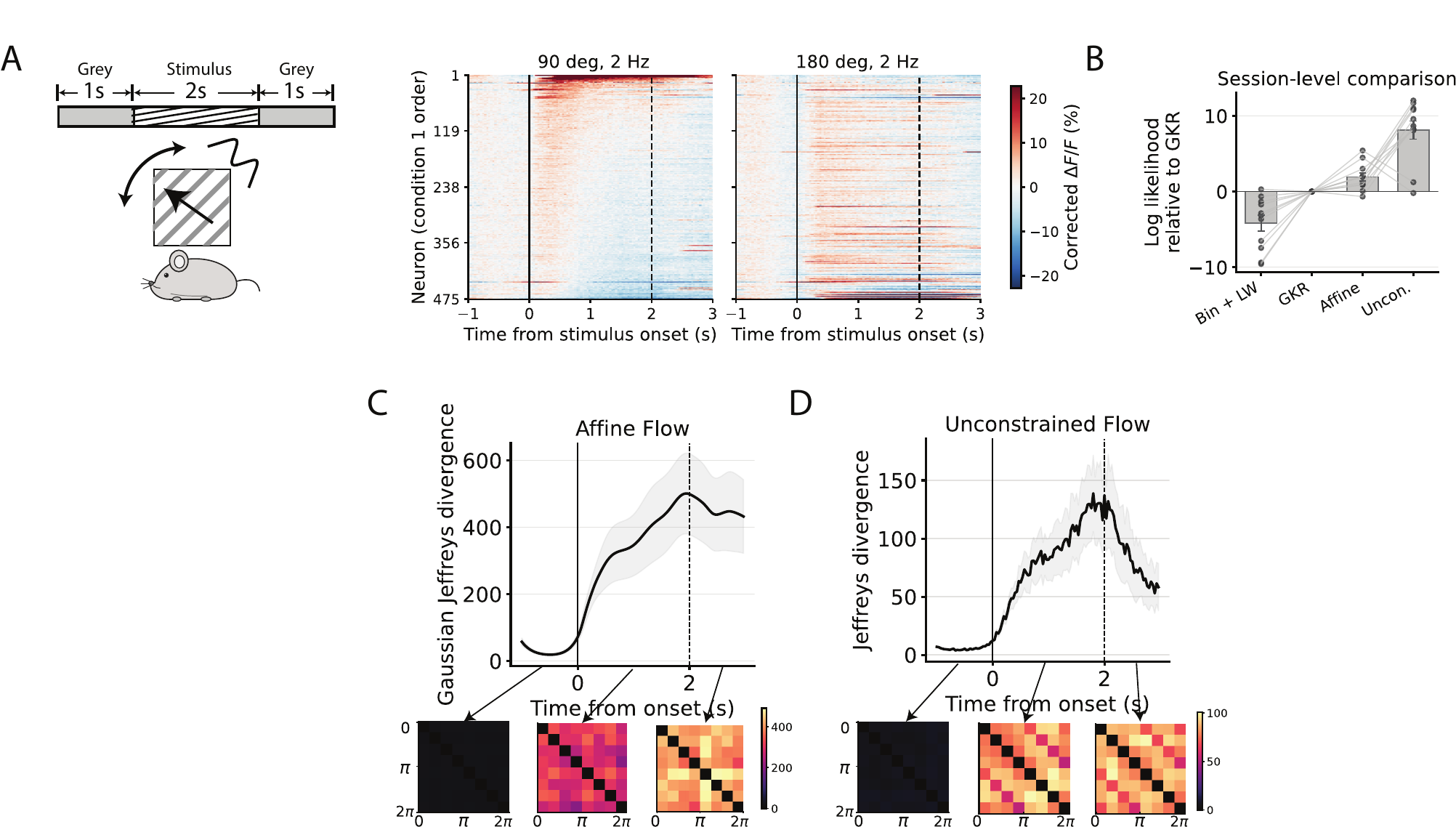}
  \caption{\textbf{Responses, held-out likelihood, and time-resolved representational geometry in the Allen two-photon recordings.} (\textbf{A}) Four-second trial structure and baseline-corrected response heat maps for two drifting-grating conditions in a representative session. (\textbf{B}) Held-out log likelihood relative to GKR across 11 imaging sessions; gray lines connect estimates from the same session, and bars and error bars show the across-session mean and standard error of the mean. (\textbf{C}) Affine-flow Gaussian Jeffreys divergence averaged over condition pairs and 4-Hz direction RDMs at representative pre-stimulus, stimulus, and post-stimulus times. (\textbf{D}) Corresponding full Jeffreys divergence and RDMs from unconstrained flow. Curves and shaded regions summarize mean and SEM of the 11 sessions.}
  \label{fig:supp-allen-two-photon-summary}
\end{figure}
\clearpage

\begin{figure}[H]
  \centering
  \includegraphics[width=0.7\linewidth]{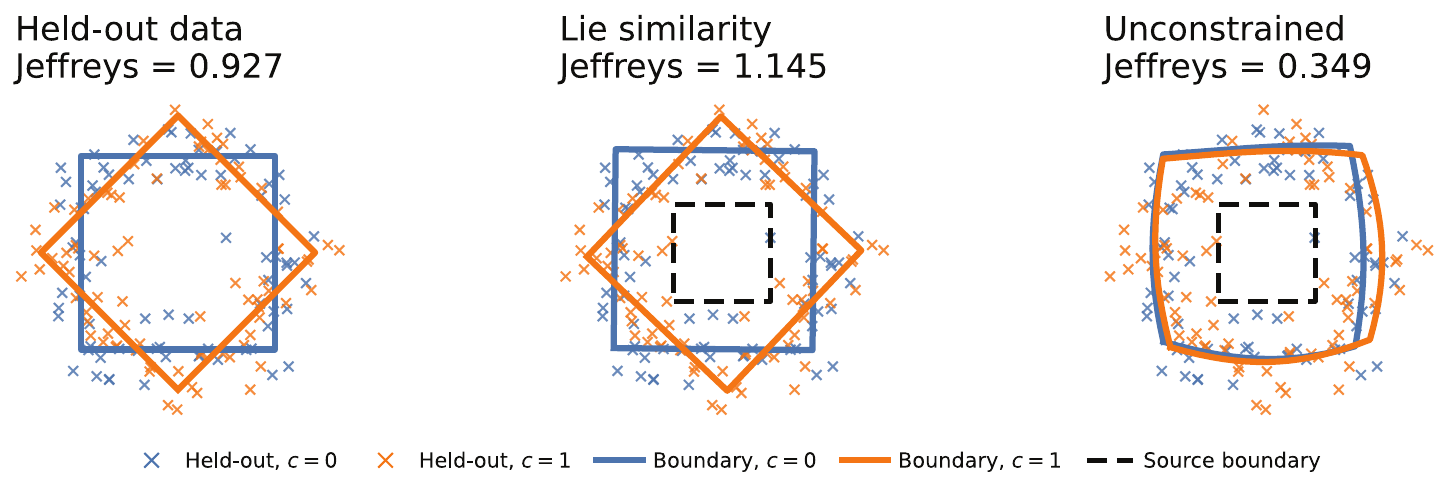}
  \caption{\textbf{Similarity-constrained flow preserves a square template when estimating representational dissimilarity.} Left: held-out observations from two noisy-square conditions, with solid curves showing their noiseless target boundaries. Middle and right: fits obtained using similarity-constrained and unconstrained flows, respectively, with the same side-length-one square source. In the fitted panels, the dashed black curve shows the original source boundary, and the solid colored curves show the condition-specific transported boundaries. Crosses denote held-out observations, and each panel title reports the Jeffreys divergence between the two condition distributions. The similarity-constrained fit preserves the square geometry, whereas the unconstrained fit distorts it.}
  \label{fig:geom}
\end{figure}
\clearpage

\begin{figure}[H]
  \centering
  \includegraphics[width=\linewidth]{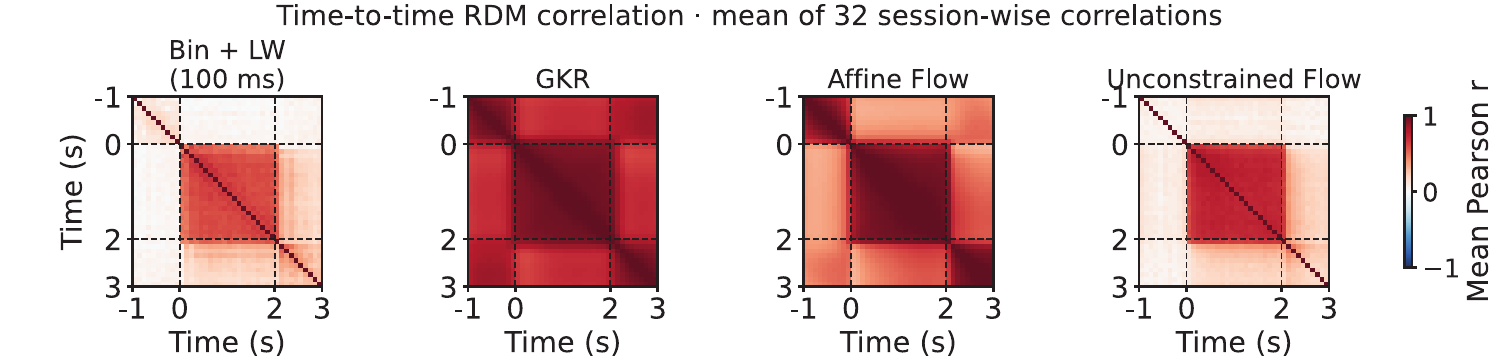}
  \caption{\textbf{Time-to-time RDM correlations across estimators in the Allen Neuropixels recordings.} Panels compare Bin\,+\,LW, GKR, affine flow, and unconstrained flow. Each matrix shows Pearson correlations between RDMs at pairs of time points, computed within each recording session and then averaged across 32 sessions. Dashed lines mark stimulus onset at $0$ seconds and offset at $2$ seconds.}
  \label{fig:supp-allen-rdm-time-correlations}
\end{figure}
\clearpage
\end{document}